\documentclass{article}
\usepackage{iclr2027_conference,times}

\usepackage[english]{babel} 
\usepackage[expansion=false]{microtype}      
\usepackage{csquotes}       
\usepackage[normalem]{ulem} 
\usepackage{xspace}         
\usepackage{textcomp}       

\usepackage{type1cm}        
\usepackage{bold-extra}     
\usepackage{bm}             
\usepackage{bbm}            
\usepackage{bbold}          
\usepackage{siunitx}        
\usepackage{xfrac}          

\usepackage{amsmath,amsfonts,amssymb,mathtools} 
\usepackage{amsthm}
\usepackage{physics}        
\usepackage{thmtools,thm-restate} 

\usepackage{booktabs}       
\usepackage{multirow}       
\usepackage{array}          
\usepackage{rotating}

\usepackage{graphicx}       
\usepackage[style=base, tableposition=top]{caption} 
\usepackage[caption=false,font=normalsize,labelfont=sf,textfont=sf]{subfig} 
\usepackage{tikz}           
\usepackage{tikz-3dplot}
\usetikzlibrary{patterns, positioning, arrows, bayesnet, calc, fit, backgrounds, decorations.pathreplacing} 
\usepackage{tikz-cd}        
\usepackage{pgfplots}       
\usepgfplotslibrary{patchplots}
\pgfplotsset{compat=1.15}   
\usepackage[most]{tcolorbox}      
\usepackage{wrapfig}
\usepackage{algorithm}
\usepackage{algorithmic}

\usepackage{hyperref} 
\usepackage[capitalize]{cleveref} 

\usepackage{balance}        
\usepackage{hyphenat}       
\usepackage[show]{chato-notes} 
\usepackage{stfloats}       
\usepackage{verbatim}       

\newcommand{\rankA}[1]{\ensuremath{\bm{#1}}}
\newcommand{\rankB}[1]{\underline{\ensuremath{#1}}}
\newcommand{\rankC}[1]{\ensuremath{\mathit{#1}}}

\newcommand{\spara}[1]{\smallskip\noindent\textbf{#1}}

\newcommand{\epara}[1]{\smallskip\noindent\emph{#1}}

{\par\egroup\vskip 0.25ex}

\usepackage[dvipsnames]{xcolor}
\hypersetup{
   colorlinks=true,
   linkcolor={black},
   filecolor={black},
   citecolor={black}, 
   urlcolor={black},
}

\definecolor{mypurple}{RGB}{254, 68, 218}
\definecolor{myred}{HTML}{E13D66}
\definecolor{mycyan}{HTML}{70D7D0}
\definecolor{mylightblue}{HTML}{2274A5}
\definecolor{mydarkblue}{HTML}{0C0A3E}
\definecolor{pil}{RGB}{245,166,35}
 
\newcommand{\pilotgrid}[5]{%
  \foreach \r in {1,...,#3}{\foreach \c in {1,...,#4}{
    \filldraw[fill=pil!55, draw=pil!75!black, line width=0.3pt]
      ({#1+(\c-1)*#5},{#2-#3*#5/2+(\r-1)*#5}) rectangle ({#1+\c*#5},{#2-#3*#5/2+\r*#5});}}}
\newcommand{\datamat}[9]{%
  \foreach \r in {1,...,#3}{\foreach \c in {1,...,#4}{
    \filldraw[fill=#8, draw=#9, line width=0.3pt]
      ({#1+(\c-1)*#5},{#2-#3*#5/2+(\r-1)*#5}) rectangle ({#1+\c*#5},{#2-#3*#5/2+\r*#5});}}%
  \foreach \r in {1,...,#3}{\foreach \c in {#6,...,\the\numexpr#6+#7-1\relax}{
    \filldraw[fill=pil!55, draw=pil!75!black, line width=0.35pt]
      ({#1+(\c-1)*#5},{#2-#3*#5/2+(\r-1)*#5}) rectangle ({#1+\c*#5},{#2-#3*#5/2+\r*#5});}}}
 
\newcommand{\datamatR}[9]{%
  \foreach \r in {1,...,#3}{\foreach \c in {1,...,#4}{
    \filldraw[fill=#8, draw=#9, line width=0.3pt]
      ({#1+(\c-1)*#5},{#2-#3*#5/2+(\r-1)*#5}) rectangle ({#1+\c*#5},{#2-#3*#5/2+\r*#5});}}%
  \foreach \r in {#6,...,\the\numexpr#6+#7-1\relax}{\foreach \c in {1,...,#4}{
    \filldraw[fill=pil!55, draw=pil!75!black, line width=0.35pt]
      ({#1+(\c-1)*#5},{#2-#3*#5/2+(\r-1)*#5}) rectangle ({#1+\c*#5},{#2-#3*#5/2+\r*#5});}}}
 
\newcommand{\vfunnel}[8]{%
  \foreach \b [evaluate=\b as \ws using {#4+(#5-#4)*(\b-1)/#1},
               evaluate=\b as \we using {#4+(#5-#4)*\b/#1},
               evaluate=\b as \ya using {#3-(\b-1)*#6},
               evaluate=\b as \yb using {#3-\b*#6}] in {1,...,#1}{
    \filldraw[fill=#7, draw=#8, line width=0.5pt, line join=round]
      ({#2-\ws},\ya) -- ({#2-\we},\yb) -- ({#2+\we},\yb) -- ({#2+\ws},\ya) -- cycle;}}

\tdplotsetmaincoords{70}{30}

\colorlet{cI}{mylightblue}
\colorlet{cJ}{myred}
\colorlet{cK}{Mulberry}
\colorlet{cE}{mycyan!62!black}     
\colorlet{cA}{black!78}            

\newcommand{\blue}[1]{\textcolor{black}{#1}}
\newcommand{\bluekeep}[1]{\textcolor{blue}{#1}}

\newcommand{\red}[1]{\textcolor{red}{#1}}

\newcommand{\mulberry}[1]{\textcolor{Mulberry}{#1}}
\newcommand{\teal}[1]{\textcolor{teal}{#1}}

\theoremstyle{plain}
\newtheorem{theorem}{Theorem}[section]
\newtheorem{proposition}[theorem]{Proposition}
\newtheorem{lemma}[theorem]{Lemma}

\newtheorem{problem}{Problem}
\newtheorem{remark}{Remark}

\newtheorem{assumption}{Assumption}

\tcolorboxenvironment{example}{
  colback=black!5!white,
  colframe=black!5!white,
  boxrule=0.8pt,
  arc=3mm,
  fonttitle=\bfseries,
  before skip=10pt,
  after skip=10pt,
  breakable
}

\crefname{theorem}{Thm.}{Thms.}
\crefname{proposition}{Prop.}{Props.}
\crefname{lemma}{lem.}{lems.}
\crefname{corollary}{Cor.}{Cors.}
\crefname{definition}{Def.}{Defs.}
\crefname{section}{Sec.}{Secs.}
\crefname{figure}{Fig.}{Figs.}
\crefname{problem}{Prob.}{Probs.}
\crefname{appendix}{App.}{Apps.}
\crefname{equation}{Eq.}{Eqs.}
\crefname{algorithm}{Alg.}{Algs.}
\crefname{table}{Tab.}{Tabs.}
\crefname{assumption}{Assumption}{Assumptions}

\newcommand{\nb}{~}

\DeclareMathOperator*{\argmin}{arg\,min}

\newcommand{\at}[2][]{#1|_{#2}}
\newcommand{\frob}[1]{\ensuremath{\norm{#1}_{\mathrm{F}}}\xspace}

\newcommand{\im}[1]{\ensuremath{\mathrm{Im}\left(#1\right)}\xspace}

\newcommand{\reall}{\ensuremath{\mathbb{R}}\xspace}

\newcommand{\edgeset}{\ensuremath{\mathcal{E}}\xspace}
\newcommand{\graph}{\ensuremath{\mathcal{G}}\xspace}

\newcommand{\vertexset}{\ensuremath{\mathcal{N}}\xspace}

\newcommand{\ort}[1]{\ensuremath{\mathrm{O}({#1})}\xspace}
\newcommand{\stiefel}[2]{\ensuremath{\mathrm{St}({#1},{#2})}\xspace}

\newcommand{\ones}{\ensuremath{\boldsymbol{1}}\xspace}

\newcommand{\x}{\ensuremath{\mathbf{x}}\xspace}
\newcommand{\y}{\ensuremath{\mathbf{y}}\xspace}

\newcommand{\z}{\ensuremath{\mathbf{z}}\xspace}

\newcommand{\A}{\ensuremath{\mathbf{A}}\xspace}

\newcommand{\identity}{\ensuremath{\mathbf{I}}\xspace}
\newcommand{\eye}[1]{\ensuremath{\identity_{#1}}\xspace}

\newcommand{\myO}{\ensuremath{\mathbf{O}}\xspace}
\newcommand{\myP}{\ensuremath{\mathbf{P}}\xspace}
\newcommand{\Q}{\ensuremath{\mathbf{Q}}\xspace}

\newcommand{\T}{\ensuremath{\mathbf{T}}\xspace}
\newcommand{\U}{\ensuremath{\mathbf{U}}\xspace}
\newcommand{\V}{\ensuremath{\mathbf{V}}\xspace}

\newcommand{\W}{\ensuremath{\mathbf{W}}\xspace}
\newcommand{\Y}{\ensuremath{\mathbf{Y}}\xspace}

\newcommand{\ns}{\ensuremath{\mathcal{F}}\xspace}

\title{Sheaf-Based Federated Representation \\Learning
}

\author{
Gabriele D'Acunto$^{*}$
\And
Enrico Grimaldi$^{*}$
\And
Valeria Avino
\AND
Mario Edoardo Pandolfo
\And
Leonardo Di Nino
\And
Sergio Barbarossa
\AND
{}
\And
\hspace{.5cm} Paolo Di Lorenzo}

\iclrfinalcopy 

\begin{document}

\maketitle

\begingroup
\renewcommand\thefootnote{*}
\footnotetext{These authors contributed equally. All authors are with Sapienza University of Rome, Italy.
Correspondence to: \texttt{gabriele.dacunto@uniroma1.it} and \texttt{enrico.grimaldi@uniroma1.it}.}
\endgroup

\begin{abstract}
Heterogeneous federated systems require agents to learn and exchange informative representations despite differences in data distributions, model architectures, and latent dimensionalities. We propose Sheaf-based Federated Representation Learning (SFRL), a framework that jointly learns and aligns heterogeneous latent representations without sharing parameters, labels, or a single latent space. SFRL relates agent-specific latent spaces through learnable orthogonal transformations and isometric embeddings, promoting consistency via a sheaf-Laplacian gluing regularizer evaluated on a shared set of pilot samples. We develop Sheaf-FRL, a decentralized alternating algorithm combining local gradient updates with closed-form Procrustes updates of the alignment maps, and establish convergence \blue{guarantees} in deterministic and stochastic settings.
\blue{On} collaborative classification \blue{with} semantic communication \blue{across four datasets}, under model and data heterogeneity, Sheaf-FRL \blue{improves communication accuracy over federated baselines at no cost in private accuracy}, and is more robust to latent-space compression.
\end{abstract}

\section{Introduction}\label{sec:introduction}

Modern machine learning systems increasingly operate in decentralized and heterogeneous environments, where data, models, sensing modalities, and computational resources are distributed across multiple agents, devices, or institutions. In this context, we adopt the term \emph{agent} to denote any physical or artificial entity performing a task in an environment. Privacy constraints, regulatory requirements, and communication limitations often prevent centralized data aggregation, motivating \emph{federated learning} (FL), where agents optimize local models on private data while exchanging limited information rather than raw samples~\citep{kairouz2021advances}.

A central challenge in such systems is that agents may learn representations that are useful locally but incompatible across the network. Existing FL methods commonly rely on parameter averaging, global regularization, or shared encoders, and therefore implicitly assume that features learned by different agents are directly comparable~\citep{DBLP:conf/aistats/McMahanMRHA17,scardapane2017framework,DBLP:conf/mlsys/LiSZSTS20,pmlr-v119-karimireddy20a,Collins2021FedRep,Li2021MOON,Dong2021FedMoCo}. This assumption can fail when agents differ in their data distributions, sensing modalities, architectures, latent dimensions, or local objectives. Independently trained encoders may produce representations that encode equivalent task-relevant information while differing by rotations, reflections, or more general isometric embeddings. Such ambiguity reflects the inherent \emph{gauge freedom} of representation learning objectives, whereby equivalent solutions may differ by transformations that preserve geometric structure~\citep{wang2020understanding}. Consequently, forcing all agents into a single shared latent space can be unnecessarily restrictive and may degrade local representation quality or hinder knowledge transfer.

These observations motivate \emph{semantic alignment}: rather than requiring agents to adopt identical latent coordinates, one should learn how their local representation spaces are related. This perspective is supported by recent work on relative representations, which shows that independently learned latent spaces can support zero-shot communication when their relational structure is preserved~\citep{moschella2023relative}. 
Semantic alignment can support heterogeneous federated representation learning across supervised, semi-supervised, and self-supervised settings, as well as cooperative tasks in which agents exploit representations received from their neighbors. For example, it is relevant to multimodal sensing, collaborative robotics, distributed channel charting, transfer across heterogeneous models, and task-aware semantic communication, where latent information generated by one agent must remain useful to another despite differences in their internal representations.

A natural mathematical framework for these settings is a \emph{cellular sheaf} on graph, hereafter referred to as a \emph{network sheaf}~\citep{curry2014sheaves}. A network sheaf associates a local vector space with each node and linear restriction maps with graph edges, thereby encoding how information in neighboring spaces should be compared or transported. Recent work in graph signal processing and graph neural networks has shown that sheaf-based models can represent vector-valued data coupled through local consistency constraints~\citep{Hansen2018SpectralSheaves,di2024learning,di2025learning,Bodnar2022NeuralSheafDiffusion,Barbero2022SheafLaplacians}. This provides a natural abstraction for federated representation learning: each agent maintains its own latent space, while learnable edge-wise maps encode the geometric relations required to align representations across neighboring agents. Global consistency then emerges from local compatibility relations, rather than from a shared latent space.

\textbf{Contributions.}
We propose a sheaf-based framework for federated representation learning that enables semantic alignment across heterogeneous agents without requiring a shared latent space. Specifically, 
\emph{(i)} we formulate federated representation learning over a learnable network sheaf, where agent-specific latent spaces are related through edge-wise geometric transport maps; 
\emph{(ii)} we introduce a quadratic gluing penalty induced by a learnable sheaf Laplacian, evaluated on a shared set of reference representations, hereinafter \emph{pilots}, to promote alignment while limiting communication; 
\emph{(iii)} we employ orthogonal and Stiefel restriction maps to capture rotations, reflections, and isometric embeddings across latent spaces of possibly different dimensions; 
\emph{(iv)} we develop a fully decentralized alternating algorithm, termed Sheaf-FRL, that combines local gradient updates with closed-form Procrustes updates of the transport maps, and establish convergence guarantees in deterministic and stochastic settings;
and \emph{(v)} we demonstrate \blue{on four datasets} that Sheaf-FRL improves \blue{representation transfer over baseline approaches, at no cost in local accuracy,} for collaborative supervised classification \blue{with} semantic communication, under model and data heterogeneity.
The code implementing the proposed method and reproducing all the experiments is publicly available.
\footnote{\href{https://github.com/SPAICOM/sheaf-based-federated-representation-learning.git}{https://github.com/SPAICOM/sheaf-based-federated-representation-learning.git}}

\section{Related Work}\label{sec:related}

\spara{Federated learning under heterogeneity.}
\emph{Federated learning} (FL) trains a shared model through repeated local stochastic gradient updates followed by either server-side aggregation\nb\citep{DBLP:conf/aistats/McMahanMRHA17} or neighbor-based combination over a communication graph\nb\citep{scardapane2017framework}, with stabilizing mechanisms such as proximal regularization\nb\citep{DBLP:conf/mlsys/LiSZSTS20} and control variates\nb\citep{pmlr-v119-karimireddy20a} addressing the client drift induced by statistical and system heterogeneity.
\emph{Federated multi-task learning} (FMTL) instead trains client-specific models while exploiting relatedness across tasks or clients\nb\citep{smith2017federated}, most prominently through graph-based regularization, where Laplacian penalties enforce smoothness of \emph{model parameters} across neighboring clients\nb\citep{dinh2022new,issaid2025tackling}.
\emph{Multimodal FL} (MFL)\nb\citep{che2023multimodal} targets clients holding distinct modalities; recent graph-based methods such as CoMFed\nb\citep{badi2026} handle heterogeneity by aligning compressed latent prototypes, yet still require a predefined common embedding space.
Finally, \emph{federated representation learning} (FRL) seeks transferable representations via auxiliary consistency losses on latent embeddings\nb\citep{Li2021MOON}, split architectures with a common encoder and client-specific heads\nb\citep{Collins2021FedRep,mclaughlin2024personalized}, or prototype alignment\nb\citep{tan2022fedproto,tran2024fedntproto}; analogous ideas extend to unsupervised and self-supervised regimes through contrastive objectives\nb\citep{zhuang2021fedu,han2022fedx,zhang2023furl,miao2024contrastive,ghalkha2026sheafalign}; a further line instead distils or correlates model outputs on shared public data\nb\citep{li2019fedmd, huang2024federated}.
These paradigms differ mainly in \emph{what} is assumed to be shared among agents (a single global model, a coupled set of parameters, a common multimodal embedding, a shared latent space, or a public dataset) rather than in the underlying goal of exploiting relatedness. 
We instead learn edge-wise geometric transport maps relating agent-specific latent spaces and promote their semantic compatibility. 
We recover each of these paradigms as a special case; App.\nb\ref{app:ext-related} expands on this positioning.

\spara{Representation alignment, synchronization, and manifold methods.}
Alignment up to orthogonal transformations arises naturally when representations are identifiable only up to rotations or reflections, with the orthogonal Procrustes problem providing a classical closed-form solution~\citep{Schonemann1966Procrustes}. 
We leverage this solution in a federated and decentralized setting, both under homogeneous latent spaces (cf. App.\nb\ref{app:algorithms}) and under stringent privacy constraints (cf. App.\nb\ref{app:tradeoff}). 
Related ideas arise in group synchronization and connection-Laplacian methods~\citep{Singer2009AngularSync,SingerWu2011VDM,bandeira2013cheeger,thunberg2017sync}, which enforce consistency across pairwise transformations; unlike these approaches, 
we jointly learn latent representations at the nodes and alignment maps on the edges. 
To accommodate heterogeneous latent dimensions, we further consider semi-orthogonal maps on the Stiefel manifold, enabling isometric embeddings across representation spaces. 
This connects our formulation to optimization on matrix manifolds~\citep{absil2008optimization,edelman1998geometry} and federated optimization under manifold constraints~\citep{li2022riemannianFL}; however, these works focus on optimizing model parameters subject to geometric constraints, rather than learning alignment maps as a component of representation learning.

\section{Framework and Problem Formulation}\label{sec:fpf}

\begin{figure}[t]
\centering
\resizebox{.85\linewidth}{!}{%
\begin{tikzpicture}[font=\small, >={Stealth[length=5pt]},
  sbox/.style={draw, rounded corners=3pt, line width=0.8pt, inner sep=4pt, outer sep=2pt,
               minimum width=2.8cm, minimum height=0.68cm, font=\small},
  estalk/.style={sbox, draw=cE, fill=cE!12, dash pattern=on 2.6pt off 1.8pt, minimum width=2.5cm},
  gnode/.style={circle, minimum size=8.5pt, inner sep=0pt, outer sep=2pt,
                line width=0.7pt, font=\small, text=black},
  gcell/.style={rectangle, minimum size=6pt, inner sep=0pt, line width=0.7pt, outer sep=2pt,
                fill=cE!25, draw=cE},
  assignO/.style={cA, dash pattern=on 2.4pt off 1.7pt, line width=0.75pt, ->},
  assignM/.style={cA, dash pattern=on 0.6pt off 1.5pt, line width=0.75pt, ->},
  incid/.style={cE, line width=0.42pt, double, double distance=1.1pt,
                -{Stealth[length=4pt,width=3.4pt]}},
  rmap/.style={line width=1pt, ->},
  flow/.style={black!50, line width=0.7pt, ->},
  panel/.style={draw=black!0, fill=none, rounded corners=6pt, line width=0.7pt, inner sep=13pt},
  ptitle/.style={anchor=north west, font=\small\bfseries, text=black!72, inner sep=0pt},
  callout/.style={text width=4.5cm, align=left, inner sep=0pt, text=black!85},
  manual/.style={align=left, inner sep=0pt, text=black!85}]
 
\begin{scope}[shift={(0.6,0.6)}]
\node[ptitle] at (-8.15,4.2) {Local model of agent $i$};
 
\datamatR{-7.37}{2.5}{8}{7}{0.22}{4}{3}{black!8}{black!45}
\draw[decorate,decoration={brace,amplitude=3pt},black!55,line width=0.5pt]
  (-7.37,3.44) -- (-5.83,3.44) node[midway,above=1pt,font=\scriptsize,black!60]{$p_i$};
\draw[decorate,decoration={brace,amplitude=3pt},black!55,line width=0.5pt]
  (-7.43,1.62) -- (-7.43,3.38) node[midway,left=2pt,font=\scriptsize,black!60]{$M_i$};
\node[font=\scriptsize,anchor=west] at (-5.77,3.3) {$\mathcal{T}_i \sim P_i$};
\draw[pil!72!black, line width=0.8pt, ->] (-6.6,1.55) -- (-6.6,1.25);
\pilotgrid{-7.37}{1.05}{1}{7}{0.22}
\node[font=\scriptsize,pil!72!black,anchor=west] at (-5.77,1.05)
     {$\x_i^{k},\; k\in\mathcal{A}$};
\node[font=\scriptsize,pil!72!black,anchor=east] at (-7.52,1.05)
    {pilot};
 
\draw[flow] (-6.6,0.92) -- (-6.6,0.62);
\vfunnel{1}{-6.6}{0.56}{0.77}{.33}{.5}{cI!14}{cI}
\vfunnel{1}{-6.6}{0.06}{0.33}{0.33}{0.26}{cI!40}{cI}
\node[font=\scriptsize,cI,anchor=east] at (-7.52,0.27) {$f_{\bm\varphi_i}$};
\node[font=\scriptsize,cI,anchor=east] at (-7.52,-0.15) {CenterNorm};
 
\draw[flow] (-6.6,-0.3) -- (-6.6,-0.6);
\pilotgrid{-6.93}{-1.2}{1}{3}{0.22}
\node[font=\scriptsize,pil!72!black,anchor=west] at (-6.8,-.85) {$\z_i^{k}$};
\draw[decorate,decoration={brace,mirror,amplitude=3pt},pil!72!black,line width=0.5pt]
  (-6.95,-1.35) -- (-6.25,-1.35) node[midway,below=1pt,font=\scriptsize,pil!72!black]{$d_i$};
\node[manual, anchor=north west, font=\small\linespread{0.95}\selectfont] at (-7.5,-1.1)
{\hangindent=1.05em\hangafter=1
  \textbf{(a)}};

\draw[pil!72!black, line width=0.8pt, ->] (-6.3,-0.7) to[out=60,in=120] (-5.20,-0.6);
 
\draw[flow] (-6.6,-1.8) -- (-6.6,-2.1);
\vfunnel{1}{-6.6}{-2.15}{0.33}{0.33}{0.26}{cI!14}{cI}
\node[font=\scriptsize,cI,anchor=east] at (-7.10,-2.4) {$h_{\bm\psi_i}$};
 
\draw[flow] (-6.6,-2.5) -- (-6.6,-2.8);
\node[font=\scriptsize,anchor=north] at (-6.6,-2.8) {$\mathcal{Y}_i$};


\coordinate (aNW) at (-8.20,3.60); \coordinate (aSE) at (-4.75,-2.35);
 
\coordinate (pi) at (-0.98,-0.52);
\coordinate (pj) at ( 0.98,-0.52);
\coordinate (pk) at ( 0.00, 0.80);
\node[draw=black!25, fill=black!4, rounded corners=3pt, line width=0.5pt,
      inner sep=10pt, fit=(pi)(pj)(pk)] (inset) {};
\node[above=1pt, black!55, font=\small] at ([xshift=-12pt]inset.north) {Oriented \graph};
\draw[cE, line width=0.9pt, ->, shorten >=4pt, shorten <=4pt] (pj) -- (pi);
\draw[cE, line width=0.9pt, ->, shorten >=4pt, shorten <=4pt] (pk) -- (pi);
\draw[cE, line width=0.9pt, ->, shorten >=4pt, shorten <=4pt] (pk) -- (pj);
\node[gcell] (sij) at ($(pi)!0.5!(pj)$) {$e_{ij}$};
\node[gcell] (sik) at ($(pi)!0.5!(pk)$) {$e_{ik}$};
\node[gcell] (sjk) at ($(pj)!0.5!(pk)$) {$e_{jk}$};
\node[gnode, fill=cI!22, draw=cI!85!black] (ci) at (pi) {$i$};
\node[gnode, fill=cJ!22, draw=cJ!85!black] (cj) at (pj) {$j$};
\node[gnode, fill=cK!22, draw=cK!85!black] (ck) at (pk) {$k$};
\draw[incid] (ci) to[out=-55,in=-125] node[midway](midinc){} (sij);
 
\node[sbox, draw=cI, fill=cI!10] (Fi) at (-4.25,-1.05) {\hspace{1cm}$\in \ns(i)\cong\reall^{3}$};
\node[sbox, draw=cJ, fill=cJ!10] (Fj) at ( 4.25,-1.05) {\hspace{1cm}$\in\ns(j)\cong\reall^{2}$};
\node[sbox, draw=cK, fill=cK!10] (Fk) at ( 0.00, 3.75) {\hspace{1cm}$\in\ns(k)\cong\reall^{2}$};
\pilotgrid{-5.25}{-1.05}{3}{3}{0.16}
\pilotgrid{ 3.25}{-1.05}{2}{3}{0.16}
\pilotgrid{-1.00}{ 3.75}{2}{3}{0.16}
\draw[decorate,decoration={brace,amplitude=3pt},pil!72!black,line width=0.5pt]
  (-5.25,-0.65) -- (-4.75,-0.65) node[midway,above=2pt,font=\scriptsize,pil!72!black]{$K$};
\node[right=3pt,font=\scriptsize,pil!72!black] at (-5.95,-1.05) {$\A_i$};
\node[right=3pt,font=\scriptsize,pil!72!black] at (2.55,-1.05) {$\A_j$};
\node[right=3pt,font=\scriptsize,pil!72!black] at (-1.7,3.75) {$\A_k$};
\node[manual, font=\small\linespread{0.95}\selectfont] at (-2,3.75)
{\hangindent=1.05em\hangafter=1
  \textbf{(b)}};
 
\node[estalk] (Eik) at (-2.55, 2.20) {$\ns(e_{ik})\cong\reall^{3}$};
\node[estalk] (Ejk) at ( 2.55, 2.20) {$\ns(e_{jk})\cong\reall^{2}$};
\node[estalk] (Eij) at ( 0.00,-3.15) {$\ns(e_{ij})\cong\reall^{3}$};
 
\draw[assignO] (ci) -- (Fi);
\draw[assignO] (cj) -- (Fj);
\draw[assignO] (ck) -- node[pos=0.72, solid, draw=cA, fill=white, rounded corners=2pt,
                  line width=0.6pt, inner sep=2pt, font=\scriptsize, text=cA]{$\ns(\cdot)$} (Fk);
\draw[assignO] (sik) -- (Eik);
\draw[assignO] (sjk) -- (Ejk);
\draw[assignO] (sij) -- (Eij);
 
\draw[rmap,draw=cI] (Fi) -- node[above left=-3pt,cI]{$\myO_{ki}$}  (Eik);
\draw[rmap,draw=cK] (Fk) -- node[above left=-4pt,cK]{$\V_{ik}$}    (Eik);
\draw[rmap,draw=cJ] (Fj) -- node[above right=-3pt,cJ]{$\myO_{kj}$} (Ejk);
\draw[rmap,draw=cK] (Fk) -- node[above right=-4pt,cK]{$\myO_{jk}$} (Ejk);
\draw[rmap,draw=cI] (Fi) -- node[pos=0.55,below left=-2pt,cI]
     {$\ns_{i\trianglelefteq e_{ij}}=\myO_{ji}$} (Eij);
\draw[rmap,draw=cJ] (Fj) -- node[below right=-3pt,cJ]{$\V_{ij}$}   (Eij);
\coordinate (onOji) at ($(Fi)!0.62!(Eij)$);
\draw[assignM] (midinc) to[out=-135,in=25] (onOji);
 
\draw[cE, dash pattern=on 1.5pt off 1.5pt, line width=0.5pt] (Eij.east) -- (3.5,-3.15);
\begin{scope}[shift={(5.5,-3)}, scale=1.15]
\begin{scope}[tdplot_main_coords]
  \fill[cJ!30, opacity=0.5] (-1.15,-1.15,0) -- (1.15,-1.15,0) -- (1.15,1.15,0) -- (-1.15,1.15,0) -- cycle;
  \foreach \g in {-1,-0.5,0,0.5,1}{
    \draw[cJ!60, line width=0.25pt] (\g,-1.15,0) -- (\g,1.15,0);
    \draw[cJ!60, line width=0.25pt] (-1.15,\g,0) -- (1.15,\g,0);}
  \draw[cJ!80, line width=0.5pt] (-1.15,-1.15,0) -- (1.15,-1.15,0) -- (1.15,1.15,0) -- (-1.15,1.15,0) -- cycle;
  \draw[->, cJ, line width=1.1pt] (0,0,0) -- (0.95,0.15,0)
        node[anchor=north west, cJ, inner sep=1pt, font=\scriptsize] {$\V_{ij}\z_j$};
  \draw[->, cI, line width=1.1pt] (0,0,0) -- (0.95,0.15,0.62)
        node[anchor=south west, cI, inner sep=1pt, font=\scriptsize] {$\myO_{ji}\z_i$};
  \draw[dashed, cI!70, line width=0.6pt] (0.95,0.15,0.62) -- (0.95,0.15,0);
  \draw[cJ!70, line width=0.4pt] (0.83,0.13,0) -- (0.83,0.13,0.11) -- (0.95,0.15,0.11);
  \node[anchor=east, cJ!80, font=\scriptsize] at (-1.25,0.55,0) {$\im{\V_{ij}}$};
\end{scope}\end{scope}
\node[manual, anchor=north west, font=\small\linespread{0.95}\selectfont] at (4.3,-3.6)
{\hangindent=1.05em\hangafter=1
  \textbf{(c)}};
 
\node[ptitle] at (5,3.8) {Network sheaf construction};
\begin{scope}[shift={(5.1,3.)}]
  \coordinate (legNW) at (-0.14,0.30); \coordinate (legSE) at (3.60,-3.06);
  \draw[black!30, rounded corners=3pt, line width=0.5pt, fill=white] (legNW) rectangle (legSE);
  \node[gnode, fill=black!12, draw=black!62] at (0.30,0.02) {};
  \node[gcell, minimum size=6pt, inner sep=0pt] at (0.66,0.02) {};
  \node[anchor=west, font=\small, black!75] at (0.92,0.02) {node and edge};
  \node[draw=black!62, fill=black!6, rounded corners=2pt, line width=0.8pt,
        minimum width=0.46cm, minimum height=0.24cm, inner sep=0pt] at (0.30,-0.44) {};
  \node[anchor=west, font=\small, black!75] at (0.92,-0.44) {node stalk};
  \node[draw=cE, fill=cE!12, rounded corners=2pt, line width=0.8pt,
        dash pattern=on 2.6pt off 1.8pt, minimum width=0.46cm, minimum height=0.24cm,
        inner sep=0pt] at (0.30,-0.90) {};
  \node[anchor=west, font=\small, black!75] at (0.92,-0.90) {edge stalk};
  \draw[incid] (0.08,-1.36) -- (0.56,-1.36);
  \node[anchor=west, font=\small, black!75] at (0.92,-1.36) {incidence $\trianglelefteq$};
  \draw[assignO] (0.08,-1.82) -- (0.56,-1.82);
  \node[anchor=west, font=\small, black!75] at (0.92,-1.82) {\ns\ on objects};
  \draw[assignM] (0.08,-2.28) -- (0.56,-2.28);
  \node[anchor=west, font=\small, black!75] at (0.92,-2.28) {\ns\ on morphisms};
  \draw[line width=1pt,->,black!72] (0.08,-2.74) -- (0.56,-2.74);
  \node[anchor=west, font=\small, black!75] at (0.92,-2.74) {restriction map};
\end{scope}
\coordinate (bSW) at (4.2,-3.5);
\end{scope}

\begin{scope}[on background layer]
  \node[panel, fit=(aNW)(aSE)] (panA) {};
  \node[panel, fit=(Fi)(Fj)(Fk)(Eik)(Ejk)(Eij)(inset)(legNW)(legSE)(bSW)] (panB) {};
\end{scope}
\end{tikzpicture}}
\caption{%
\emph{Sheaf-based federated representation learning} with three heterogeneous agents ($d_i=3$ and $d_j=d_k=2$).
\emph{Left:} the local model of agent $i$, from the local dataset $\mathcal{T}_i \sim P_i$ to the output space $\mathcal{Y}_i$, through the learnable encoder $f_{\bm\varphi_i}$ and head $h_{\bm\psi_i}$.
\blue{The encoder terminates in a center-and-normalize layer (CenterNorm), which maps every latent representation into the unit ball and provides the common geometric reference on which latent spaces are compared; the head operates directly on the normalized space.}
Only the $K$ pilots take part in the alignment, their latent representations forming the pilot matrix $\A_i\in\reall^{d_i\times K}$.
In our experiments the pilot pool is a held-out dataset, disjoint from every local pool, 
\blue{pilots enter training only through the gluing penalty,}
and the heterogeneity of the latent spaces is not diluted by shared data.
\emph{Right:}
the network sheaf $\ns$ assigns a stalk to every node and edge of the oriented graph $\graph$ (with $d_{ij}=\max(d_i,d_j)$), and a restriction map to every incidence relation, e.g.\ $\ns_{i\trianglelefteq e_{ij}}=\myO_{ji}$ (orthogonal map) and $\ns_{j\trianglelefteq e_{ij}}=\V_{ij}$ (isometric embedding).
Each edge is oriented from the lower- to the higher-dimensional stalk; on $e_{jk}$, the two stalks have equal dimension, so both maps are orthogonal and the orientation is arbitrary.
Three features are highlighted.
\textbf{(a)} Agents exchange latent representations, not model parameters or raw data.
\textbf{(b)} No single global latent space is assumed: global consistency may only emerge from the alignment of the agents' node stalks over the sheaf.
\textbf{(c)} The gluing penalty measures the residual $\norm{\myO_{ji}\z_i-\V_{ij}\z_j}_2^2$ on the edge stalk, i.e.\ the component of the transported representation orthogonal to $\im{\V_{ij}}$.%
}
\label{fig:sfrl}
\end{figure}
\vspace{-10pt}

\subsection{Local agents and network model}\label{subsec:agents}

Let $\graph = (\vertexset, \edgeset)$ denote a finite undirected graph with $|\vertexset| = N$ nodes, illustrated in \Cref{fig:sfrl}; the notation used throughout is summarized in App.\nb\ref{app:notation}.
Each node is an agent, while an edge $(i,j) \in \edgeset$ encodes a structural relationship (similarity in tasks, data distributions, modalities, or architectures) realized through a physical communication channel.

Each agent $i \in \vertexset$ observes a local dataset $\mathcal{T}_i$ comprising $M_i$ samples $\x_i^n \in \reall^{p_i}$, $n \in [M_i]$, drawn from an underlying distribution $P_i$, and learns a local representation model comprising a neural encoder
\begin{equation}\label{eq:local-encoder}
    f_{\bm\varphi_i}: \reall^{p_i}\to\ns(i)\,,
\end{equation}
where $\bm\varphi_i$ is the set of learnable parameters, and $\ns(i)\cong\reall^{d_i}$ is the latent representation space of agent $i$.
When required by the local task, the encoder is followed by a personalized head parametrized by $\bm\psi_i$, i.e.,
\begin{equation}\label{eq:local-decoder}
    h_{\bm\psi_i}: \ns(i)\to\mathcal{Y}_i\,,
\end{equation}
where $\mathcal{Y}_i$ is an agent-dependent output space.
We let $\bm\theta_i$ denote the collection of local trainable parameters, including $\bm\varphi_i$ and, when present, $\bm\psi_i$.
Each agent minimizes an arbitrary local learning objective $\mathcal{L}_i(\bm\theta_i)$ based on $\mathcal{T}_i$, e.g., the empirical risk $\mathcal{L}_i(\bm\theta_i)=M_i^{-1}\sum_{n=1}^{M_i}\ell_i(h_{\bm\psi_i}\circ f_{\bm\varphi_i}(\x_i^n))$ in the supervised setting.\footnote{%
Depending on the local paradigm, $\mathcal{T}_i$ may additionally include labels, side information, paired observations, or locally generated views of the samples.
Our framework is agnostic to the choice of $\mathcal{L}_i$, and thus accommodates semi-supervised, unsupervised, and self-supervised learning.}

\subsection{The learnable latent network sheaf}\label{subsec:sheaf}

To formalize the relationships, and thus the communication, among the latent representation spaces $\ns(i)$, we adopt the framework of network sheaves\nb\citep{bredon1997sheaf,curry2014sheaves}.
A \emph{network sheaf} $\ns$ on $\graph$ consists of the following assignments.
To each node $i \in \vertexset$ it associates a $d_i$-dimensional real vector space $\ns(i)$, the learnable latent space of agent $i$, referred to as \emph{node stalk}, whose valuation is the embedding $\z_i = f_{\bm\varphi_i}(\x_i)$.
We consider the general case in which the node stalks have different dimensionalities $d_i$.
To each edge $e_{ij} = (i,j) \in \edgeset$ it assigns a $d_{ij}$-dimensional real vector space $\ns(e_{ij})$, called \emph{edge stalk}, encoding pairwise compatibility between agents $i$ and $j$.
Stacking node and edge valuations yields the cochain spaces $C^0(\graph;\ns) = \bigoplus_{i \in \vertexset} \ns(i)$ and $C^1(\graph;\ns) = \bigoplus_{e_{ij} \in \edgeset} \ns(e_{ij})$.
Finally, to each incidence relation $i \trianglelefteq e_{ij}$ it specifies a linear \emph{restriction map} $\ns_{i \trianglelefteq e_{ij}} : \ns(i) \to \ns(e_{ij})$, which relates the local latent representation at node $i$ to the corresponding edge space and thereby enables communication among latent spaces.

We focus on learnable restriction maps that preserve distances and angles, so as to handle the well-known geometric invariance of latent representations\nb\citep{moschella2023relative} without amplifying latent perturbations.
These are orthogonal transformations and Stiefel matrices, where the orthogonal and Stiefel manifolds are
\begin{equation}\label{eq:manifolds}
    \ort{d}\coloneqq\{\myO \!\in\! \reall^{d \times d} \, \mid\, \myO^\top\!=\!\myO^{-1}\}\,, \quad \text{and} \quad
    \stiefel{d}{k}\coloneqq\{\V \!\in\! \reall^{d \times k} \, \mid\, \V^\top\V\!=\!\identity_k,\, k<d\}\,.
\end{equation}
Accordingly, for each edge $e_{ij}=(i,j) \in \edgeset$ with $d_i>d_j$ we set $\ns(e_{ij})=\ns(i)\cong \reall^{d_i}$, which implies $\ns_{i\trianglelefteq e_{ij}} \!=\! \myO_{ji} \in \ort{d_i}$ and $\ns_{j \trianglelefteq e_{ij}} \!=\! \V_{ij} \!\in\! \stiefel{d_i}{d_j}$, as depicted in \Cref{fig:sfrl}.
App.\nb\ref{app:add_discussion} discusses the benefits of the choice $d_{ij} = \max(d_i,d_j)$ as opposed to compressing.

\spara{Heterogeneous latent space geometry.}
Encoders are trained \blue{locally}, on different data and possibly with different architectures, so the induced latent spaces exhibit heterogeneous geometry (different scales and anisotropies) even when the dimensionalities $d_i$ coincide\nb\citep{pandolfo2026semasia, svendsen2026improving}.
Since isometric restriction maps measure inconsistency reliably only when the representations they compare adhere to a common geometric reference on the edge stalk, 
\blue{each encoder terminates in a center-and-normalize layer, $f_{\bm\varphi_i}=n_i\circ b_{\bm\varphi_i}$ with $n_i(\mathbf{u})=(\mathbf{u}-\tilde{\bm\mu}_i)/\max\{\|\mathbf{u}-\tilde{\bm\mu}_i\|_2,\epsilon\}$, where $b_{\bm\varphi_i}$ is the backbone, $\tilde{\bm\mu}_i$ combines the minibatch mean with a running mean as in momentum batch normalization, and $\epsilon\approx0$ is a small constant that prevents division by zero.
The layer has no trainable parameters and removes node-specific offsets and scales, mapping every latent representation onto the unit sphere; unlike whitening, it does not equalize anisotropies, which are left to the alignment.
Training examples and pilots alike pass through this layer, so the head operates directly on the normalized space and a representation transported from a neighbor is read with no inverse map.}
App.\nb\ref{app:add_discussion} details this construction.

\spara{Sheaf Laplacian and global sections.}
The sheaf induces a coboundary operator $\bm\Delta:C^0(\graph;\ns)\to C^1(\graph;\ns)$ defined edgewise as
\begin{equation}\label{eq:coboundary}
    (\bm\Delta\z)_{e_{ij}} = \myO_{ji}\z_i - \V_{ij}\z_j\,, \qquad e_{ij}=(i,j)\in \edgeset\,.
\end{equation}
Equipping $C^0$ and $C^1$ with the canonical Euclidean inner products makes the adjoint $\bm\Delta^\top$ well defined, and yields the symmetric positive semidefinite network sheaf Laplacian
\begin{equation}\label{eq:ns_Laplacian}
    \mathbf{L}_{\ns}=\bm\Delta^\top \bm\Delta\,.
\end{equation}
A valuation $\z^\star \in \ker(\mathbf{L}_\ns)$ is said \emph{global section}: it assigns latent representations to the node stalks without breaking any local consistency rule, i.e., $\myO_{ji}\z_i^\star = \V_{ij}\z_j^\star$ for each $e_{ij}\in\edgeset$.
Requiring the learned representations to be a global section of the learned sheaf would be the idealized notion of network-wide alignment.
Such a constraint, however, is unrealistic across heterogeneous agents, and couples all node representations simultaneously, thus precluding a decentralized solution; we state and discuss this centralized formulation in App.\nb\ref{app:add_discussion}.
This motivates the relaxed, distributed formulation developed next.

\subsection{Scalable gluing penalty and decentralized problem}\label{subsec:sfrl}

We replace the global-section requirement with a soft penalty based on the sheaf total variation, which promotes geometric consistency across agents without enforcing exact agreement:
\begin{equation}\label{eq:total_variation}
    \mathcal{TV}(\z) \coloneqq \norm{\bm\Delta \z}_2^2 = \z^\top\mathbf{L}_{\ns}\z= \sum_{e_{ij} \in \edgeset} \norm{\myO_{ji}\z_i - \V_{ij} \z_j}_2^2\,.
\end{equation}
Constraining $\z$ to be a global section is equivalent to zeroing \cref{eq:total_variation}.
Additional discussion on the geometric consistency promoted by the action of the restriction maps, as well as on the use of the total variation penalty, is provided in App.\nb\ref{app:add_discussion}.

\spara{Semantic embedding and orientation.}
Exploiting the transitive action of the orthogonal group on the Stiefel manifold, we reparameterize $\widetilde{\V}_{ij}\coloneqq\myO_{ji}^\top\V_{ij}$, so that each edge term in \cref{eq:total_variation} becomes
\begin{equation}\label{eq:local_gluing}
    \norm{\myO_{ji}\z_i - \V_{ij} \z_j}_2^2 = \norm{\z_i - \widetilde{\V}_{ij} \z_j}_2^2\,;
\end{equation}
hereafter we drop the tilde and write $\V_{ij}$ for the reparameterized map.
\Cref{eq:local_gluing} highlights that null total variation corresponds to the higher-dimensional latent space being an isometric embedding of the lower-dimensional one.
While this condition does not always hold, our restriction maps adhere to the \emph{semantic embedding principle}\nb\citep{d2025calsep,dacunto2026}: 
informally, the semantics of the lower-dimensional node stalk $\ns(j)$ is preserved when embedded into the higher-dimensional $\ns(i)$.  
The reparameterization also induces a natural \emph{embedding orientation} on $\graph$, from nodes with lower-dimensional latent spaces to those with higher-dimensional ones, i.e., along the direction in which semantics is preserved (\Cref{fig:sfrl}); we exploit it in the sequel to reduce the communication cost.
For $e_{ij}$ we let node $i$ be the head and $j$ the tail, consistently with $d_i \geq d_j$.
Ties $d_i=d_j=d$ are broken arbitrarily, since $\V_{ij}=\V_{ji}^\top \in \ort{d}$ makes the edge contribution orientation-invariant (cf.\ $e_{jk}$ in \Cref{fig:sfrl}).
On such edges the Stiefel constraints in Prob.\nb\ref{prob:sfrl} are understood as orthogonality constraints.

\spara{Incoming and outgoing embedding maps.}
For each $i \in \vertexset$, let $\vertexset(i)^-$ collect the neighbors $j$ for which $i$ is the head, and $\vertexset(i)^+$ those for which $i$ is the tail, so that $d_j\leq d_i$ on $\vertexset(i)^-$, $d_j\geq d_i$ on $\vertexset(i)^+$, with equality only on tied edges, and $\vertexset(i)=\vertexset(i)^-\cup\vertexset(i)^+$.
Symmetrizing \cref{eq:total_variation} over the two sets, the total variation decouples across nodes as $\mathcal{TV}(\z)=\sum_{i\in\vertexset}\mathcal{TV}(\z)\at{i}$, where the contribution of node $i$ is made of two terms with different geometric meanings:
\begin{equation}\label{eq:tv_at_i}
    \mathcal{TV}(\z)\at{i}=\dfrac{1}{2}\underbrace{\sum_{j \in \vertexset(i)^-} \norm{\z_i - \V_{ij} \z_j}_2^2}_{\text{Incoming Embedding}} + \dfrac{1}{2} \underbrace{\sum_{j \in \vertexset(i)^+} \norm{\z_j - \V_{ji} \z_i}_2^2}_{\text{Outgoing Embedding}}\,.
\end{equation}
Accordingly, the matrices $\V_{ij}$ in the first term are referred to as the \emph{incoming} embedding maps of node $i$, and the matrices $\V_{ji}$ in the second term as the \emph{outgoing} ones.
\cref{eq:tv_at_i} is what enables local updates and reduced communication cost in our decentralized algorithm.

\spara{Pilots and problem formulation.}
For the sake of exposition, we assume that a common index set $\{1,\dots,M\}$ of pilot inputs has been agreed upon, and we let $\mathcal{A}\subset\{1,\dots,M\}$, with $|\mathcal{A}|=K$, denote fixed pilot indices, so that agent $i$ holds the pilot feature matrix $\A_i \coloneqq \A_i(\bm\varphi_i) = [\z_i^k]_{k\in\mathcal{A}} \in \reall^{d_i\times K}$.
The pilot pool need not be part of the local training sets: it may consist of public data, calibration signals, or an overlapping portion of the local datasets (App.\nb\ref{app:pilots}).
Identifiability of the transport maps requires $K \geq \max_i d_i$ (App.\nb\ref{app:identifiability}), which lower-bounds the pilot budget.
Restricting \cref{eq:tv_at_i} to $\mathcal{A}$, we associate a positive regularization coefficient $\lambda_{ij}=\lambda_{ji}>0$ to each undirected edge $(i,j) \in \edgeset$, treated as an edge-wise hyperparameter and kept fixed throughout training. 
The gluing penalty \blue{is} $\mathcal{R}_{\mathcal{A}}=\blue{\frac{1}{2}}\sum_{i\in\vertexset} \mathcal{R}_{\mathcal{A}}\at{i}$, where
\begin{equation}\label{eq:local_gluing_penalty}
        \mathcal{R}_{\mathcal{A}}\at{i}(\bm\varphi_i, \{\V_{ij}\}, \{\V_{ji}\})
        =\dfrac{1}{2K}\Bigg[
        \sum_{j \in \vertexset(i)^-}\!\!\!\! \lambda_{ij}\frob{\A_i-\V_{ij}\A_j}^2 +
        \!\!\sum_{j \in \vertexset(i)^+}\!\!\!\! \lambda_{ij}\frob{\A_j-\V_{ji}\A_i}^2
        \Bigg]\,.
\end{equation}
\blue{The factor $\frac12$ is the usual normalization of a quadratic penalty; with it, the partial gradient of $\mathcal{R}_{\mathcal A}$ with respect to $\bm \varphi_i$ coincides with that of the local term $\mathcal{R}_{\mathcal A}|_i$, since $\bm\varphi_i$ enters the edge terms held by agent $i$ and those held by its neighbors equally (\cref{subsec:gradient}).}
Thus, the same edge coefficient $\lambda_{ij}$ weights the alignment error at both endpoints of an edge, while allowing different edges to have different regularization strengths.
Enforcing alignment on $\mathcal{A}$ only is a deliberately weak requirement: agreement on the pilots does not imply that the remaining local representations are perfectly alignable through the learned restriction maps, consistently with the fact that we neither assume nor aim for a single global latent representation.

\begin{problem}[Sheaf-based Federated Representation Learning]\label{prob:sfrl}
Given the graph $\graph$ endowed with the embedding orientation, the local datasets $\{\mathcal{T}_i\}_{i\in\vertexset}$, and the pilot indices $\mathcal{A}$, the goal of sheaf-based federated representation learning (SFRL) is to learn, for each node $i$, a local model $\bm\theta_i=(\bm\varphi_i,\bm\psi_i)$ together with incoming and outgoing embeddings $\{\V_{ij} \in \stiefel{d_i}{d_j}\}$ and $\{\V_{ji} \in \stiefel{d_j}{d_i}\}$ aligning $\ns(i)$ with the latent spaces of its neighbors, by solving in a decentralized manner
\begin{equation}
\label{eq:global_problem}
\min_{\substack{\{\bm\theta_i=(\bm\varphi_i,\bm\psi_i)\}\\
\{\V_{ij} \in \stiefel{d_i}{d_j}\}\\
\{\V_{ji} \in \stiefel{d_j}{d_i}\}}}
\quad
\sum_{i\in \vertexset}
\mathcal{L}_i(\bm\theta_i)
+ \blue{\frac{1}{2}} \sum_{i\in \vertexset} \mathcal{R}_{\mathcal{A}}\at{i}(\bm\varphi_i, \{\V_{ij}\}, \{\V_{ji}\})\,;
\tag{SFRL}
\end{equation}
\blue{without exchanging raw data or model parameters, and with reduced}
communication costs. 
\end{problem}

Prob.\nb\eqref{eq:global_problem} is a regularized nonconvex optimization problem, which couples smooth local neural parameters optimization with manifold-constrained edge transports.
It can be interpreted as a scalable decentralized relaxation of the exact global-section formulation reported in App.\nb\ref{app:add_discussion}.

\section{The Sheaf-FRL Algorithm}\label{sec:algorithmic_solution}

We solve Prob.\nb\eqref{eq:global_problem} in a decentralized manner via alternating minimization, consisting of the following two steps.

\subsection{Step 1: Isometric Embedding Update}\label{subsec:embedding-update}

Given the neural parameters $\{\bm\theta_i^t=(\bm\varphi_i^t,\bm\psi_i^t)\}$ at communication round $t$, let $\A_{i}^{t} \coloneqq \A_{i}(\bm\varphi_i^{t})$ denote the pilot feature matrix of agent $i$, \blue{normalized as in \cref{eq:cn}}.
Since the gluing penalty in \cref{eq:local_gluing_penalty} separates over the oriented edges, each node $i$ updates its embedding maps by solving two symmetric alignment problems, involving only its neighbors:
\begin{align}
    \V_{ij}^{t} &= \argmin_{\V_{ij} \in \stiefel{d_i}{d_j}} \;\; \frob{\A_{i}^{t} - \V_{ij} \A_j^{t}}^2 \,, \quad \forall\, j \in \vertexset(i)^- \,,
    \tag{P1a}\label{prob:alignment_V_ij}\\
    \V_{ji}^{t} &= \argmin_{\V_{ji} \in \stiefel{d_j}{d_i}} \;\; \frob{\A_{j}^{t} - \V_{ji} \A_i^{t}}^2 \,, \quad \forall\, j \in \vertexset(i)^+ \,.
    \tag{P1b}\label{prob:alignment_V_ji}
\end{align}
Expanding the Frobenius norm in \eqref{prob:alignment_V_ij}, the terms $\tr{\A_{i}^{t}(\A_{i}^{t})^\top}$ and $\tr{\A_{j}^{t}(\A_{j}^{t})^\top}$ do not depend on $\V_{ij}$, so the problem reduces to maximizing $\tr{\A_{i}^{t}(\A_{j}^{t})^\top\V_{ij}^\top}$.
Denoting by $\U \bm\Sigma \W^\top$ the thin SVD of the cross-covariance $\A_{i}^{t}(\A_{j}^{t})^\top$, the solution to \eqref{prob:alignment_V_ij} is
\begin{equation}\label{eq:procrustes_solution_V}
    \V_{ij}^{t} = \U \W^\top\,, \quad \text{where} \quad \U \in \stiefel{d_i}{d_j} \text{ and } \W \in \ort{d_j}\,;
\end{equation}
and the solution to \eqref{prob:alignment_V_ji} follows analogously from the thin SVD of $\A_{j}^{t}(\A_{i}^{t})^\top$.
Each node computes \cref{eq:procrustes_solution_V} locally, once its neighbors have shared their pilot features $\{\A_{j}^{t}\}_{j\in\mathcal{N}(i)}$, using the orientation to distinguish the edges on which it acts as the head from those on which it acts as the tail.
In the homogeneous case, as well as on tied edges of the heterogeneous setting, $\stiefel{d_i}{d_j}$ coincides with $\ort{d}$, and \eqref{prob:alignment_V_ij}--\eqref{prob:alignment_V_ji} reduce to canonical orthogonal Procrustes problems whose solution follows from the full SVD by setting $\myO_{ij}^{t}=\U \W^\top$ and $\myO_{ji}^{t}=\myO_{ij}^{{t}\top}$, so that tied edges require no separate treatment.
Conditions ensuring identifiability of the transport maps, in both the homogeneous and heterogeneous settings, are discussed in App.\nb\ref{app:identifiability}.

\subsection{Step 2: Gradient Update of Neural Parameters}\label{subsec:gradient}

Each agent $i$ updates its local parameters $\bm\theta_i$ through a gradient step on the objective of Prob.\nb\eqref{eq:global_problem}, while keeping the neighboring parameters $\{\bm\theta_j^t\}_{j\in\mathcal N(i)}$ fixed at the current iteration.
The regularization term depends exclusively on the encoder parameters $\bm\varphi_i$.
Differentiating \cref{eq:local_gluing_penalty}, and adopting for $j \in \vertexset(i)^+$ the convention $\V_{ij}\coloneqq\V_{ji}^\top$, which satisfies $\V_{ij}\V_{ji}=\identity_{d_i}$ by orthonormality of the columns of $\V_{ji}$, the incoming and outgoing contributions admit the unified expression
\begin{equation}\label{eq:gradient_gluing_penalty}
\nabla_{\bm\varphi_i}\mathcal R_{\mathcal{A}}\at{i}=\nabla_{\bm\varphi_i}\mathcal R_{\mathcal{A}}
= \frac{1}{K} \sum_{j\in \mathcal N(i)}\lambda_{ij}\sum_{k\in \mathcal{A}}\left(\nabla_{\bm\varphi_i}f_{\bm\varphi_i}(\x_i^{k})\right)^\top \left(f_{\bm\varphi_i}(\x_i^{k})-\V_{ij}f_{\bm\varphi_j}(\x_j^{k})\right)\,,
\end{equation}
where $\nabla_{\bm\varphi_i} f_{\bm\varphi_i}(\x_i^{k}) \in \reall^{d_i \times |\bm\varphi_i|}$ denotes the Jacobian of the feature map with respect to $\bm\varphi_i$, evaluated at $\x_i^{k}$.
The full derivation is reported in App.\nb\ref{sec:Convergence_analysis}.
Letting $\mathbf{r}_i^t \in \reall^{|\bm\theta_i|}$ denote \cref{eq:gradient_gluing_penalty} evaluated at round $t$ and zero-padded on the head coordinates $\bm\psi_i$, the neural parameters are updated locally via
\begin{equation}\label{eq:theta_update_Q}
\bm\theta_i^{t+1} = \bm\theta_i^{t} - \eta_\theta\left(\nabla_{\bm\theta_i}\mathcal{L}_i(\bm\theta_i^t) + \mathbf{r}^t_i\right)\,,
\end{equation}
where $\eta_\theta>0$ is the stepsize.
These computations are fully decentralized: agent $i$ evaluates \cref{eq:theta_update_Q} using only locally available information together with the pilot representations $\{\A_j^t\}_{j\in\mathcal N(i)}$ already received in Step~1, so that no additional communication is required.
The resulting algorithm, termed \emph{Sheaf-FRL}, is summarized in App.\nb\ref{app:algorithms} for both the heterogeneous and homogeneous settings.

\blue{
\spara{Confidence-gated alignment.}
In \cref{eq:gradient_gluing_penalty} both endpoints of an edge move towards each other on every pilot. 
We instead let the endpoint that is more competent on a pilot act as a fixed target for the other.
For each edge $e_{ij}$ and pilot $k$, let $c_i^k$ denote the maximum logit of the head $h_{\bm\psi_i}$ on $\z_i^k$, standardized over the current minibatch of pilots, and let
\begin{equation}\label{eq:gate}
  a_{ij}^k=\sigma\big((c_i^k-c_j^k)/\tau\big),
\end{equation}
with temperature $\tau>0$, treated as a constant in the gradient.
The gated update replaces $\mathbf r_i^t$ in \cref{eq:theta_update_Q} by
\begin{equation}\label{eq:gated}
  \mathbf r_i^t=\frac1K\sum_{j\in \mathcal{N}(i)}\lambda_{ij}\sum_{k\in\mathcal A}w_{ij}^k\big(\nabla_{\bm\varphi_i}f_{\bm\varphi_i}(\x_i^k)\big)^\top\big(f_{\bm\varphi_i}(\x_i^k)-\V_{ij}f_{\bm\varphi_j}(\x_j^k)\big);
\end{equation}
where $w_{ij}^k=1-a_{ij}^k$ for $j\in \mathcal{N}(i)^-$, and $a_{ji}^k$, for $j\in \mathcal{N}(i)^+$.
\Cref{eq:gated} is the gradient of $\mathcal{R}_{\mathcal A}$ in which, on each pilot, the representation of the more confident endpoint is held fixed.
Since $a_{ij}^k+(1-a_{ij}^k)=1$, the gate does not change the value of $\mathcal{R}_{\mathcal{A}}$: the penalty-free configurations, and hence the global sections and Step~1, are unchanged, and only the gradient is redistributed between the endpoints.
For $a_{ij}^k=\frac12$, \cref{eq:gated} is half of \cref{eq:gradient_gluing_penalty}, and symmetric alignment is recovered up to a rescaling of $\lambda_{ij}$.
This variant of Sheaf-FRL is termed \emph{DiSheaf-FRL}.}

\blue{\spara{Convergence analysis.}}
In App.\nb\ref{sec:Convergence_analysis} we \blue{analyze} Sheaf-FRL \blue{in the} standard-privacy setting\blue{, in both the deterministic (full-batch) and the stochastic (minibatch) case.
Since the CenterNorm layer normalizes with a combination of batch and running statistics, the objective drifts as the running mean is updated, by an amount of order $\alpha$ per iteration, where $\alpha$ is the momentum.}
In the deterministic case\blue{, where the full-batch mean is used ($\alpha=1$),} the objective decreases monotonically and $\min_{0\leq t\leq T}\norm{\nabla_{\bm\theta}J(\bm\theta^t,\V^t)}^2=\mathcal{O}(1/T)$, whereas with unbiased stochastic gradients of bounded conditional variance a constant stepsize yields $\mathcal{O}(1/\sqrt{T})$ in expectation\blue{, up to a neighbourhood of order $\sqrt{\alpha}$ induced by the running mean, which vanishes for pure minibatch statistics ($\alpha=1$).
With a summable momentum schedule or with $\alpha=1$,} Robbins--Monro stepsizes yield almost-sure convergence.
\blue{Beyond boundedness from below of $J$, the analysis requires smoothness, which holds as long as the representations stay away from their normalization means, where the clamp of the CenterNorm layer is inactive, compactness of the feasible set, which holds since Stiefel manifolds are compact, and exact minimization over $\V$ at every iteration, which \eqref{prob:alignment_V_ij}--\eqref{prob:alignment_V_ji} provide in closed form.
The analysis covers symmetric alignment with plain (stochastic) gradient steps; the confidence gate and the strict-privacy variant of App.\nb\ref{app:tradeoff} are not covered.}

\spara{Discussion.}
\emph{Communication)} At each iteration node $i$ broadcasts $\mathbf{A}^t_i \!\in\! \mathbb{R}^{d_i\! \times \!K}$ to its neighbors, so that the \emph{outgoing} volume per agent scales as $\mathcal{O}(|\vertexset(i)| d_i K)$. The \emph{incoming} volume received by agent $i$ instead scales as $\mathcal{O}(K \sum_{j \in \mathcal{N}(i)} d_j)$, which can exceed the outgoing payload whenever neighbors hold higher-dimensional latent spaces ($d_j \gg d_i$); only pilot representations are exchanged in either direction, while model parameters and labels remain local (cf. \cref{tab:tradeoff} for the full breakdown).
\emph{Privacy)} Sharing latent representations may expose agents to reconstruction and inference attacks\nb\citep{fredrikson2015model,shokri2017membership,melis2019exploiting}; this is mitigated by exchanging only transported representations, at the price of additional local computation and communication (App.\nb\ref{app:tradeoff} adapts Sheaf-FRL for strict-privacy settings).
\emph{Local computation)} The number of alignment updates could in principle be halved, since the isometric embedding is only required at the lower-dimensional endpoint; we instead exchange raw pilot features and let both endpoints compute their maps locally, shifting the burden from communication to computation, and the trade-off improves further in the homogeneous case, where the number of SVDs can be halved (Apps.\nb\ref{app:algorithms} and \ref{app:tradeoff}).
\emph{Accelerated methods)} Momentum-based and adaptive schemes, such as Nesterov acceleration or Adam, can be incorporated into \cref{eq:theta_update_Q} without affecting the decentralized structure of the algorithm\blue{, although they are not covered by the analysis of App.~\ref{sec:Convergence_analysis}}.

\section{Collaborative supervised classification with semantic communication}\label{sec:sem_comm}
We evaluate Sheaf-FRL on a collaborative and decentralized classification task in the setting of semantic, \blue{or goal-oriented}, communication\nb\cite{gunduz2022beyond,barbarossa2023semantic},
\blue{whose objective is no longer the faithful reconstruction of the transmitted data, but the delivery of the information the receiver needs to accomplish its task.
This paradigm is regarded as a key enabler of AI-native 6G networks\nb\citep{strinati2024goal,pandolfo2025latent,grimaldi2025learning}.
A transmitter observes data that the receiver does not have, and conveys their meaning through a compact latent representation rather than through the raw observation; the receiver then performs its task directly on what it receives.
Since transmitter and receiver are trained locally, possibly with different architectures, their latent spaces are not compatible a priori, and a receiver cannot interpret the representation of a neighbor unless the two spaces are aligned.
In our setting, each agent acts both as a transmitter, sending the latent representations of its observations to its neighbors, and as a receiver, classifying the representations it receives with its own head.}
Each agent minimizes a local cross-entropy loss $\mathcal{L}_i(\bm\theta_i)\coloneq\mathrm{CE}(\bm\theta_i)$, while representation alignment across agents is promoted via the sheaf regularization term.
Accordingly, we monitor two metrics, both averaged across agents: the \emph{private accuracy}, i.e., the accuracy of agent $i$ on the test set through its \mulberry{local latent representations}, and the \emph{communication accuracy}, i.e., the accuracy of the head $h_{\bm\psi_i}$ on the \teal{latent representations produced by the neighbors} $j\in\vertexset(i)$ and transported into the latent space of $i$.
\blue{For Sheaf-FRL and DiSheaf-FRL the transport is given by the learned restriction maps, and for ComFed by the maps learned by the method, whereas for the remaining baselines it is the closed-form Procrustes map of \cref{eq:procrustes_solution_V}, fitted on the pilot representations at the end of training.
Since each agent trains only on its target classes, the communication accuracy of agent $i$ is computed on the test samples whose class belongs to its target set $\mathcal{C}_i$, i.e., on the classes it is able to classify.
Within these, the samples in $\mathcal{C}_i\setminus\mathcal{C}_j$, whose class the sender $j$ has never been trained on, isolate the information that is actually transferred through communication, and we also report the accuracy on them.}
Formal definitions are reported in App.\nb\ref{app:exp_details}.

\blue{\spara{Setup.}
We consider four benchmarks with $N=4$ agents, on MNIST\nb\citep{lecun1998gradient}, Fashion-MNIST\nb\citep{xiao2017fashion}, FEMNIST\nb\citep{caldas2018leaf} and CIFAR-10\nb\citep{krizhevsky2009learning}.
Every agent has a convolutional encoder followed by global average pooling, the CenterNorm layer of \cref{eq:cn}, and an MLP head.
Within each benchmark the agents differ in architecture family, depth and width, and each has a different latent dimension $d_i$, so that every restriction map is rectangular; for instance, on CIFAR-10 the agents combine narrow ResNet-10, ResNet-18 and VGG-style encoders with $d_i\in\{8,16,32,64\}$ (App.\nb\ref{app:exp_details}).
Data heterogeneity is induced by label shift: each agent $i$ trains only on its target class set $\mathcal{C}_i\subset\mathcal{Y}$, with an equal number of samples per class, and misses the remaining classes.
On FEMNIST we add a style shift, each agent seeing only the samples of its own writers.
The pilot set $\mathcal{A}$, with $K=4096$ samples ($2048$ on Fashion-MNIST), is drawn uniformly at random from the training pool, over all classes, and removed from it before the local datasets are drawn, so that it is disjoint from every local dataset (cf.\ App.\nb\ref{app:pilots}).
Agents communicate over a fully connected graph.
All methods are evaluated on the official test split, shared by all agents, except on FEMNIST, where the test set is a held-out portion of the agents' writers (App.\nb\ref{app:exp_details}).
Results are reported as mean and standard deviation over $3$ seeds.}

\spara{Baselines.}
Model and submodel averaging\nb\citep{scardapane2017framework,diao2021heterofl,li2021lotteryfl}, as well as shared-representation methods\nb\citep{Collins2021FedRep,liang2020think}, assume homogeneous architectures or a shared latent space, and are therefore not applicable here.
We thus compare against \blue{five} methods designed for heterogeneous settings, \emph{ComFed}\nb\citep{badi2026}, \emph{Sheaf-FMTL}\nb\citep{issaid2025tackling}, \emph{FedProto}\nb\citep{tan2022fedproto}, \emph{FedMuscle}\nb\citep{setayesh2026toward}\blue{, and \emph{FedRE}\nb\citep{yao2026fedre}}.
\blue{ComFed projects every latent space onto a shared space, whose dimension we set to $d_{\min}=\min_id_i$, while FedProto and FedMuscle, which assume a common latent dimension, are extended to heterogeneous latent dimensions by a learned projection onto a shared space of dimension $d_{\max}=\max_id_i$, so that no agent is compressed (App.\nb\ref{app:exp_details}).}
We further add a \emph{non-cooperative} baseline in which models are trained independently, providing a natural lower bound.
\blue{All methods use the same encoders, CenterNorm layer, optimizer, communication topology and pilot set, and both Sheaf-FRL and DiSheaf-FRL are reported.}

\begin{table}[t]
\centering
\scriptsize
\setlength{\tabcolsep}{3pt}
\renewcommand{\arraystretch}{1.2}
\caption{\blue{Task accuracy (\%), mean and standard deviation over $3$ seeds. Private accuracy is computed on the target classes $\mathcal{C}_i$ of each agent; communication accuracy on the classes known by the receiver $i$: overall ($\mathcal{C}_i$), unseen by the sender $j$ ($\mathcal{C}_i\setminus\mathcal{C}_j$), and shared with the sender ($\mathcal{C}_i\cap\mathcal{C}_j$). Per dataset and communication column, methods are ranked in groups of statistically indistinguishable results, i.e., whose intervals mean $\pm$ s.d.\ overlap: the best group is in bold, the second is underlined, and the third is in italics. The private accuracy is not marked, as all methods are tied.}}
\vspace{-4pt}
\label{tab:main}
\vspace{4pt}
\resizebox{.8\linewidth}{!}{%
\begin{tabular}{l*{8}{c}}
\toprule
 & \multicolumn{4}{c}{CIFAR-10} & \multicolumn{4}{c}{FEMNIST} \\
\cmidrule(lr){2-5}\cmidrule(lr){6-9}
 & Priv. & \multicolumn{3}{c}{Comm.} & Priv. & \multicolumn{3}{c}{Comm.} \\
\cmidrule(lr){2-2}\cmidrule(lr){3-5}\cmidrule(lr){6-6}\cmidrule(lr){7-9}
 & Overall & Overall & RX-only seen & Shared & Overall & Overall & RX-only seen & Shared \\
Method & $\mathcal{C}_i$ & $\mathcal{C}_i$ & $\mathcal{C}_i\!\setminus\!\mathcal{C}_j$ & $\mathcal{C}_i\!\cap\!\mathcal{C}_j$ & $\mathcal{C}_i$ & $\mathcal{C}_i$ & $\mathcal{C}_i\!\setminus\!\mathcal{C}_j$ & $\mathcal{C}_i\!\cap\!\mathcal{C}_j$ \\
\midrule
Non-coop & $92.0_{\pm .21}$ & \rankC{73.0_{\pm .11}} & $36.2_{\pm .80}$ & \rankA{91.3_{\pm .23}} & $99.5_{\pm .04}$ & $93.9_{\pm .04}$ & $85.3_{\pm .12}$ & \rankA{99.3_{\pm .05}} \\
FedRE & $92.0_{\pm .13}$ & \rankC{73.1_{\pm .11}} & \rankC{41.4_{\pm .86}} & \rankC{88.9_{\pm .29}} & $99.6_{\pm .03}$ & \rankC{94.0_{\pm .04}} & $85.7_{\pm .18}$ & \rankC{99.2_{\pm .06}} \\
FedProto & $91.9_{\pm .14}$ & \rankC{73.1_{\pm .25}} & $37.6_{\pm .27}$ & \rankA{90.9_{\pm .36}} & $99.5_{\pm .03}$ & \rankC{94.0_{\pm .44}} & $85.6_{\pm 1.18}$ & \rankC{99.2_{\pm .06}} \\
FedMuscle & $91.9_{\pm .13}$ & \rankC{73.1_{\pm .22}} & $37.8_{\pm 1.02}$ & \rankB{90.7_{\pm .26}} & $99.5_{\pm .03}$ & \rankB{94.7_{\pm .02}} & \rankC{87.2_{\pm .03}} & \rankA{99.4_{\pm .05}} \\
ComFed & $91.9_{\pm .03}$ & $56.9_{\pm 2.20}$ & $8.6_{\pm 1.46}$ & $81.0_{\pm 3.11}$ & $99.5_{\pm .03}$ & $56.5_{\pm 2.94}$ & $8.8_{\pm 3.21}$ & $86.0_{\pm 2.80}$ \\
Sheaf-FMTL & $92.0_{\pm .20}$ & \rankC{73.1_{\pm .20}} & $37.5_{\pm .70}$ & \rankA{90.9_{\pm .20}} & $99.5_{\pm .00}$ & \rankC{94.0_{\pm .30}} & $85.5_{\pm .80}$ & \rankB{99.2_{\pm .10}} \\
Sheaf-FRL & $91.6_{\pm .14}$ & \rankB{79.5_{\pm .54}} & \rankB{61.2_{\pm 2.74}} & \rankC{88.6_{\pm 1.10}} & $99.5_{\pm .03}$ & \rankA{99.2_{\pm .04}} & \rankB{99.0_{\pm .05}} & \rankA{99.4_{\pm .04}} \\
DiSheaf-FRL & $91.7_{\pm .17}$ & \rankA{82.1_{\pm .42}} & \rankA{68.3_{\pm .35}} & \rankC{89.0_{\pm .59}} & $99.5_{\pm .04}$ & \rankA{99.2_{\pm .02}} & \rankA{99.1_{\pm .01}} & \rankB{99.3_{\pm .03}} \\
\midrule
 & \multicolumn{4}{c}{MNIST} & \multicolumn{4}{c}{Fashion-MNIST} \\
\cmidrule(lr){2-5}\cmidrule(lr){6-9}
 & Priv. & \multicolumn{3}{c}{Comm.} & Priv. & \multicolumn{3}{c}{Comm.} \\
\cmidrule(lr){2-2}\cmidrule(lr){3-5}\cmidrule(lr){6-6}\cmidrule(lr){7-9}
 & Overall & Overall & RX-only seen & Shared & Overall & Overall & RX-only seen & Shared \\
Method & $\mathcal{C}_i$ & $\mathcal{C}_i$ & $\mathcal{C}_i\!\setminus\!\mathcal{C}_j$ & $\mathcal{C}_i\!\cap\!\mathcal{C}_j$ & $\mathcal{C}_i$ & $\mathcal{C}_i$ & $\mathcal{C}_i\!\setminus\!\mathcal{C}_j$ & $\mathcal{C}_i\!\cap\!\mathcal{C}_j$ \\
\midrule
Non-coop & $99.6_{\pm .01}$ & $93.8_{\pm .17}$ & \rankC{84.8_{\pm .47}} & \rankB{99.4_{\pm .04}} & $95.3_{\pm .07}$ & $71.2_{\pm 1.24}$ & $56.3_{\pm .52}$ & \rankB{86.2_{\pm 1.98}} \\
FedRE & $99.5_{\pm .04}$ & $93.2_{\pm .44}$ & $83.3_{\pm 1.11}$ & \rankB{99.4_{\pm .03}} & $95.4_{\pm .03}$ & $70.3_{\pm .72}$ & $53.7_{\pm 1.45}$ & \rankB{87.0_{\pm 2.81}} \\
FedProto & $99.6_{\pm .02}$ & $94.0_{\pm .28}$ & \rankC{85.2_{\pm .70}} & \rankA{99.5_{\pm .02}} & $95.3_{\pm .14}$ & $70.8_{\pm .52}$ & $55.8_{\pm 1.64}$ & \rankB{85.8_{\pm 1.10}} \\
FedMuscle & $99.6_{\pm .02}$ & \rankC{95.2_{\pm .18}} & \rankB{88.4_{\pm .50}} & \rankA{99.5_{\pm .02}} & $95.2_{\pm .04}$ & \rankC{77.0_{\pm 1.12}} & \rankC{65.2_{\pm .59}} & \rankB{88.9_{\pm 2.00}} \\
ComFed & $99.6_{\pm .01}$ & $47.7_{\pm 1.34}$ & $13.8_{\pm 2.81}$ & $68.8_{\pm 1.02}$ & $95.3_{\pm .09}$ & $49.9_{\pm 1.03}$ & $13.4_{\pm 1.54}$ & \rankB{86.5_{\pm 1.64}} \\
Sheaf-FMTL & $99.6_{\pm .03}$ & $93.5_{\pm .59}$ & \rankC{84.0_{\pm 1.57}} & \rankB{99.4_{\pm .04}} & $95.3_{\pm .10}$ & $70.6_{\pm 2.00}$ & $55.3_{\pm 1.00}$ & \rankB{85.9_{\pm 3.90}} \\
Sheaf-FRL & $99.6_{\pm .01}$ & \rankA{99.3_{\pm .04}} & \rankA{99.1_{\pm .10}} & \rankB{99.4_{\pm .01}} & $95.3_{\pm .05}$ & \rankB{91.1_{\pm .13}} & \rankB{88.0_{\pm .51}} & \rankA{94.3_{\pm .29}} \\
DiSheaf-FRL & $99.6_{\pm .03}$ & \rankB{99.2_{\pm .02}} & \rankA{99.1_{\pm .09}} & \rankC{99.3_{\pm .03}} & $95.3_{\pm .03}$ & \rankA{92.4_{\pm .14}} & \rankA{90.5_{\pm .33}} & \rankA{94.3_{\pm .54}} \\
\bottomrule
\end{tabular}
}
\vspace{-4pt}
\end{table}

\spara{Results.}
\blue{\cref{tab:main} reports the private and communication accuracy on the four benchmarks.
On the private task all methods are statistically indistinguishable.
Here, the cooperation does not degrade the local performance of the agents, whichever the method.
On the communication task, instead, the two variants of Sheaf-FRL outperform all baselines on every benchmark, with statistically significant gain on the overall communication accuracy.
The gap between the best-performing variant and the best baseline is about $9$ points on CIFAR-10, $5$ on FEMNIST, $4$ on MNIST and $15$ on Fashion-MNIST.
To understand where this gain comes from, we split the communication accuracy according to whether the class of a test sample is known only by the receiver or by both agents.
Notably, on the classes known only by the receiver, where the sender has never been trained, the gap reaches about $27$ points on CIFAR-10, $12$ on FEMNIST, $11$ on MNIST and $25$ on Fashion-MNIST.
On the classes known by both agents, the results are mixed across benchmarks, which is expected.
Indeed, on these classes every agent already attains a high private accuracy, and a post-hoc alignment is enough to transfer them to a neighbor.
The value added by Sheaf-FRL thus lies precisely where communication has to transfer knowledge that the receiver could not obtain from the sender otherwise, and it is there that it improves the interoperability between the agents of the network.
}

\blue{\spara{Alignment and robustness to compression.}
On all benchmarks, both variants also attain the best alignment of the latent spaces, as measured by the mean reciprocal rank, the total variation, and the smallest eigenvalue of the sheaf Laplacian (App.\nb\ref{app:alignment}).
On FEMNIST with two agents, Sheaf-FRL variants show to be robust to the bottleneck compression. 
The gap on the classes unseen by the sender grows from about $14$ to about $40$ points as the latent dimension shrinks from $64$ to $8$ (App.\nb\ref{sec:compression}).
}

\blue{\spara{Ablation studies and emergence of global sections.}
App.\nb\ref{app:ablation} isolates the contribution of the main design choices of DiSheaf-FRL, being the best performing variant.
In more detail: 
\emph{(i)} the number of available pilots,
\emph{(ii)} the Stiefel geometry, compared with unconstrained linear maps,
\emph{(iii)} the regularization coefficient, 
\emph{(iv)} the geometric reference provided by the CenterNorm layer, compared with analytic and learnable whitening,
and \emph{(v)} the communication topology, comparing star, ring and the fully connected graphs used in this section. 
App.\nb\ref{app:global-sections} quantifies how far the learned sheaf is from a shared latent space, through the spectrum of the sheaf Laplacian induced by the learned restriction maps, and the total variation \blue{on the test set}.}
\section{Limitations}\label{sec:limitations}
Our formulation is agnostic to the local objective, but we validate it only on supervised classification over \blue{four single-modality datasets}, so its behaviour under unsupervised, self-supervised, and multi-modal objectives remains a conjecture.
The communication graph and the edge coefficients $\lambda_{ij}$ are fixed throughout training, whereas the sheaf formulation would naturally accommodate adapting them to the inconsistency actually observed on each edge.
We presume a known correspondence between pilot inputs across agents, which is realistic for calibration signals or public data but excludes settings in which the correspondence must itself be inferred.
\blue{
On the theoretical side, our guarantees concern optimization only; bounds on the generalization of the learned representations and on the transport error of the maps beyond the pilots, as well as the analysis of the strict-privacy variant, are left to future work.
Finally, we do not evaluate privacy risks empirically. 
The transmitted latent representations remain exposed to inversion and inference attacks, in particular at inference time, and assessing them requires a dedicated threat model that we leave to future work.}
\section{Conclusions and Future Work}
\label{sec:conclusions}

We introduced \emph{Sheaf-based Federated Representation Learning}, which aligns heterogeneous agents through learnable edge-wise transport maps rather than a shared global latent space, and \emph{Sheaf-FRL}, a fully decentralized alternating algorithm that solves it by coupling local gradient updates with closed-form Procrustes updates of the restriction maps.
For the standard-privacy case, we establish convergence to first-order stationary points in the deterministic setting \blue{and, up to a neighbourhood induced by the running normalization statistics, in the stochastic setting.}
On cooperative classification \blue{with} semantic communication \blue{over four datasets}, Sheaf-FRL outperformed baselines designed for heterogeneous settings in \blue{post-communication accuracy at no cost in local accuracy}, and degraded more gracefully under latent-space compression, while exchanging only pilot representations, never parameters or data.
That a weak, pilot-restricted notion of alignment suffices suggests that global consistency need not be imposed on heterogeneous agents, but can be left to emerge from local compatibility relations.
The limitations above set the agenda for future work. The formulation also points to a natural extension: the sheaf total variation decomposes edge by edge, so the disagreement it measures could itself drive the adaptation of the edge coefficients, or of the graph.
\clearpage
\subsection*{AI use statement}




In this work, we used generative AI tools to aid in and polish writing, to aid in code development, and to improve proofs readability.
Additionally, we used generative AI tools to aid in the development of diagrams and tables, and to improve the layout and readability of figures. 

We have not used generative AI tools in research ideation, in the design and development of the overall framework, in the formulation of mathematical claims, in the design of the proof strategies, in the design of the experimental settings, or in the interpretation of results.
The rest of the required disclosure tasks are not applicable to this work.

We have reviewed all AI-assisted work. 
The final paper content has been double-checked and refined by more than two authors.
The generated code was verified and tested throughout the experimental development and execution. 
We take responsibility for the final content of this work.

\subsection*{Ethics statement}



Sheaf-FRL exchanges latent representations rather than raw data or model parameters, which limits direct exposure of local data but does not eliminate it: latent features remain susceptible to model inversion, membership inference, and attribute inference\nb\citep{fredrikson2015model, shokri2017membership, melis2019exploiting}.
App.\nb\ref{app:tradeoff} addresses this by a variant in which agents exchange only transported representations, and Prop.\nb\ref{prop:privacy_nonidentifiability} shows that under unknown restriction maps the underlying pilots are identifiable only up to an orthogonal transformation.
A consideration specific to our setting is that the pilot pool is shared by construction: every agent observes the same inputs, so the pilot set should be drawn from public data, calibration signals, or standardized inputs, and not from records that any single agent is obliged to keep private.
Our experiments use MNIST and involve no human subjects or personally identifiable information.

\subsection*{Reproducibility statement}


The code implementing Sheaf-FRL and reproducing every experiment is available at the anonymized repository linked in \Cref{sec:introduction}.
App.\nb\ref{app:algorithms} reports the pseudocode for all four protocol variants, and App.\nb\ref{app:exp_details} specifies the agent architectures, the construction of the communication graph, the data partitioning and splits, the evaluation metrics, and the hyperparameter selection procedure.
The assumptions and full proofs of the convergence results are in App.\nb\ref{sec:Convergence_analysis}.


\subsubsection*{Acknowledgments}
The work was supported by the SNS JU project 6G-GOALS~\citep{strinati2024goal} under the EU’s Horizon program Grant Agreement No 101139232, and by Huawei Technology France SASU under Grant N. Tg20250616041.

\bibliography{references}
\bibliographystyle{iclr2027_conference}

\clearpage
\appendix
\section{Notation}\label{app:notation}

Scalars are denoted by lowercase italic letters, vectors by lowercase boldface, matrices by uppercase boldface, and sets, spaces, and operators by calligraphic letters.
Vectors of learnable parameters are denoted by boldface Greek letters.

\begin{table}[h]
\centering
\small
\begin{tabular}{l l}
\toprule
\multicolumn{2}{l}{\emph{Sets and indices}} \\
\midrule
$\reall$, $[N]$ & real numbers; the index set $\{1,\dots,N\}$ \\
$|\mathcal{S}|$ & cardinality of the set $\mathcal{S}$ \\
$\graph=(\vertexset,\edgeset)$ & communication graph, with $|\vertexset| = N$ agents \\
$e_{ij}=(i,j)\in\edgeset$ & edge oriented \blue{with} head $i$ \blue{and} tail $j$ \\
$i \trianglelefteq e_{ij}$ & incidence relation between node $i$ and edge $e_{ij}$ \\
$\vertexset(i)^-,\ \vertexset(i)^+$ & neighbors for which $i$ is the head, resp.\ the tail \\
$\vertexset(i)$ & $\vertexset(i)^- \cup \vertexset(i)^+$ \\
\midrule
\multicolumn{2}{l}{\emph{Data and local models}} \\
\midrule
$\mathcal{T}_i$, $M_i$ & local dataset of agent $i$ and its size \\
$\x_i^n \in \reall^{p_i}$, $P_i$ & $n$-th local sample and its distribution \\
$\mathcal{Y}_i$, $\mathcal{C}_i$ & local output space and target class set \\
$\mathcal{A}$, $K$ & pilot index set and its cardinality $|\mathcal{A}| = K$ \\
$f_{\bm\varphi_i}$, $h_{\bm\psi_i}$ & local encoder and head, with parameters $\bm\varphi_i$, $\bm\psi_i$ \\
$\bm\theta_i = (\bm\varphi_i,\bm\psi_i)$ & \blue{trainable} parameters of agent $i$ \\
$b_{\bm\varphi_i},\,n_i$ & backbone and CenterNorm layer, $f_{\bm\varphi_i}=n_i\circ b_{\bm\varphi_i}$ \\
$\tilde{\bm\mu}_i,\ \mathbf{m}_i,\ \alpha,\ \epsilon$ & normalization mean, running mean, momentum and clamping threshold of $n_i$ \\
$\mathcal{L}_i$, $\mathcal{R}_\mathcal{A}$, \blue{$\mathcal{R}_\mathcal{A}\at{i}$} & local objective, gluing penalty, \blue{and its local view at node $i$} \\
$\lambda_{ij}$, $\eta_\theta$ & edge regularization coefficient and stepsize \\
\midrule
\multicolumn{2}{l}{\emph{Sheaf structure}} \\
\midrule
$\ns$ & network sheaf on $\graph$ \\
$\ns(i) \cong \reall^{d_i}$ & node stalk of agent $i$, of dimension $d_i$ \\
$\ns(e_{ij}) \cong \reall^{d_{ij}}$ & edge stalk, with $d_{ij}=\max(d_i,d_j)$ \\
$\ns_{i \trianglelefteq e_{ij}}$ & restriction map from $\ns(i)$ to $\ns(e_{ij})$ \\
$C^0(\graph;\ns)$, $C^1(\graph;\ns)$ & spaces of $0$- and $1$-cochains \\
$\bm\Delta$, $\mathbf{L}_\ns = \bm\Delta^\top\bm\Delta$ & coboundary operator and sheaf Laplacian \\
$\mathcal{TV}(\z)$ & sheaf total variation \\
\midrule
\multicolumn{2}{l}{\emph{Latent quantities and transport maps}} \\
\midrule
$\z_i = f_{\bm\varphi_i}(\x_i)$ & latent representation of agent $i$ \\
$\A_i \in \reall^{d_i \times K}$ & pilot feature matrix of agent $i$ \\
$\myO_{ji} \in \ort{d_i}$ & orthogonal restriction map \\
$\V_{ij} \in \stiefel{d_i}{d_j}$ & isometric embedding (semi-orthogonal map) \\
$\ort{d}$, $\stiefel{d}{k}$ & orthogonal group and Stiefel manifold, Eq.\nb(3) \\
\midrule
\multicolumn{2}{l}{\emph{Operators}} \\
\midrule
$\norm{\cdot}_2$, $\norm{\cdot}_{\mathrm{F}}$ & Euclidean and Frobenius norms \\
$(\cdot)^\top$, $\tr\{\cdot\}$ & transpose and trace \\
$\ker(\cdot)$, $\im{\cdot}$ & kernel and image of a linear map \\
$\oplus$ & direct sum of vector spaces \\
$\mathbf{I}_k$ & $k \times k$ identity matrix \\
$\ones[\cdot]$ & indicator function \\
$\nabla_{\bm\varphi_i}$ & gradient with respect to $\bm\varphi_i$ \\
\bottomrule
\end{tabular}
\end{table}

\section{Extended related work and positioning}\label{app:ext-related}

Here we expand the discussion of Sec.~\ref{sec:related}, detailing how our framework relates to each of the paradigms mentioned there.

\spara{Federated learning under heterogeneity.} 
A large body of work studies the impact of statistical and system heterogeneity, which can lead to client drift and unstable convergence, and proposes stabilizing mechanisms such as proximal regularization (FedProx)\nb\citep{DBLP:conf/mlsys/LiSZSTS20} and control variates (SCAFFOLD)\nb\citep{pmlr-v119-karimireddy20a}. 
More broadly, these approaches aim to improve optimization stability while maintaining a \emph{single global model}. 
In contrast, we focus on settings where enforcing a shared latent representation is inherently restrictive.

\spara{Federated multi-task learning.} 
Early formulations build upon regularized multi-task learning objectives, where task relationships are captured through shared covariance or relational structures\nb\citep{evgeniou2004regularized,evgeniou2005learning,jacob2008clustered,zhang2010convex,zhang2017learning}. 
In the federated setting, this leads to approaches that jointly optimize local models and their coupling structure via alternating or primal-dual procedures\nb\citep{smith2017federated,jaggi2014communication,ma2015adding}. 
A prominent class of methods leverages graph-based regularization, where a predefined or learned graph encodes similarities among clients and enforces smoothness of model parameters across the network. 
In particular, Laplacian-based approaches penalize discrepancies between neighboring models through total variation terms\nb\citep{dinh2022new,issaid2025tackling}. 
While these methods capture task relatedness, they operate in the \emph{parameter space} and typically assume homogeneous representations. 
In contrast, our approach operates in the \emph{representation space} and models relationships between clients up to learnable transformations, enabling alignment across heterogeneous latent spaces and accommodating different architectures or dimensions.

\spara{Multimodal Federated Learning.}
Existing MFL approaches are typically categorized based on data heterogeneity into congruent (shared modalities and feature spaces) and incongruent settings, the latter including vertical, transfer, and hybrid scenarios with partially overlapping or distinct modalities. 
In general, MFL optimizes a global objective by aggregating client-specific multimodal losses, often defined as weighted combinations of modality-wise contributions. 
However, most methods assume a shared latent space and rely on centralized or weakly decentralized protocols. 
Recent graph-based approaches such as CoMFed\nb\citep{badi2026} address heterogeneity by aligning compressed latent prototypes across clients, but still require a predefined common embedding space and enforce similarity via projections. 
In contrast, our sheaf-based framework does not assume a global latent space; instead, it enables decentralized learning in which each client maintains its own representation, and cross-client consistency is enforced through edge-dependent geometric alignments without information loss, naturally accommodating heterogeneous representation spaces, including settings where different modalities admit a set of semantically corresponding pilot inputs.

\spara{Federated Representation Learning.}
In the supervised setting, existing approaches either enforce consistency via auxiliary losses on latent embeddings\nb\citep{Li2021MOON}, typically assuming a shared global model, or adopt a split architecture with a common encoder and client-specific heads\nb\citep{Collins2021FedRep,mclaughlin2024personalized}, still requiring a shared latent dimensionality. 
More recent methods based on prototype alignment\nb\citep{tan2022fedproto,tran2024fedntproto} partially address heterogeneity but continue to rely on a common embedding space and often centralized training. 
In contrast, our framework allows each client to learn its own latent space and enforces consistency in a decentralized manner through structured transport maps, so that classical FRL methods can be interpreted as a special case in which all latent spaces share the same dimensionality and are constrained to coincide via identity alignments.
Beyond supervised settings, unsupervised and self-supervised approaches operating on unlabeled data aim to improve representation transferability across clients, often leveraging contrastive learning frameworks\nb\citep{zhuang2021fedu,han2022fedx,zhang2023furl,miao2024contrastive,ghalkha2026sheafalign}.
Our proposed framework is not limited to any specific local loss or model architecture; thus, it is also amenable to the latter settings. 

\spara{Distillation and correlation-based alignment on public data.}
A further line of heterogeneous FL achieves interoperability by evaluating all participants on a shared set of inputs rather than by sharing parameters.
FedMD\nb\citep{li2019fedmd} assumes a labeled public dataset and has a central server average the participants' class scores into a consensus that each model then distils towards, while FCCL+\nb\citep{huang2022learn,huang2023generalizable} removes the labeling requirement and aligns, again through the server, both the cross-correlation matrix of the logits and the instance-similarity matrix of a public batch towards their network-wide averages; broader coverage of this literature is given in\nb\citet{huang2024federated}.
Our pilot set plays the same role of a common reference, but these methods are centralized and drive every participant towards a single global target, whereas we impose pairwise compatibility on the edges of a communication graph and never posit a global latent space.
They further rely on the labels of the public data, or on a common label space for the logits, while our gluing penalty uses the pilot inputs alone and compares latent representations of possibly different dimensions.
Most importantly, FCCL+ accommodates differing feature dimensions precisely by comparing the instance-similarity matrix, which is invariant to any orthogonal transformation of the feature space: it constrains representations up to exactly the gauge freedom we exploit, without ever constructing the transformation, so that a participant still cannot consume a neighbor's latent vector after training, whereas our restriction maps are the object being learned and are what makes transport possible.

\spara{Positioning of this paper.}
Overall, our framework generalizes across all paradigms discussed above, which differ mainly in \emph{what} is assumed to be shared among agents (a single global model, a coupled set of parameters, a common multimodal embedding, or a shared latent space) rather than in the underlying goal of exploiting relatedness across agents.
Indeed, methods introduced within one paradigm, such as prototype alignment in FRL or graph-based regularization in FMTL, are often applicable to the others as well.
Rather than enforcing any of these forms of sharing, we learn edge-wise geometric transport maps that relate neighboring latent spaces and promote their semantic compatibility.
In the homogeneous case, our formulation becomes compatible with classical FL and recovers its shared-representation structure as a limiting case, together with orthogonal alignment, synchronization, and connection-Laplacian models.
More generally, it extends graph-based FMTL and Laplacian regularization through a \emph{learnable sheaf Laplacian} acting on latent representations rather than model parameters.
Additionally, it relaxes the shared, predefined embedding space assumed by MFL approaches such as CoMFed, and recovers FRL methods as the special case in which all latent spaces coincide via identity alignments.
Finally, pilot-restricted gluing enables communication-efficient alignment, while manifold-constrained transport maps admit decentralized optimization through local model updates and closed-form edge-wise Procrustes steps.
\section{Pseudocode of the Sheaf-FRL Algorithm}
\label{app:algorithms}

In this section, we provide additional details on the proposed decentralized algorithm, reported in Alg.~\ref{alg:sheaf_frl_1comm}. The method follows an alternating minimization strategy, where each communication round consists of local feature extraction, decentralized alignment updates, and gradient-based optimization of the neural parameters.
\begin{algorithm}[t]
\caption{Sheaf-FRL}
\label{alg:sheaf_frl_1comm}
\begin{algorithmic}[1]
\REQUIRE Oriented $\graph=(\vertexset,\edgeset)$, pilots $\mathcal A$, edge weights $\{\lambda_{ij}>0\}$, $\eta_\theta>0$, iterations $T$
\STATE Initialize $\{\boldsymbol{\theta}_i^0=(\boldsymbol{\varphi}_i^0,\boldsymbol{\psi}_i^0)\}_{i\in\vertexset}$

\FOR{$t=0,\dots,T-1$}

\FORALL{$i\in\vertexset$ \textbf{in parallel}}
  \STATE $\mathbf{A}_i^t \gets [f_{\boldsymbol{\varphi}_i^t}(\x_i^{k})]_{k\in\mathcal A}$
\ENDFOR

\STATE Each node $i$ broadcasts $\mathbf{A}_i^t$ to \blue{all} $j\in\mathcal N(i)$

\FORALL{$i\in\vertexset$ \textbf{in parallel}}
  \FOR{$j \in \mathcal N(i)$}
    \IF{$j\in\mathcal N(i)^-$}
      \STATE $[\mathbf{U},\mathbf{\Sigma},\mathbf{W}^\top]\gets\mathrm{thinSVD}(\mathbf{A}_i^t\mathbf{A}_j^{t^\top})$
      \STATE $\mathbf{V}_{ij}^{t}\gets \mathbf{U}\mathbf{W}^\top$
    \ELSE
      \STATE $[\mathbf{U},\mathbf{\Sigma},\mathbf{W}^\top]\gets\mathrm{thinSVD}(\mathbf{A}_j^t\mathbf{A}_i^{t^\top})$
      \STATE $\mathbf{V}_{ji}^{t}\gets \mathbf{U}\mathbf{W}^\top$
    \ENDIF
  \ENDFOR
\ENDFOR

\FORALL{$i\in\vertexset$ \textbf{in parallel}}
  \STATE $\boldsymbol{\theta}_i^{t+1}
  \gets
  \boldsymbol{\theta}_i^t
  -\eta_\theta\left(
  \nabla_{\boldsymbol{\theta}_i}\mathcal L_i(\boldsymbol{\theta}_i^t)
  +
  \mathbf{r}^t_i
  \right)\,; \quad \text{with $\mathbf{r}^t_i$ as in \blue{\cref{eq:gradient_gluing_penalty} evaluated at $t$}}$
\ENDFOR

\ENDFOR
\end{algorithmic}
\end{algorithm}
\paragraph{Heterogeneous case.} Alg.~\ref{alg:sheaf_frl_1comm} implements a fully decentralized procedure in which agents iteratively update both their local representations and the alignment maps across the network. 
At each iteration $t$, every agent $i$ first computes its pilot feature matrix $\mathbf{A}_i^t$ using the current encoder (lines 3–5). These pilot representations provide a compact summary of the local latent space and constitute the only information exchanged across agents. The pilot features are then broadcast to neighboring nodes (line 6). Upon receiving $\{\mathbf{A}_j^t\}_{j \in \mathcal{N}(i)}$, each agent locally updates the alignment maps associated with its incident edges (lines 7–17). In particular, for each neighbor $j$, agent $i$ solves a small Procrustes problem based on the cross-covariance of the corresponding pilot features. The orientation of the graph determines whether the update corresponds to an incoming map $\mathbf{V}_{ij}^t$ or an outgoing map $\mathbf{V}_{ji}^t$, ensuring consistency with the underlying sheaf structure. Finally, each agent performs a gradient-based update of its local parameters $\boldsymbol{\theta}_i$ (lines 18–20), combining the gradient of the local loss with the contribution of the sheaf-based regularization. This step enforces alignment of neighboring representations while preserving local task objectives. Overall, the algorithm alternates between closed-form alignment updates and gradient-based parameter updates. All computations are carried out locally, and communication is limited to low-dimensional pilot features, resulting in an efficient and scalable decentralized learning procedure.

\begin{algorithm}[t]
\caption{Sheaf-FRL (homogeneous case)}
\label{alg:sheaf_frl_hom}
\begin{algorithmic}[1]
\REQUIRE Oriented $\graph=(\vertexset,\edgeset)$, pilots $\mathcal A$, edge weights $\{\lambda_{ij}>0\}$, $\eta_\theta>0$, iterations $T$
\STATE Initialize $\{\boldsymbol{\theta}_i^0=(\boldsymbol{\varphi}_i^0,\boldsymbol{\psi}_i^0)\}_{i\in\vertexset}$ and $\{\myO_{ji}^0  =\eye{d}\}_{j \in \vertexset(i), i \in \vertexset}$ 

\FOR{$t=0,\dots,T-1$}

\FORALL{$i\in\vertexset$ \textbf{in parallel}}
  \STATE $\mathbf{A}_i^t \gets [f_{\boldsymbol{\varphi}_i^t}(\mathbf{x}_i^{k})]_{k\in\mathcal A}$
\ENDFOR

\STATE Each node $i$ exchanges $\myO_{ji}^{t}\mathbf{A}_i^t$ with all $j\in\mathcal N(i)$

\FORALL{$i\in\vertexset$ \textbf{in parallel}}
  \FORALL{$j\in\vertexset(i)^+$}
    \STATE $[\mathbf{U},\mathbf{\Sigma},\mathbf{W}^\top]
    \gets
    \mathrm{SVD}\big(\myO_{ij}^{t}\mathbf{A}_j^t\mathbf{A}_i^{t^\top}\big)$
    \STATE $\myO_{ji}^{t}\gets \mathbf{U}\mathbf{W}^\top$
  \ENDFOR
\ENDFOR

\FORALL{$i\in\vertexset$ \textbf{in parallel}}
  \STATE $\boldsymbol{\theta}_i^{t+1}
  \gets
  \boldsymbol{\theta}_i^{t}
  -\eta_\theta\left(
  \nabla_{\boldsymbol{\theta}_i}\mathcal L_i(\boldsymbol{\theta}_i^{t})
  +
  \mathbf{r}_i^{t}
  \right)$
\ENDFOR

\ENDFOR

\end{algorithmic}
\end{algorithm}

\spara{Homogeneous case.} 
When the latent dimensions are homogeneous, i.e., $d_i = d$ for all $i \in \vertexset$, the alignment maps reduce to orthogonal transformations $\myO_{ij} \in \ort{d}$, for each $(i,j) \in \edgeset$. 
In this setting, as discussed in \cref{sec:algorithmic_solution} and further detailed in App.\nb\ref{app:tradeoff}, the local computational cost of the alignment step can be reduced without increasing communication. 
Alg.~\ref{alg:sheaf_frl_hom} presents a simplified Sheaf-FRL protocol tailored to this scenario. 
For each agent $i$, we initialize to the identity the maps used to transport its pilot representations toward its neighbors, i.e., $\myO_{ji}^0 = \mathbf{I}_d$ for all $j \in \mathcal{N}(i)$.
At each iteration $t$, each agent first computes its pilot representations $\mathbf{A}_i^t$ and sends to each neighbor $j \in \mathcal{N}(i)$ the transported features $\myO_{ji}^t \mathbf{A}_i^t$. 
Notably, since we exploit the reparameterization in \cref{eq:local_gluing_ort}, when $j \in \mathcal{N}^-(i)$, we have $\myO_{ji}^t = \myO_{ji}^0=\mathbf{I}_d$ for all $t=0, \ldots, T-1$, so that the transmitted features coincide with the original ones and no additional communication cost is incurred. 
From a local computation perspective, since transported pilot representations are communicated instead of raw ones, agent $i$ does not need to solve the orthogonal Procrustes problem arising in \eqref{prob:alignment_V_ij} for evaluating the alignment and the gradient required for the local parameter update.
As a consequence, only the SVDs for $j \in \vertexset(i)^+$ are required, effectively halving the computational cost of the alignment step, as specified in lines 8-10 of Alg.~\ref{alg:sheaf_frl_hom}.
Overall, this variant preserves the communication efficiency of the general algorithm while reducing the computational burden (cf.\nb\cref{tab:tradeoff}).
\section{Trade-off between Privacy, Communication, and Local Computation}\label{app:tradeoff}
Our proposed algorithmic solution in \cref{sec:algorithmic_solution} leverages sharing of raw latent representations to lower the communication cost, according to the objective in \Cref{prob:sfrl}.
This choice is deliberate.
In principle, the number of alignment updates could be halved, since for each edge $e_{ij}$ the isometric embedding is only required at the lower-dimensional node.
However, this would require each node $i$ to transmit its transported representations $\V_{ji}\A_i^t$, whose dimension is at least $d_i$, in order to evaluate \cref{eq:gradient_gluing_penalty}, thereby increasing the communication cost.
We therefore exchange the raw pilot features $\A_i^t$ and let both endpoints compute the corresponding restriction maps locally, shifting the burden from communication to computation, which is well aligned with the increasing computational capabilities of modern edge devices.
This is compatible with standard regulatory constraints on privacy, for instance when latent representations do not allow direct identification of individuals or when appropriate safeguards such as aggregation, anonymization, or secure communication protocols are in place, as commonly assumed in federated learning systems~\citep{kairouz2021advances,DBLP:conf/aistats/McMahanMRHA17}. 
However, as highlighted in \cref{sec:algorithmic_solution}, in some circumstances more stringent regulatory constraints on privacy impair the sharing of raw latent representations. For example, this occurs when latent representations may still encode sensitive or identifiable information and are therefore susceptible to reconstruction or inference attacks, such as model inversion~\citep{fredrikson2015model}, membership inference~\citep{shokri2017membership}, or attribute inference~\citep{melis2019exploiting}, making their direct exchange incompatible with strict privacy requirements. Our proposed approach naturally extends to this more stringent setting. Instead of sharing raw latent representations, agents exchange their transported counterparts, i.e., latent features mapped through the restriction maps associated with the edges. This limits the direct exposure of local features. However, this enhanced privacy comes at a cost: it requires additional local computation and increases the communication overhead, as detailed below.

Concerning the local computational cost, each agent $i \in \vertexset$ needs to encode its latent representations toward both sets of neighbors $\vertexset(i)^-$ and $\vertexset(i)^+$. 
To this end, we adopt the formulation in \Cref{eq:local_gluing} \emph{before reparameterization}, where both orthogonal and Stiefel maps explicitly appear, allowing consistent encoding across heterogeneous neighbors. 
Specifically, $i$ must communicate $\V_{ji}^t \A_i^t \in \reall^{d_j \times K}$ to neighbors $j \in \vertexset(i)^+$, where $\V_{ji}^t \in \stiefel{d_j}{d_i}$ is the outgoing embedding matrix, and $\myO_{ji}^t \A_i^t \in \reall^{d_i \times K}$ to neighbors $j \in \vertexset(i)^-$, where $\myO_{ji}^t \in \ort{d_i}$ is an orthogonal transformation. 
Hence, in this setting we cannot exploit the reparameterized form in \Cref{eq:local_gluing}, which avoids the explicit use of orthogonal maps. 
In detail, after receiving the encoded latent representations from its neighbors, namely
\begin{equation}
\V_{ij}^t\A_j^{t} \text{ from } j \in \vertexset(i)^- \quad \text{and}\quad \myO_{ij}^t\A_j^{t} \text{ from } j \in \vertexset(i)^+\,;
\end{equation}
under stringent privacy constraints, agent $i$ must solve the following problems to update its restriction maps:
\begin{align}
\myO_{ji}^{t} &= \argmin_{\myO_{ji} \in \ort{d_i}} \; \frob{\myO_{ji}\A_{i}^{t} - \V_{ij}^t \A_j^{t}}^2 \,, \quad \forall\, j \in \mathcal{N}(i)^- \,,
\tag{P1c}\label{prob:alignment_O_ij_strict_privacy}\\
\V_{ji}^{t} &= \argmin_{\V_{ji} \in \stiefel{d_j}{d_i}} \; \frob{\myO_{ij}^t\A_{j}^{t} - \V_{ji} \A_i^{t}}^2 \,, \quad \forall\, j \in \mathcal{N}(i)^+ \,.
\tag{P1d}\label{prob:alignment_V_ji_strict_privacy}
\end{align}

Unlike the case of standard privacy constraints, where for $j \in \vertexset(i)^-$ we can only update the incoming restriction map $\V_{ij}$ and solve Prob.\nb\eqref{prob:alignment_V_ij} via a thin SVD at a computational cost of $\mathcal{O}(d_i d_j^2)$ with $d_j \leq d_i$, Prob.\nb\eqref{prob:alignment_O_ij_strict_privacy} is a canonical orthogonal Procrustes problem. Its solution is obtained via a classical SVD of $\V_{ij}^t \A_j^t \A_i^{t^\top}$ at a cost of $\mathcal{O}(d_i^3)$, thus leading to an increase in the local computational burden.
Conversely, similarly to Problem~\eqref{prob:alignment_V_ji}, for $j \in \vertexset(i)^+$ the update of the outgoing restriction map $\V_{ji}$ can be computed via a thin SVD of $\myO_{ij}^t \A_j^t \A_i^{t^\top}$ at a cost of $\mathcal{O}(d_j d_i^2)$, where in this case $d_j \geq d_i$.

We remark that the costs above account only for the SVD factorization step. 
Constructing  the cross-covariance matrix $\A^t_j \A^{t^\top}_i \in  \reall^{d_j \times d_i}$ (or its transported counterpart under strict privacy), requires  an additional $\mathcal{O}(d_i d_j K)$ operations per neighbor, for both privacy regimes. 
Since the number of pilots $K$ is typically much larger than the latent dimensions in practice (cf. App.\nb\ref{app:exp_details}), this construction step dominates the overall local computational burden and is reported explicitly in \cref{tab:tradeoff}.

Please notice that more stringent privacy constraints affect only the update of the restriction maps.
Indeed, the local gradient computation in \Cref{eq:gradient_gluing_penalty} remains unchanged:
\begin{equation}\label{eq:gradient_gluing_penalty_privacy}
    \begin{aligned}
\nabla_{\boldsymbol{\varphi}_i}\mathcal R_{\mathcal{A}}\at{i}
=&\frac{1}{K}\Bigg[\sum_{j\in \vertexset(i)^-}\!\!\lambda_{ij}\sum_{k\in \mathcal{A}}\left( \nabla_{\boldsymbol{\varphi}_i}f_{\boldsymbol{\varphi}_i}(\mathbf{x}_i^{k})\right)^\top \myO_{ji}^\top\left(\myO_{ji} f_{\boldsymbol{\varphi}_i}\left(\mathbf{x}_i^{k}\right)-\V_{ij}f_{\boldsymbol{\varphi}_j}\left(\mathbf{x}_j^{k}\right)\right) \\
&- \sum_{j\in \vertexset(i)^+}\!\!\lambda_{ij}\sum_{k\in \mathcal{A}}\left(\nabla_{\boldsymbol{\varphi}_i}f_{\boldsymbol{\varphi}_i}(\mathbf{x}_i^{k})\right)^\top\V_{ji}^\top\left(\myO_{ij}f_{\boldsymbol{\varphi}_j}\left(\mathbf{x}_j^{k}\right)-\V_{ji}f_{\boldsymbol{\varphi}_i}\left(\mathbf{x}_i^{k}\right)\right)\Bigg]\\
=&\frac{1}{K}\Bigg[\sum_{j\in \vertexset(i)^-}\!\!\lambda_{ij}\sum_{k\in \mathcal{A}}\left( \nabla_{\boldsymbol{\varphi}_i}f_{\boldsymbol{\varphi}_i}(\mathbf{x}_i^{k})\right)^\top \left( f_{\boldsymbol{\varphi}_i}\left(\mathbf{x}_i^{k}\right)-\myO_{ji}^\top\V_{ij}f_{\boldsymbol{\varphi}_j}\left(\mathbf{x}_j^{k}\right)\right) \\
&- \sum_{j\in \vertexset(i)^+}\!\!\lambda_{ij}\sum_{k\in \mathcal{A}}\left(\nabla_{\boldsymbol{\varphi}_i}f_{\boldsymbol{\varphi}_i}(\mathbf{x}_i^{k})\right)^\top\left(\V_{ji}^\top\myO_{ij}f_{\boldsymbol{\varphi}_j}\left(\mathbf{x}_j^{k}\right)-f_{\boldsymbol{\varphi}_i}\left(\mathbf{x}_i^{k}\right)\right)\Bigg]\\
\stackrel{(a)}{=}& \frac{1}{K} \sum_{j\in \mathcal N(i)}\!\!\lambda_{ij}\sum_{k\in \mathcal{A}}\left(\nabla_{\boldsymbol{\varphi}_i}f_{\boldsymbol{\varphi}_i}(\mathbf{x}_i^{k})\right)^\top \left(f_{\boldsymbol{\varphi}_i}\left(\mathbf{x}_i^{k}\right)-\myO_{ji}^\top\V_{ij}f_{\boldsymbol{\varphi}_j}\left(\mathbf{x}_j^{k}\right)\right)\,;
    \end{aligned}
\end{equation}
where in $(a)$, for those $j \in \vertexset(i)^+$ we used that, under the reparameterization exploited in \cref{eq:local_gluing}, we have $\V_{ji}^\top\myO_{ij}=\myO_{ji}^\top\V_{ij}$.
This also shows the equivalence between \cref{eq:gradient_gluing_penalty} and \cref{eq:gradient_gluing_penalty_privacy}.

Regarding the communication cost, sharing transported latent representations increases the communication overhead toward neighbors $j \in \vertexset(i)^+$. 
Under standard privacy constraints, agent $i$ transmits $\A_i^t \in \reall^{d_i \times K}$, whereas in the more stringent setting it must send $\V_{ji}\A_i^t \in \reall^{d_j \times K}$, with $d_j \geq d_i$. 
Thus, the outgoing payload scales with $d_j$, and strictly increases whenever $d_j>d_i$.

The \emph{incoming} volume received by agent $i$ also changes under strict privacy, but 
asymmetrically across the two neighbor sets. 
From $j \in \mathcal{N}(i)^+$, agent $i$ receives the orthogonally transformed representation $\myO^t_{ij}\A^t_j \in \reall^{d_j \times K}$, whose dimension is unchanged relative to standard privacy. 
From $j \in \mathcal{N}(i)^-$, however, agent $i$ now receives the already-embedded representation $\V^t_{ij}\mathbf{A}^t_j \in \reall^{d_i \times K}$ rather than  the raw $\A^t_j \in \reall^{d_j \times K}$, so the incoming payload from this neighbor set grows from $d_j$ to $d_i$. 
The resulting incoming communication cost is
$$
\mathcal{O}(|\mathcal{N}(i)^-| d_i K) \quad \text{(standard: } \mathcal{O}(K\textstyle\sum_{j\in\mathcal{N}(i)^-}d_j)\text{)}
\qquad+\qquad
\mathcal{O}\Big(K\sum_{j\in\mathcal{N}(i)^+}d_j\Big) \quad \text{(unchanged)}.
$$

Our modified Sheaf-FRL protocols under stringent privacy regulatory constraints for both the heterogeneous and homogeneous cases are given in \Cref{alg:sheaf_frl_1comm_stringent_privacy,alg:sheaf_frl_hom_stringent_privacy}, respectively.
Additionally, \cref{tab:tradeoff} details the privacy vs. communication and computational costs trade-off.

\begin{table}[h]
\small
\centering
\caption{Trade-off between privacy and communication and local computation cost per agent. The local computation refers to the update of the restriction maps, i.e., the part influenced by more stringent privacy constraints. To aid comparison, for the heterogeneous case we highlight in \bluekeep{blue} the larger of $d_i$ and $d_j$ (either one when $d_i=d_j$) when considering either $j \in \mathcal{N}(i)^-$ or $j \in \mathcal{N}(i)^+$. In \red{red}, we instead highlight the (potential) increase in cost implied by more stringent privacy constraints.}
\begin{tabular}{llcc}
\toprule
 &  & Standard Privacy & Strict Privacy \\
\midrule

\multirow{4}{*}{Heterogeneous}
& \shortstack[l]{\textit{Communication} \\Outgoing} 
& $\mathcal{O}(|\mathcal{N}(i)| d_i K)$ 
& $\mathcal{O}\big(|\mathcal{N}(i)^-| d_i K + K\!\sum_{j \in \vertexset(i)^+}\!\red{d_j}\big)$ \\

& Incoming 
& $\mathcal{O}(K \sum_{j \in \vertexset(i)} d_j)$ 
& $\mathcal{O}\big(|\mathcal{N}(i)^-| \red{d_i} K + K\!\sum_{j \in \vertexset(i)^+}\!d_j\big)$ \\

\cmidrule(lr){2-4}
\addlinespace[2pt]

& \shortstack[l]{\textit{Local computation} \\ Cross-covariance} 
& $\mathcal{O}\big(K \big(\!\sum_{j \in \mathcal{N}(i)^-}\! \bluekeep{d_i} d_j + \sum_{j \in \mathcal{N}(i)^+}\! d_i\bluekeep{d_j}\big)\big)$  
& $\mathcal{O}\big(K \big(|\vertexset(i)^-|\red{d_i^2} + \sum_{j \in \mathcal{N}(i)^+}\! d_i\bluekeep{d_j}\big)\big)$ \\

& SVD 
& $\mathcal{O}\big(\bluekeep{d_i}\!\sum_{j \in \mathcal{N}(i)^-}\! d_j^2 + d_i^2\!\sum_{j \in \mathcal{N}(i)^+}\!\bluekeep{d_j}\big)$  
& $\mathcal{O}\big(|\mathcal{N}(i)^-| \red{d_i^3} + d_i^2\!\sum_{j \in \mathcal{N}(i)^+}\!\bluekeep{d_j}\big)$ \\

\addlinespace[6pt]
\midrule

\multirow{4}{*}{Homogeneous}
&\shortstack[l]{\textit{Communication} \\Outgoing}  
& $\mathcal{O}(|\mathcal{N}(i)| d K)$ 
& $\mathcal{O}(|\mathcal{N}(i)| d K)$ \\
& Incoming 
& $\mathcal{O}(|\mathcal{N}(i)| d K)$ 
& $\mathcal{O}(|\mathcal{N}(i)| d K)$ \\

\cmidrule(lr){2-4}
\addlinespace[2pt]

& \shortstack[l]{\textit{Local computation}\\Cross-covariance} 
& $\mathcal{O}(K|\mathcal{N}(i)^+| d^2)$ 
& $\mathcal{O}(K|\red{\mathcal{N}(i)^-}\cup \mathcal{N}(i)^+| d^2)$ \\

& SVD 
& $\mathcal{O}(|\mathcal{N}(i)^+| d^3)$ 
& $\mathcal{O}(|\red{\mathcal{N}(i)^-}\cup \mathcal{N}(i)^+| d^3)$ \\

\bottomrule
\end{tabular}
\label{tab:tradeoff}
\end{table}

\spara{\blue{Non-identifiability of the exchanged representations.}} \cref{tab:tradeoff} highlights a clear trade-off. Under stricter privacy requirements, the proposed protocol incurs higher communication and local computational costs, since agents exchange transported latent representations and must explicitly update the corresponding restriction maps. 
The benefit 
\blue{is a reduced exposure:}
the communicated messages no longer reveal the latent pilots in their original coordinates. 
More precisely, when only transported pilots are shared, an observer that does not know the underlying restriction maps cannot uniquely reconstruct the original latent representations, but can only recover geometric information that is invariant under isometries, as formalized in the following proposition.

\begin{proposition}[Non-identifiability of latent pilots under unknown maps]
\label{prop:privacy_nonidentifiability}
Let $\A \in \reall^{d \times K}$ be a matrix of latent pilot representations, and let $\Y = \T \A$, where $\T$ is an unknown restriction map satisfying $\T^\top \T = \mathbf{I}_d$. This includes both the homogeneous case $\T \in \ort{d}$ and the heterogeneous case $\T \in \stiefel{D}{d}$ with $D \ge d$. Then the latent pilots $\A$ are not uniquely identifiable from $\Y$ without knowledge of $\T$. In particular, $\A$ is determined only up to an orthogonal transformation.
\end{proposition}
\begin{proof}
Let $\Q \in \ort{d}$ and $\A'=\Q\A$. Then $\Y = \T \A = \T\Q^\top \A'$, and $(\T\Q^\top)^\top(\T\Q^\top)=\mathbf{I}_d$. Thus $\T\Q^\top$ is a valid map, implying that $\A$ is identifiable only up to orthogonal transformations.
\end{proof}

\begin{algorithm}[t]
\caption{Sheaf-FRL under stringent privacy regulatory constraints}
\label{alg:sheaf_frl_1comm_stringent_privacy}
\begin{algorithmic}[1]
\REQUIRE Oriented $\graph=(\vertexset,\edgeset)$, pilots $\mathcal A$, edge weights $\{\lambda_{ij}>0\}$, $\eta_\theta>0$, iterations $T$
\STATE Initialize $\{\boldsymbol{\theta}_i^0=(\boldsymbol{\varphi}_i^0,\boldsymbol{\psi}_i^0)\}_{i\in\vertexset}$, $\{\myO_{ji}^0\in\ort{d_i}\}_{j\in\mathcal N(i)^-,\,i\in\vertexset}$,
$\{\V_{ji}^0\in\stiefel{d_j}{d_i}\}_{j\in\mathcal N(i)^+,\,i\in\vertexset}$

\FOR{$t=0,\dots,T-1$}

\FORALL{$i\in\vertexset$ \textbf{in parallel}}
  \STATE $\mathbf{A}_i^t \gets [f_{\boldsymbol{\varphi}_i^t}(\x_i^{k})]_{k\in\mathcal A}$
\ENDFOR

\STATE Each node $i$ broadcasts $\myO_{ji}^t\mathbf{A}_i^t$ to those $j\in\mathcal N(i)^-$ and $\V_{ji}^t\mathbf{A}_i^t$ to $j\in\mathcal N(i)^+$

\FORALL{$i\in\vertexset$ \textbf{in parallel}}
  \FOR{$j \in \mathcal N(i)$}
    \IF{$j\in\mathcal N(i)^-$}
      \STATE $[\mathbf{U},\mathbf{\Sigma},\mathbf{W}^\top]\gets\mathrm{SVD}(\V_{ij}^t\mathbf{A}_j^{t}\mathbf{A}_i^{t^\top})$
      \STATE $\myO_{ji}^{t+1}\gets \mathbf{U}\mathbf{W}^\top$
    \ELSE
      \STATE $[\mathbf{U},\mathbf{\Sigma},\mathbf{W}^\top]\gets\mathrm{thinSVD}(\myO_{ij}^t\mathbf{A}_j^t\mathbf{A}_i^{t^\top})$
      \STATE $\mathbf{V}_{ji}^{t+1}\gets \mathbf{U}\mathbf{W}^\top$
    \ENDIF
  \ENDFOR
\ENDFOR

\FORALL{$i\in\vertexset$ \textbf{in parallel}}
  \STATE $\boldsymbol{\theta}_i^{t+1}
  \gets
  \boldsymbol{\theta}_i^t
  -\eta_\theta\left(
  \nabla_{\boldsymbol{\theta}_i}\mathcal L_i(\boldsymbol{\theta}_i^t)
  +
  \mathbf{r}_i^t
  \right)$
\ENDFOR

\ENDFOR
\end{algorithmic}
\end{algorithm}

\begin{algorithm}[t]
\caption{Sheaf-FRL under stringent privacy regulatory constraints (homogeneous case)}
\label{alg:sheaf_frl_hom_stringent_privacy}
\begin{algorithmic}[1]
\REQUIRE Oriented $\graph=(\vertexset,\edgeset)$, pilots $\mathcal A$, edge weights $\{\lambda_{ij}>0\}$, $\eta_\theta>0$, iterations $T$
\STATE Initialize $\{\boldsymbol{\theta}_i^0=(\boldsymbol{\varphi}_i^0,\boldsymbol{\psi}_i^0)\}_{i\in\vertexset}$,  $\{\myO_{ji}^0 \in \ort{d}\}_{j \in \vertexset(i), i \in \vertexset}$

\FOR{$t=0,\dots,T-1$}

\FORALL{$i\in\vertexset$ \textbf{in parallel}}
  \STATE $\mathbf{A}_i^t \gets [f_{\boldsymbol{\varphi}_i^t}(\mathbf{x}_i^{k})]_{k\in\mathcal A}$
\ENDFOR

\STATE Each node $i$ exchanges $\myO_{ji}^{t}\mathbf{A}_i^t$ with all $j\in\mathcal N(i)$

\FORALL{$i\in\vertexset$ \textbf{in parallel}}
  \FORALL{$j\in\mathcal N(i)$}
    \STATE $[\mathbf{U},\mathbf{\Sigma},\mathbf{W}^\top]
    \gets
    \mathrm{SVD}\big(\myO_{ij}^{t}\mathbf{A}_j^t\mathbf{A}_i^{t^\top}\big)$
    \STATE $\myO_{ji}^{t+1}\gets \mathbf{U}\mathbf{W}^\top$
  \ENDFOR
\ENDFOR

\FORALL{$i\in\vertexset$ \textbf{in parallel}}
  \STATE $\boldsymbol{\theta}_i^{t+1}
  \gets
  \boldsymbol{\theta}_i^{t}
  -\eta_\theta\left(
  \nabla_{\boldsymbol{\theta}_i}\mathcal L_i(\boldsymbol{\theta}_i^{t})
  +
  \mathbf{r}_i^{t}
  \right)$
\ENDFOR

\ENDFOR

\end{algorithmic}
\end{algorithm}

\section{Additional discussion of our framework}\label{app:add_discussion}

This section collects additional material on the framework introduced in \cref{sec:fpf}.
We discuss, in order, the choice of the edge stalk dimensionality, the heterogeneous geometry of local latent spaces, the exact centralized formulation under global alignment, the geometric consistency enforced by the restriction maps, and the homogeneous setting in which all node stalks share the same dimensionality.

\spara{Embedding vs. compression.}
Choosing the edge stalk dimensionality $d_{ij} = \max(d_i,d_j)$ avoids introducing an information bottleneck when aligning heterogeneous latent representations.
In contrast, as done for instance in \citep{badi2026}, enforcing $d_{ij} \le \min(d_i,d_j)$ would require projecting both representations onto a shared lower-dimensional subspace, implicitly assuming the existence of a common latent structure and potentially discarding node-specific information.
Moreover, such a projection only enforces alignment on the retained subspace, leaving the remaining components unconstrained.
As a result, the learned restriction maps capture consistency only up to an information loss, potentially leading to inconsistencies when transferring latent representations across models.
Our choice instead pursues a geometric notion of consistency across representations, and preserves the semantics of local representations.
Importantly, this also makes our framework well-suited for downstream transfer of latent representations among pre-trained models, where alignment based on compression may result in performance degradation on downstream local tasks.

\spara{Heterogeneous latent space geometry.}
\blue{We expand here on the CenterNorm layer introduced in \cref{subsec:sheaf}.
In our setting, the encoders $f_{\bm\varphi_i}$ are trained locally, on different data, and possibly with different architectures.
Then, it is reasonable to expect the induced latent spaces $\mathcal F(i)$ to exhibit heterogeneous geometry, in the sense of different offsets, scales and anisotropies \citep{pandolfo2026semasia, svendsen2026improving}.}
\blue{This may occur even when the dimensionalities $d_i$ coincide.
However, the geometric restriction maps considered in this work are effective for measuring inconsistency only when the latent representations $\mathbf z_i$ and $\mathbf z_j$ adhere to a common geometric reference on the edge stalk.
Thus, to make representations from distinct node stalks comparable once transported, each encoder terminates in a center-and-normalize layer,}
\begin{equation}\label{eq:cn}
  \blue{f_{\bm\varphi_i}=n_i\circ b_{\bm\varphi_i},\qquad n_i(\mathbf{u})=\frac{\mathbf{u}-\tilde{\bm\mu}_i}{\max\{\|\mathbf{u}-\tilde{\bm\mu}_i\|_2,\epsilon\}},}
\end{equation}
\blue{where $b_{\bm\varphi_i}$ is the trainable backbone and $\epsilon>0$ is a small fixed constant that prevents division by zero ($\epsilon=10^{-6}$ in our experiments).
As in batch normalization, during training $\bm\mu_i$ is the mean of the current minibatch and is differentiated through, while a running estimate accumulated during training is used at inference.
The layer has no trainable parameters, so $\bm\theta_i$ collects the parameters of the backbone and of the head.
It removes node-specific offsets and scales.
Every representation lies in the closed unit ball, and on the unit sphere whenever its distance from the mean exceeds $\epsilon$.
Unlike whitening, it does not equalize the covariance, so residual anisotropies are left to the alignment and are reflected in its residual.
Training examples and pilots alike pass through the layer, so the local objective $\mathcal L_i$ is defined on the normalized space and a representation transported from a neighbor is read by the head directly, with no inverse map.}

\blue{Two consequences follow.
First, the reference is compatible with isometric restriction maps. 
An isometric embedding preserves norms, so transported representations remain in the unit ball, and in the reparameterized form of \cref{eq:local_gluing} each edge term reduces to a cosine distance, $\|\z_i-\mathbf \V_{ij}\z_j\|_2^2=\|\z_i\|_2^2+\|\z_j\|_2^2-2\,\z_i^\top\V_{ij}\z_j$.
Second, wherever $\|\mathbf{u}-\tilde{\bm\mu}_i\|_2\ge r_0>\epsilon$, the clamp is inactive, $n_i$ is smooth, and its first and second derivatives are bounded by constants that depend on $r_0$ only, the norm of its Jacobian being at most $1/r_0$, so that the smoothness required by \cref{ass:smooth} is inherited from the backbone $b_{\bm\varphi_i}$ (App.\nb\cref{sec:Convergence_analysis}).}

\spara{Exact global alignment: the centralized formulation.}
In \cref{subsec:sheaf} we mention that requiring the learned representations to be a global section of the learned sheaf is the idealized notion of network-wide alignment, and that such a requirement is impractical.
We make both statements precise here.
Recall that a valuation $\z^\star \in C^0(\graph;\ns)$ belonging to $\ker(\mathbf{L}_\ns)$ is said \emph{global section}, and can be thought of as an assignment of latent representations to the node stalks that does not break local rules.
Indeed, from \cref{eq:ns_Laplacian} it is clear that $\bm\Delta\z^\star=\mathbf{0}$, which implies
\begin{equation}\label{eq:global-sec-conditions}
    \myO_{ji}\z_i^\star = \V_{ij}\z_j^\star\, \quad \text{for each } e_{ij} \in \edgeset\,;
\end{equation}
where $\z_i^\star\in \reall^{d_i}$ and $\z_j^\star \in \reall^{d_j}$ are the components of $\z^\star$ at node $i$ and $j$, respectively.
A plausible strategy would then be to constrain the learned latent representations across the network to be globally consistent with respect to the learned network sheaf $\ns$.
For simplicity, assume the existence of a set of shared sample indices across agents, and let $M \leq \min_{i\in\vertexset} M_i$ be the number of such comparable samples.
Then, according to the previous rationale, we can pose the following centralized problem formulation
\begin{equation}\label{eq:centralized_exact}
\begin{aligned}
\min_{\substack{\{\bm\theta_i=(\bm\varphi_i,\bm\psi_i)\}_{i\in\vertexset}\\
\{\myO_{ji} \in \ort{d_i},\,\V_{ij} \in \stiefel{d_i}{d_j} \}_{e_{ij}\in\edgeset}}}
\quad &
\sum_{i\in\vertexset} \mathcal{L}_i(\bm\theta_i)\\
\text{s.t.}\quad &
\z^n =\{f_{\bm\varphi_i}(\x_i^n)\}_{i\in\vertexset} \in \ker(\mathbf L_{\ns}),
\qquad \forall n \in [M]\,;
\end{aligned}
\end{equation}
where the network sheaf $\ns$ is learned through its restriction maps under a known communication topology.
However, the constraints in Prob.\nb\eqref{eq:centralized_exact} might be too restrictive in practice.
Enforcing exact compatibility across heterogeneous agents is generally unrealistic.
Moreover, satisfying the constraint for all samples would in principle require a centralized solution, as it couples all node representations simultaneously.
This limitation motivates the relaxed, decentralized formulation of \cref{subsec:sfrl}, of which Prob.\nb\eqref{eq:global_problem} is the scalable counterpart.

\spara{On the geometric consistency.}
To better understand the action of the restriction maps as well as the role of the total variation penalty, let us consider $\ns(i)\rightarrow \ns(e_{ij}) \leftarrow\ns(j)$ in \Cref{fig:sfrl}(c).
When transported to the edge stalk $\ns(e_{ij}) \cong \reall^3$, the latent representation $\z_i \in \ns(i)$ is rotated and reflected by the orthogonal map $\myO_{ji}$.
Instead, the latent representation $\z_j$ is embedded via the Stiefel map $\V_{ij}$ into the higher-dimensional edge stalk.
Interestingly, the image \im{\V_{ij}} of $\V_{ij}$ defines a two-dimensional subspace embedded in $\ns(e_{ij})$.
Thus, a perfect alignment between the latent representations $\z_i$ and $\z_j$ implies that $\z_i$ is a rotated and reflected version of a vector $\y_{j\to e_{ij}}=\V_{ij}\z_j$ belonging to \im{\V_{ij}}.
When this occurs for all the latent representations in $\ns(i)$, it means that the latter representations lie on a linear two-dimensional subspace of $\ns(i)$.

In \Cref{fig:sfrl}(c), we notice that there is only a partial alignment between $\z_i$ and $\z_j$ when they are transported onto $\ns(e_{ij})$.
Indeed, even though the projection of $\myO_{ji}(\z_i)$ onto \im{\V_{ij}} exactly matches $\V_{ij}\z_j$, there is still an orthogonal component that makes non-null the Euclidean distance between the two transported vectors.
Equivalently, in $\ns(i)$ this means that $\z_i$ does not lie on the subspace $\myO_{ji}^\top(\im{\V_{ij}})$, onto which the transported representation of agent $j$ is mapped.
Thus, the total variation term acts as a regularizer that, for all $e_{ij} \in \edgeset$, \blue{penalizes both the mismatch, within $\im{\V_{ij}}$, between $\V_{ij}\z_j$ and the projection of $\myO_{ji}\z_i$, and the component of $\myO_{ji}\z_i$ orthogonal to $\im{\V_{ij}}$. 
The former vanishes in the \cref{fig:sfrl}(c), and the latter is what remains.} 

Notably, the latter discussion highlights that within our proposed framework we do not assume or leverage the existence of a \emph{single} global latent representation, as instead pursued in recent works \citep{setayesh2026toward,badi2026}.
Indeed, in our framework a global representation may eventually emerge from the alignment of local representations over the network sheaf.

\spara{Homogeneous setting.}
When $\ns(i)=\ns(e)\cong\reall^d$ for all $i \in \vertexset$ and $e \in \edgeset$, the local total variation term in \cref{eq:local_gluing} reduces to
\begin{equation}\label{eq:local_gluing_ort}
    \norm{\z_i - \myO_{ij} \z_j}_2^2\,,
\end{equation}
where $\myO_{ij} \in \ort{d}$.
In this case, alignment between neighboring representations is achieved up to an orthogonal transformation, i.e., rotations and reflections of the latent space.
This setting connects \Cref{eq:total_variation} to the Dirichlet energy of the connection Laplacian \citep{chung2014ranking}, and more generally to synchronization problems over the orthogonal group \citep{SingerWu2011VDM}.
In particular, the transport maps satisfy $\myO_{ji} = \myO_{ij}^{\top}$, and the two variation terms in \cref{eq:tv_at_i} coincide.
In contrast, in the heterogeneous setting we allow semi-orthogonal maps on the Stiefel manifold, which act as isometric embeddings and are only left-invertible.
Finally, since all node stalks share the same dimension, the orientation of \graph does not affect the model and can be chosen arbitrarily.
\section{Pilot Selection Strategies}
\label{app:pilots}

The alignment of latent spaces relies on a set of \emph{semantic pilots}: 
inputs whose correspondence across agents is known a priori, so that $f_{\bm\varphi_i}(\x_i^{k})$ and $f_{\bm\varphi_j}(\x_j^{k})$ can be treated as two latent representations of the same underlying sample. In practice, pilots arise from public datasets, calibration signals, standardized inputs, or an overlapping portion of the local datasets, and we assume that a common index set $\{1,\ldots,M\}$ has been agreed upon.

The choice of $\mathcal{A}\subset\{1,\ldots,M\}$, with $|\mathcal{A}|=K\ll M$, governs the trade-off between communication cost, computational efficiency, geometric coverage, and alignment quality. In the sequel, we outline the most widely used practical selection criteria.

\spara{Original pilots.}
Each agent transmits the embeddings $f_{\bm\varphi_i}(\x_i^{k})$ of the shared inputs themselves, where $k \in \mathcal A$. Taking $\mathcal{A}=\{1,\ldots,M\}$ evaluates the penalty exactly and is the most reliable option, but scales linearly in $M$ in both communication and Procrustes cost; subsampling to a budget $K\ll M$ instead yields an unbiased Monte Carlo estimator of the full alignment energy.

\epara{Unsupervised selection:} 
Pilots are drawn uniformly, which requires no side information but may under-represent low-density regions of the representation space and lead to ill-conditioned $\A_i$. Clustering the representation space and drawing a fixed number of pilots per cluster recovers the same stratification that labels provide in the supervised case, at the cost of one clustering pass, and guarantees that sparsely populated regions are represented.

\epara{Supervised selection:} 
Labels are exploited since they induce a natural stratification, and a fixed number of pilots is drawn per class, guaranteeing balanced coverage even for small $K$.

\spara{Prototype pilots.}
The shared set is partitioned into $K$ groups and each pilot is defined as the group average (centroid),
\[
\bar{\mathbf{f}}_{i,c}
= \frac{1}{|\mathcal{S}_c|}\sum_{k\in\mathcal{S}_c}
f_{\bm\varphi_i}(\x_i^{k}),
\qquad c=1,\ldots,K \,;
\]
which preserves semantic structure while averaging out per-sample noise.
When groups are large, the centroid need not be computed over all their members: averaging over a subset $\mathcal{S}_c' \subset \mathcal{S}_c$ gives an unbiased estimate at a fraction of the local forward passes, trading estimation variance for computational cost.

\epara{Unsupervised partitioning:} 
The partition is obtained by clustering the representation space (e.g., via $k$-means), with the resulting index assignment broadcast once and reused thereafter \citep{huttebraucker2024relative, fiorellino2026frame}.   

\epara{Supervised partitioning:} 
Class labels provide the partition directly, recovering prototype-based federated schemes such as FedProto \citep{tan2022fedproto}.

\spara{Ensuring correspondence and adaptive selection.}
To ensure that the transmitted quantities are in correspondence, both original and prototype pilots require the selection to be replicated across agents.
Specifically: \emph{(i)} the sampled indices in the first case, and \emph{(ii)} the partition and the subsets used to estimate the centroids in the second. Furthermore, the selection can be made adaptive by prioritizing the samples that currently exhibit the largest cross-client mismatch, biasing which inputs are transmitted or which groups are refined so as to concentrate the alignment effort on geometrically inconsistent regions. 

\spara{Dynamic selection.}
Any of the above criteria can be applied per iteration, letting $\mathcal{A}_t$ vary over time and turning the penalty into a stochastic approximation of the full alignment energy. A practical recipe combines a small persistent subset (for stability of the transport maps) with a periodically refreshed remainder (for coverage), which empirically improves generalization of the learned maps beyond the sampled pilots \citep{fiorellino2024dynamic}.
\section{Identifiability of the Transport Maps}
\label{app:identifiability}

The edge-wise transports are computed via (semi-)orthogonal Procrustes problems
\begin{equation}
\label{eq:procrustes_generic}
\mathbf{V}_{ij} \in \arg\min_{\mathbf{V}\in\mathcal{C}_{ij}}
\left\| \A_i - \mathbf{V}\A_j \right\|_F^2,
\end{equation}
where $\mathcal{C}_{ij}=\ort{d}$ in the homogeneous case and $\mathcal{C}_{ij}=\mathrm{St}(d_i,d_j)$ for $d_i> d_j$.
We summarize conditions under which the minimizer is unique.

\spara{Homogeneous case ($d_i=d_j=d$).}
Let $\mathbf{M}_{ij} \coloneqq \A_i \A_j^\top \in \mathbb{R}^{d\times d}$.
An optimal solution is $\mathbf{V}_{ij}^\star = \mathbf{U}\mathbf{W}^\top$, where $\mathbf{M}_{ij}=\mathbf{U}\mathbf{\Sigma}\mathbf{W}^\top$ is an SVD \citep{Schonemann1966Procrustes}. 
The solution is unique if $\mathbf{M}_{ij}$ is full rank, i.e., $\mathrm{rank}(\mathbf{M})=d$.
Repeated singular values may make the individual SVD factors $\mathbf{U}$ and $\mathbf{W}$ non-unique, but the product $\mathbf{U}\mathbf{W}^\top$, hence the optimal polar factor, remains unique.
In particular, if $\mathrm{rank}(\A_i)=\mathrm{rank}(\A_j)=d$ (e.g., $K\ge d$ with non-degenerate pilots), then $\mathbf{M}_{ij}$ is generically full rank and $\mathbf{V}_{ij}^\star$ is uniquely defined up to measure-zero degeneracies.

\spara{Stiefel case ($d_i> d_j$).}
Let $\mathbf{M}_{ij}=\A_i\A_j^\top\in\mathbb{R}^{d_i\times d_j}$ and $\mathbf{M}_{ij}=\mathbf{U}\mathbf{\Sigma}\mathbf{W}^\top$ a thin SVD, with $\mathbf{U}\in\mathbb{R}^{d_i\times d_j}$, $\mathbf{W}\in\mathbb{R}^{d_j\times d_j}$.
A minimizer is
\begin{equation}
\label{eq:stiefel_solution}
\mathbf{V}_{ij}^\star = \mathbf{U}\mathbf{W}^\top \in \mathrm{St}(d_i,d_j)\,.
\end{equation}
Uniqueness holds if $\mathrm{rank}(\mathbf{M}_{ij})=d_j$. 
Crucially, $\mathbf{V}_{ij}$ is identifiable only on the subspace spanned by $\A_j$: if $\mathrm{rank}(\A_j)=r<d_j$, the solution is underdetermined on the orthogonal complement. 
Hence, identifiability requires $\mathrm{rank}(\A_j)=d_j$ (typically ensured when $K\ge d_j$ and pilots provide sufficient geometric coverage), in which case $\mathbf{V}_{ij}^\star$ is generically unique up to measure-zero degeneracies.
\blue{Across the network, identifiability therefore requires $K\ge\max_{(i,j)\in \edgeset}\min(d_i,d_j)$.
When the pilot features are centered on the mean of the pilot set, their columns satisfy one linear relation, so that $\mathrm{rank}(\A_j)\le K-1$ and one additional pilot is needed.}
\section{Convergence Analysis}\label{sec:Convergence_analysis}

In this section we establish convergence guarantees for the proposed Sheaf-FRL algorithm in both deterministic and stochastic settings.
The convergence analysis in this appendix applies to the standard-privacy Sheaf-FRL iteration, for which the restriction-map block is updated by the exact closed-form minimizers in \eqref{prob:alignment_V_ij}--\eqref{prob:alignment_V_ji}. 
\cref{alg:sheaf_frl_1comm_stringent_privacy} uses a strict-privacy parallel update of the coupled restriction maps \eqref{prob:alignment_O_ij_strict_privacy}--\eqref{prob:alignment_V_ji_strict_privacy}. 
Although each individual Procrustes/SVD subproblem is solved exactly conditional on the exchanged representations, the simultaneous updates do not constitute an exact joint minimization of the full restriction-map block. 
Therefore, the descent and convergence results below are not claimed for \cref{alg:sheaf_frl_1comm_stringent_privacy}.
We first recall the classical descent result for smooth functions, and provide auxiliary derivations.

\begin{lemma}[Descent Lemma, {\citealp[Prop.~1.2.3]{Nesterov2004},\citealp[Sec.~2.1]{Beck2017}}]
\label{lem:descent_V}
Let $f$ be differentiable with $L$-Lipschitz gradient.
Then for any $\mathbf{x}$ and any $\eta>0$,
\[
f(\mathbf{x}-\eta\nabla f(\mathbf{x}))
\le
f(\mathbf{x})
-
\eta\Big(1-\frac{L\eta}{2}\Big)
\|\nabla f(\mathbf{x})\|^2.
\]
\end{lemma}

\subsection{Auxiliary derivations}\label{app:derivations}

\spara{Reduction of the alignment subproblems to Procrustes form.}
The objective of Prob.\nb\eqref{prob:alignment_V_ij}, and analogously of \eqref{prob:alignment_V_ji}, can be recast as
\begin{equation}\label{eq:local_gluing_expansion}
    \frob{\A_{i}^{t} - \V_{ij} \A_j^{t}}^2 = \tr{\A_{i}^{t}(\A_{i}^{t})^\top}+\tr{\A_{j}^{t}(\A_{j}^{t})^\top} -2 \tr{\A_{i}^{t}(\A_{j}^{t})^\top\V_{ij}^\top}\,.
\end{equation}
The first two terms do not depend on $\V_{ij}$.
Minimizing \cref{eq:local_gluing_expansion} thus corresponds to maximizing the third term, which yields the closed-form solution reported in \cref{eq:procrustes_solution_V}.

\spara{Edge-wise regularization weights.}
For each undirected edge $(i,j)\in\mathcal \edgeset$, we use a single symmetric coefficient $\lambda_{ij}=\lambda_{ji}>0$, treated as a fixed edge-wise hyperparameter. 
Hence, the same coefficient multiplies the two endpoint contributions in \cref{eq:local_gluing_penalty}. 
For an edge $e_{ij}$, the two local terms sum to $ \frac{\lambda_{ij}}{K}\frob{\A_i-\V_{ij}\A_j}^2$, \blue{so that, with the factor $\frac{1}{2}$, the edge contributes $\frac{\lambda_{ij}}{2K}\|\A_i-\V_{ij}\A_j\|_F^2$ to $\mathcal{R}_{\mathcal A}$.}
Therefore, differentiation with respect to $\bm\varphi_i$ yields the coefficient $\lambda_{ij}/K$, exactly as in \cref{eq:gradient_gluing_penalty,eq:gradient_gluing_penalty_full}. 
Thus, the local neural update uses the exact partial gradient $\nabla_{\bm\varphi_i}J(\bm\theta,\V)$ of the global objective, and no information about a neighbor-specific node weight is required.

\spara{Gradient of the gluing penalty.}
We report here the full derivation of \cref{eq:gradient_gluing_penalty}.
Starting from \cref{eq:local_gluing_penalty}, and recalling that the regularization term depends exclusively on the encoder parameters $\bm\varphi_i$, for each agent $i\in\vertexset$ we obtain
\begin{equation}\label{eq:gradient_gluing_penalty_full}
    \begin{aligned}
\nabla_{\bm\varphi_i}\mathcal R_{\mathcal{A}}\at{i}=\nabla_{\bm\varphi_i}\mathcal R_{\mathcal{A}}
=&\frac{1}{K}\Bigg[\sum_{j\in \vertexset(i)^-}\!\!\lambda_{ij}\sum_{k\in \mathcal{A}}\left( \nabla_{\bm\varphi_i}f_{\bm\varphi_i}(\x_i^{k})\right)^\top\left(f_{\bm\varphi_i}\left(\x_i^{k}\right)-\V_{ij}f_{\bm\varphi_j}\left(\x_j^{k}\right)\right) \\
&- \sum_{j\in \vertexset(i)^+}\!\!\lambda_{ij}\sum_{k\in \mathcal{A}}\left(\nabla_{\bm\varphi_i}f_{\bm\varphi_i}(\x_i^{k})\right)^\top\V_{ji}^\top\left(f_{\bm\varphi_j}\left(\x_j^{k}\right)-\V_{ji}f_{\bm\varphi_i}\left(\x_i^{k}\right)\right)\Bigg]\\
\stackrel{(a)}{=}& \frac{1}{K} \sum_{j\in \mathcal N(i)}\!\!\lambda_{ij}\sum_{k\in \mathcal{A}}\left(\nabla_{\bm\varphi_i}f_{\bm\varphi_i}(\x_i^{k})\right)^\top \left(f_{\bm\varphi_i}\left(\x_i^{k}\right)-\V_{ij}f_{\bm\varphi_j}\left(\x_j^{k}\right)\right)\,;
    \end{aligned}
\end{equation}
where $\nabla_{\bm\varphi_i} f_{\bm\varphi_i}(\x_i^{k}) \in \reall^{d_i \times |\bm\varphi_i|}$ denotes the Jacobian of the feature map $f_{\bm\varphi_i}$ with respect to the parameter vector $\bm\varphi_i$, evaluated at $\x_i^{k}$.
In $(a)$ we adopt, for $j \in \vertexset(i)^+$, the convention $\V_{ij}\coloneqq\V_{ji}^\top$, which satisfies $\V_{ij}\V_{ji}=\identity_{d_i}$ by orthonormality of the columns of $\V_{ji} \in \stiefel{d_j}{d_i}$.
Note that, under this convention, $\V_{ij}$ has orthonormal rows rather than columns on outgoing edges.
This shows that both incoming and outgoing contributions admit a unified expression, which is the one reported in \cref{eq:gradient_gluing_penalty} of the main text.

\subsection{Deterministic Setting}

Let
\begin{equation}\label{eq:loss_convergence}
J(\boldsymbol{\theta},\mathbf{V})
=
\sum_{i\in \vertexset}
\mathcal{L}_i(\bm\theta_i)
+ \mathcal{R}_{\mathcal{A}}\!\left(
\{\bm\varphi_i\}_{i\in \vertexset},
\{\V_{ij}\}_{(i,j)\in \edgeset},
\{\V_{ji}\}_{(i,j)\in \edgeset}
\right)
\end{equation}
in \eqref{eq:global_problem} where, with slight abuse of notation,
$\boldsymbol{\theta}
=
\{\boldsymbol{\theta}_i\}_{i\in \vertexset}
=
\{(\bm\varphi_i,\bm\psi_i)\}_{i\in \vertexset}$,
and
$\mathbf{V}
=
\{\V_{ij}\}_{(i,j)\in\edgeset}
\cup
\{\V_{ji}\}_{(i,j)\in\edgeset}$.
For each edge $(i,j)\in \edgeset$, the matrices $\V_{ij} \in \stiefel{d_i}{d_j}$ and $\V_{ji} \in \stiefel{d_j}{d_i}$, where $\stiefel{d_i}{d_j}$ denotes a Stiefel manifold as defined in \Cref{eq:manifolds}.
The feasible set is therefore
$\mathcal V := \left(\prod_{(i,j)\in \edgeset} \stiefel{d_i}{d_j}\right) \times \left(\prod_{(i,j)\in \edgeset} \stiefel{d_j}{d_i}\right)$. 
\blue{The CenterNorm layer in \cref{eq:cn} normalizes with $\tilde{\bm\mu}_i^t(\bm\varphi_i;B)\coloneqq(1-\alpha)\mathbf{m}_i^t+\alpha\bm\mu_i^B(\bm\varphi_i)$, where $\bm\mu_i^B$ is the mean of the batch $B$, which is differentiated through, and $\mathbf{m}_i^t$ is a running mean, which is treated as a constant and updated with momentum $\alpha\in(0,1]$.
In the deterministic setting the batch is the whole local dataset and the whole pilot set, so that the $\bm\theta$-update is classical gradient descent; in the stochastic setting of App.\nb\ref{subapp:stochastic-setting} it is a minibatch.
We write $J(\bm\theta,\V;\mathbf{m})$ for the objective \cref{eq:loss_convergence} under this normalization, in expectation over the batches, and $J_t(\bm\theta,\V)\coloneqq J(\bm\theta,\V;\mathbf{m}^t)$.}
The proposed method alternates between an exact minimization with respect to $\mathbf{V}$, a gradient step with respect to $\boldsymbol{\theta}$, \blue{and an update of the running means, which are set to the value of the normalization mean on the current batch}:
\begin{align}
\mathbf{V}^{t}
&\in
\arg\min_{\mathbf{V}\in\mathcal V}
\blue{J_t}(\boldsymbol{\theta}^{t},\mathbf{V}),\label{eq:stiefelstep_V}\\
\boldsymbol{\theta}^{t+1}
&=
\boldsymbol{\theta}^t
-
\eta
\nabla_{\boldsymbol{\theta}}
\blue{J_t}(\boldsymbol{\theta}^t,\mathbf{V}^{t}),\label{eq:gdstep_V}\\
\blue{\mathbf{m}_i^{t+1}}
&\blue{=
\tilde{\bm\mu}_i^t(\bm\varphi_i^t;B_t)
=
(1-\alpha)\mathbf{m}_i^t+\alpha\,\bm\mu_i^{B_t}(\bm\varphi_i^t).}\label{eq:mstep}
\end{align}
We introduce the following assumptions.
\begin{assumption}
\label{ass:lower}
The objective $J(\boldsymbol{\theta},\mathbf{V}\blue{;\mathbf{m}})$ is bounded from below, with lower bound $$J_{\inf}\coloneqq\inf_{\boldsymbol{\theta},\mathbf{V},\blue{\mathbf{m}}}J(\boldsymbol{\theta},\mathbf{V}\blue{; \mathbf{m}})\,;$$
and its sublevel sets are bounded \blue{uniformly in $\mathbf{m}$}.
\end{assumption}

The second condition holds whenever the local objectives include a weight-decay term, as in our experiments (App.\nb\ref{app:exp_details}).

\begin{assumption}
\label{ass:smooth}
For every fixed $\mathbf{V}\in\mathcal V$ \blue{and $\mathbf{m}$},
$J(\boldsymbol{\theta},\mathbf{V}\blue{;\mathbf{m}})$
is differentiable and its gradient is $L$-Lipschitz \blue{with $L$ independent of $\mathbf{V}$ and $\mathbf{m}$}:
\[
\|
\nabla_{\boldsymbol{\theta}}
J(\boldsymbol{\theta}_1,\mathbf{V}\blue{;\mathbf{m}})
-
\nabla_{\boldsymbol{\theta}}
J(\boldsymbol{\theta}_2,\mathbf{V}\blue{;\mathbf{m}})
\|
\le
L
\|
\boldsymbol{\theta}_1-\boldsymbol{\theta}_2
\|.
\]
\blue{Moreover, $J$ and $\nabla_{\boldsymbol{\theta}}J$ are continuous in $(\boldsymbol{\theta},\mathbf{V},\mathbf{m})$.}
\end{assumption}

\blue{Since the clamp in \cref{eq:cn} is inactive whenever $\|b_{\bm\varphi_i}(\mathbf{x})-\tilde{\bm\mu}_i\|_2>\epsilon$, \cref{ass:smooth} holds if the backbones and heads are smooth and the representations remain at distance at least $r_0>\epsilon$ from their normalization means on the sublevel sets of \cref{ass:lower}; in that region $n_i(\mathbf{u})=(\mathbf{u}-\tilde{\bm\mu}_i)/\|\mathbf{u}-\tilde{\bm\mu}_i\|_2$, whose Jacobian has norm at most $1/r_0$.}

\begin{assumption}
\label{ass:exact}
At each iteration,
$$\mathbf{V}^{t}
\in
\arg\min_{\mathbf{V}\in\mathcal V}
\blue{J_t}(\boldsymbol{\theta}^{t},\mathbf{V}),$$ 
i.e., $\mathbf{V}^{t}$ is an exact minimizer of $\blue{J_t}$ given $\boldsymbol{\theta}^{t}$.
\end{assumption}

\begin{assumption}
\label{ass:compact}
The feasible set $\mathcal V$ is compact.
\end{assumption}

\blue{
\begin{assumption}[Running mean]\label{ass:running_mean}
There exist $D,L_{\mathbf m}<\infty$ such that $\|\bm\mu^{B}(\bm\varphi^t)-\mathbf{m}^t\|\le D$ for all $t$ and $B$, where $\bm\mu^B$ and $\mathbf{m}$ stack the batch and running means of all nodes, and, for all $\mathbf{m},\mathbf{m}'$ and all $\mathbf{V}\in\mathcal V$,
\[
\big|J(\boldsymbol{\theta},\mathbf{V};\mathbf{m})-J(\boldsymbol{\theta},\mathbf{V};\mathbf{m}')\big|\le L_{\mathbf m}\,\|\mathbf{m}-\mathbf{m}'\|\,,
\]
uniformly over a set containing the iterates $\{\boldsymbol{\theta}^t\}_{t\ge0}$.
\end{assumption}
The first condition holds whenever the backbone outputs are bounded, since $\mathbf{m}^t$ is a convex combination of batch means.
The second holds whenever the local losses are Lipschitz in the latent representations and the clamp in \cref{eq:cn} is inactive, since the Jacobian of $n_i$ with respect to its mean then has norm at most $1/r_0$ (see the remark after \cref{ass:smooth}), and the latents lie on the unit sphere while the restriction maps are isometries.
Since $\mathbf{m}$ enters the normalization mean with weight $1-\alpha$, $L_{\mathbf m}$ scales with $1-\alpha$, and $L_{\mathbf m}=0$ for $\alpha=1$.}

\blue{\begin{lemma}[Drift of the objective]\label{lem:drift}
Under \cref{ass:running_mean}, for all $t$, $\boldsymbol{\theta}$ and $\mathbf{V}\in\mathcal V$,
\begin{equation*}
J_{t+1}(\boldsymbol{\theta},\mathbf{V})\le J_t(\boldsymbol{\theta},\mathbf{V})+\gamma, \qquad \gamma:=L_{\mathbf m}\,\alpha\,D .
\end{equation*}
If the momentum varies with $t$, the same bound holds with $\gamma_t:=L_{\mathbf m}\,\alpha_t\,D$.
\end{lemma}
\begin{proof}
By \cref{eq:mstep}, $\|\mathbf{m}^{t+1}-\mathbf{m}^t\|=\alpha\|\bm\mu^{B_t}(\bm\varphi^t)-\mathbf{m}^t\|\le\alpha D$, and the claim follows from \cref{ass:running_mean}.
\end{proof}
}

\blue{
\begin{remark}[Size of the neighbourhood]
The bound $\gamma=L_{\mathbf m}\alpha D$ is a worst case.
The drift is driven by the displacement $\mathbf{m}^{t+1}-\mathbf{m}^t=\alpha\big(\bm\mu^{B_t}(\bm\varphi^t)-\mathbf{m}^t\big)$, which combines the lag of the running mean behind the full-batch mean and the sampling fluctuation of the minibatch mean.
In the deterministic setting only the lag is present, and it vanishes geometrically once the iterates stabilize, so that the drift vanishes as well.
In the stochastic setting the fluctuation persists, but it has zero conditional mean at equilibrium and its size decreases with the batch size.
\end{remark}}

Under \ref{ass:lower}-\blue{\ref{ass:running_mean}}, the following convergence result holds.

\begin{theorem}
\label{thm:main_Q}
Suppose Assumptions~\ref{ass:lower}--\blue{\ref{ass:running_mean}}
hold and choose $0<\eta\le 1/L$.
Then:

\begin{enumerate}
\item \textbf{\blue{Quasi-descent.}} \blue{For all $t\ge0$,
\begin{equation}\label{eq:quasidescent}
J_{t+1}(\boldsymbol{\theta}^{t+1},\mathbf{V}^{t+1})
\le
J_t(\boldsymbol{\theta}^t,\mathbf{V}^{t})
-
\frac{\eta}{2}
\|
\nabla_{\boldsymbol{\theta}}
J_t(\boldsymbol{\theta}^t,\mathbf{V}^{t})
\|^2
+\gamma .
\end{equation}}

\item \textbf{Sublinear stationarity rate.}
For all $T\ge 0$,
\begin{equation}
\label{eq:rate_main_Q}
\min_{0\le t\le T}
\|
\nabla_{\boldsymbol{\theta}}
\blue{J_t}(\boldsymbol{\theta}^t,\mathbf{V}^{t})
\|^2
\le
\frac{
2\big(
\blue{J_0}(\boldsymbol{\theta}^0,\mathbf{V}^0)
-
J_{\inf}
\big)
}{
\eta (T+1)
}
\blue{+\frac{2\gamma}{\eta}}\,.
\end{equation}

\item \textbf{Convergence and Stationarity.} \blue{If $\alpha=1$, or if the momentum is a schedule $\alpha_t\in(0,1]$ with $\sum_t\alpha_t<\infty$, then} the sequence $\{\blue{J_t}(\boldsymbol{\theta}^t,\mathbf{V}^{t})\}$ \blue{converges}, 
$\lim_{t\to\infty}\|\nabla_{\boldsymbol{\theta}} \blue{J_t}(\boldsymbol{\theta}^t,\mathbf{V}^{t})\| = 0$, \blue{$\mathbf{m}^t\to\mathbf{m}^\star$,}
and every accumulation point $(\boldsymbol{\theta}^\star,\mathbf{V}^\star)$ is a block-wise first-order stationary point \blue{of $J(\cdot,\cdot;\mathbf{m}^\star)$}.
\end{enumerate}
\end{theorem}

\begin{proof}

\textbf{(i) \blue{Quasi-descent}}.
Fix $\mathbf{V}^{t}$ and consider
$
\blue{J_t}(\boldsymbol{\theta},\mathbf{V}^{t}).
$
By Assumption~\ref{ass:smooth}, $\blue{J_t}(\boldsymbol{\theta},\mathbf{V}^{t})$ has $L$-Lipschitz gradient. Applying Lemma~\ref{lem:descent_V} with $\eta\le 1/L$ yields
\begin{equation}\label{eq:Descent_ineq}
\blue{J_t}(\boldsymbol{\theta}^{t+1},\mathbf{V}^{t})
\le
\blue{J_t}(\boldsymbol{\theta}^t,\mathbf{V}^{t})
-
\frac{\eta}{2}
\|
\nabla_{\boldsymbol{\theta}}
\blue{J_t}(\boldsymbol{\theta}^t,\mathbf{V}^{t})
\|^2.    
\end{equation}
By exact minimization of $\mathbf{V}$ at iteration $t+1$
(Assumption~\ref{ass:exact}) \blue{and \cref{lem:drift}},
\blue{\[
J_{t+1}(\boldsymbol{\theta}^{t+1},\mathbf{V}^{t+1})
\le
J_{t+1}(\boldsymbol{\theta}^{t+1},\mathbf{V}^{t})
\le
J_t(\boldsymbol{\theta}^{t+1},\mathbf{V}^{t})+\gamma .
\]}
Combining this inequality with \eqref{eq:Descent_ineq} yields \blue{\eqref{eq:quasidescent}}.

\textbf{(ii) Sublinear Stationarity Rate.} From \blue{the quasi-descent inequality \eqref{eq:quasidescent}}, summing from $t=0$ to $T$ and telescoping yields
\[
\frac{\eta}{2}
\sum_{t=0}^{T}
\|
\nabla_{\boldsymbol{\theta}}
\blue{J_t}(\boldsymbol{\theta}^t,\mathbf{V}^{t})
\|^2
\le
\blue{J_0}(\boldsymbol{\theta}^0,\mathbf{V}^0)
-
\blue{J_{T+1}(\boldsymbol{\theta}^{T+1},\mathbf{V}^{T+1})
+\gamma(T+1)}.
\]
Since
$\blue{J_{T+1}(\boldsymbol{\theta}^{T+1},\mathbf{V}^{T+1})}\ge J_{\inf}$ (Assumption~\ref{ass:lower}),
\[
\sum_{t=0}^{T}
\|
\nabla_{\boldsymbol{\theta}}
\blue{J_t}(\boldsymbol{\theta}^t,\mathbf{V}^{t})
\|^2
\le
\frac{
2\big(
\blue{J_0}(\boldsymbol{\theta}^0,\mathbf{V}^0)
-
J_{\inf}
\big)
}{\eta}
\blue{+\frac{2\gamma(T+1)}{\eta}}.
\]
Finally, using
$\min_{0\le t\le T} a_t
\le
\frac{1}{T+1}
\sum_{t=0}^{T} a_t$
with
$
a_t =
\|
\nabla_{\boldsymbol{\theta}}
\blue{J_t}(\boldsymbol{\theta}^t,\mathbf{V}^{t})
\|^2
$,
we obtain~\eqref{eq:rate_main_Q}.

\textbf{(iii) Convergence and Stationarity.}
\blue{For $\alpha=1$, $L_{\mathbf m}=0$, so that $\gamma=0$ and $J_t$ does not depend on $t$, and the argument below applies with $\mathbf{m}^\star$ playing no role.
For a schedule $\alpha_t$, \cref{lem:drift} gives $\gamma_t=L_{\mathbf m}\alpha_tD$, with $\Gamma:=\sum_{t\ge0}\gamma_t<\infty$, and \eqref{eq:quasidescent} holds with $\gamma_t$ in place of $\gamma$.
Hence the sequence $\{J_t(\boldsymbol{\theta}^t,\mathbf{V}^{t})+\sum_{s\ge t}\gamma_s\}$ is non-increasing and bounded below by $J_{\inf}$ (Assumption~\ref{ass:lower}), so it converges, and so does $\{J_t(\boldsymbol{\theta}^t,\mathbf{V}^{t})\}$, since $\sum_{s\ge t}\gamma_s\to0$.}

Summing \blue{\eqref{eq:quasidescent}} over $t$ yields
$\sum_{t=0}^{\infty}
\|
\nabla_{\boldsymbol{\theta}}
\blue{J_t}(\boldsymbol{\theta}^t,\mathbf{V}^{t})
\|^2
\le \blue{2\big(J_0(\boldsymbol{\theta}^0,\mathbf{V}^0)-J_{\inf}+\Gamma\big)/\eta}
< \infty,$ which implies 
$$\lim_{t\to\infty}
\|
\nabla_{\boldsymbol{\theta}}
\blue{J_t}(\boldsymbol{\theta}^t,\mathbf{V}^{t})
\|
=
0.$$
\blue{Moreover, $\sum_{t\ge0}\|\mathbf{m}^{t+1}-\mathbf{m}^t\|\le D\sum_{t\ge0}\alpha_t<\infty$ by \cref{eq:mstep}, so that $\mathbf{m}^t\to\mathbf{m}^\star$.}
Since \blue{$J_t(\bm\theta^t,\V^t)\le J_0(\bm\theta^0,\V^0)+\Gamma$}, the iterates remain in a bounded sublevel set (Assumption\nb\ref{ass:lower}), so the sequence admits a convergent subsequence.
Let $(\boldsymbol{\theta}^{t_k},\mathbf{V}^{t_k})$ be a convergent subsequence\blue{, with limit $(\boldsymbol{\theta}^\star,\mathbf{V}^\star)$}. By continuity of $\nabla_{\boldsymbol{\theta}}J$ \blue{in $(\boldsymbol{\theta},\mathbf{V},\mathbf{m})$ (Assumption~\ref{ass:smooth})}, it holds
$\nabla_{\boldsymbol{\theta}}
J(\boldsymbol{\theta}^\star,\mathbf{V}^\star\blue{;\mathbf{m}^\star})=0.$
Moreover, since $\mathbf{V}^{t}$ minimizes
$\blue{J_t}(\boldsymbol{\theta}^{t},\cdot)$ over the compact set $\mathcal V$,
passing to the limit \blue{by continuity of $J$ in $(\boldsymbol{\theta},\mathbf{V},\mathbf{m})$} gives
\[
\mathbf{V}^\star
\in
\arg\min_{\mathbf{V}\in\mathcal V}
J(\boldsymbol{\theta}^\star,\mathbf{V}\blue{;\mathbf{m}^\star}).
\]
Thus every accumulation point is block-wise first-order stationary.
\end{proof}

\blue{
\begin{remark}[Full-batch normalization]\label{rem:full_batch}
In the deterministic setting the batch mean is the full-sample mean, which is a deterministic function of $\bm\theta$, so a running estimate is not needed and the natural choice is $\alpha=1$.
Then $L_{\mathbf m}=0$, $\gamma=0$ and $J_t=J$ for all $t$, so that \cref{thm:main_Q} yields monotone descent of $\{J(\bm\theta^t,\V^t)\}$, the rate \eqref{eq:rate_main_Q} without the neighbourhood term, and convergence to block-wise first-order stationary points.
We state \cref{thm:main_Q} for a general momentum $\alpha\in(0,1]$ because the same one-step argument, including the drift term, is the basis of the stochastic analysis in App.~\ref{subapp:stochastic-setting}, where the running mean is mixed into the minibatch statistics.
\end{remark}}

\subsection{Stochastic Setting}\label{subapp:stochastic-setting}

We consider the \blue{objectives $J_t$ of the deterministic setting}, and extend the $\boldsymbol{\theta}$-update to the stochastic (mini-batch) setting while keeping the $\mathbf{V}$-block \emph{exactly} minimized with respect to \blue{$J_t$}. Let $\{\Phi_t\}_{t\ge 0}$ be the natural filtration generated by the iterates\blue{, including the running means,} and the sampling randomness up to time $t$. At iteration $t$, we first update the alignment variables by exact minimization,
\begin{equation}
\mathbf{V}^{t}
\in
\arg\min_{\mathbf{V}\in\mathcal V}
\blue{J_t}(\boldsymbol{\theta}^{t},\mathbf{V}),
\label{eq:exactVstep_V}
\end{equation}
and then form a mini-batch stochastic gradient estimator
$\mathbf{g}^t \equiv \mathbf{g}(\boldsymbol{\theta}^t,\mathbf{V}^{t};\xi_t),$
where $\xi_t$ denotes the mini-batch sampling randomness. The stochastic update of the neural parameters is
\begin{equation}
\boldsymbol{\theta}^{t+1}
=
\boldsymbol{\theta}^t
-
\eta_t\,\mathbf{g}^t,
\label{eq:sgdstep_V}
\end{equation}
\blue{followed by the update \cref{eq:mstep} of the running means on the sampled mini-batch $B_t$.
Since $\mathbf{m}^t$ and $\mathbf{V}^t$ are $\Phi_t$-measurable and $\bm\mu^{B}$ is differentiated through, mini-batch gradients are unbiased estimators of $\nabla_{\boldsymbol{\theta}}J_t(\boldsymbol{\theta}^t,\mathbf{V}^t)$ by construction.}

To analyze the convergence behavior of the stochastic scheme, we impose the following assumptions on the stochastic gradient estimator.

\begin{assumption}[Unbiased stochastic gradient]
\label{ass:unbiased}
For all $t\ge 0$,
\[
\mathbb E[\mathbf{g}^t\mid \Phi_t]
=
\nabla_{\boldsymbol{\theta}} \blue{J_t}(\boldsymbol{\theta}^t,\mathbf{V}^{t}).
\]
\end{assumption}

\begin{assumption}[Bounded conditional variance]
\label{ass:var}
There exists $\sigma^2<\infty$ such that for all $t\ge 0$,
\[
\mathbb E\!\left[
\big\|
\mathbf{g}^t-\nabla_{\boldsymbol{\theta}}\blue{J_t}(\boldsymbol{\theta}^t,\mathbf{V}^{t})
\big\|^2
\ \middle|\ \Phi_t
\right]
\le
\sigma^2.
\]
\end{assumption}

Under these assumptions, we establish convergence guarantees for the stochastic variant of the proposed method in the following sections.

\subsubsection{Convergence Guarantee with Optimized Step Size}\label{subsec:stochastic_theorem}

\begin{theorem}[Stochastic convergence rate]
\label{thm:stochastic_rate_V}
Assume \cref{ass:lower,ass:smooth,ass:exact,ass:compact,ass:running_mean}, and the stochastic gradient conditions in \cref{ass:unbiased,ass:var}.
Let $\{(\boldsymbol{\theta}^t,\mathbf{V}^t)\}_{t\ge 0}$ be generated by
\eqref{eq:exactVstep_V}--\eqref{eq:sgdstep_V} \blue{and \cref{eq:mstep}}.
Define $\Delta_0 := \blue{J_0}(\boldsymbol{\theta}^0,\mathbf{V}^0)-J_{\inf}$, where
$J_{\inf} := \inf_{\boldsymbol{\theta},\mathbf{V}\blue{,\mathbf{m}}} J(\boldsymbol{\theta},\mathbf{V}\blue{;\mathbf{m}})$\blue{, and let $A_T:=\Delta_0/(T+1)+\gamma$}.
Then, for any horizon $T\ge 0$, choosing the constant step size
\[
\eta_t \equiv \eta
:=
\min\!\left\{
\frac{1}{L},
\sqrt{\frac{2\blue{A_T}}{L\sigma^2}}
\right\}
\]
yields the bound
\[
\min_{0\le t\le T}
\mathbb E\!\left[
\left\|
\nabla_{\boldsymbol{\theta}}
\blue{J_t}(\boldsymbol{\theta}^t,\mathbf{V}^{t})
\right\|^2
\right]
\le
2\sqrt{2L\sigma^2\,\blue{A_T}}
\;+\;
2L\blue{A_T}.
\]
In particular, the method attains the stochastic convergence rate
\[
\min_{0\le t\le T}
\mathbb E\!\left[
\left\|
\nabla_{\boldsymbol{\theta}}
\blue{J_t}(\boldsymbol{\theta}^t,\mathbf{V}^{t})
\right\|^2
\right]
=
\mathcal O\!\left(\frac{1}{\sqrt{T}}\right)
\blue{+2\sqrt{2L\sigma^2\gamma}+2L\gamma},
\]
\blue{where the neighbourhood term vanishes for $\alpha=1$.}
\end{theorem}

\begin{proof}
For brevity, define
\[
\blue{\bar J_t} := \blue{J_t}(\boldsymbol{\theta}^t,\mathbf{V}^t),
\qquad
\nabla_t := \nabla_{\boldsymbol{\theta}}\blue{J_t}(\boldsymbol{\theta}^t,\mathbf{V}^{t}),
\qquad
\mathbf{g}^t := \mathbf{g}(\boldsymbol{\theta}^t,\mathbf{V}^{t};\xi_t).
\]
Fix $t\ge 0$. By Assumption~\ref{ass:smooth}, for the $L$-smooth map
$\boldsymbol{\theta}\mapsto \blue{J_t}(\boldsymbol{\theta},\mathbf{V}^{t})$, we have for any vector $\mathbf{u}$,
\[
\blue{J_t}(\boldsymbol{\theta}^t+\mathbf{u},\mathbf{V}^{t})
\le
\blue{\bar J_t}
+
\nabla_t^\top \mathbf{u}
+
\frac{L}{2}\|\mathbf{u}\|^2.
\]
Applying this with $\mathbf{u}=-\eta\,\mathbf{g}^t$ yields
\begin{equation}
\label{eq:smooth_step_bound}
\blue{J_t}(\boldsymbol{\theta}^{t+1},\mathbf{V}^{t})
\le
\blue{\bar J_t}
-
\eta\,\nabla_t^\top \mathbf{g}^t
+
\frac{L\eta^2}{2}\|\mathbf{g}^t\|^2.
\end{equation}
By exact minimization in $\mathbf{V}$ (Assumption~\ref{ass:exact}) \blue{and \cref{lem:drift}},
\begin{equation}
\label{eq:exact_V_step_bound}
\blue{\bar J_{t+1}
\le
J_{t+1}(\boldsymbol{\theta}^{t+1},\mathbf{V}^t)
\le
J_t(\boldsymbol{\theta}^{t+1},\mathbf{V}^t)+\gamma.}
\end{equation}
Combining \eqref{eq:smooth_step_bound}--\eqref{eq:exact_V_step_bound} gives
\begin{equation}
\label{eq:one_step_before_expectation}
\blue{\bar J_{t+1}}
\le
\blue{\bar J_t}
-
\eta\,\nabla_t^\top \mathbf{g}^t
+
\frac{L\eta^2}{2}\|\mathbf{g}^t\|^2
\blue{+\gamma}.
\end{equation}

Taking conditional expectation w.r.t.\ $\Phi_t$ and using unbiasedness
(Assumption~\ref{ass:unbiased}) gives
\[
\mathbb E\!\left[\nabla_t^\top \mathbf{g}^t\mid \Phi_t\right]
=
\nabla_t^\top \mathbb E[\mathbf{g}^t\mid \Phi_t]
=
\|\nabla_t\|^2.
\]
Moreover, expanding the second moment and using the bounded conditional variance
(Assumption~\ref{ass:var}) yields
\begin{align*}
\mathbb E\!\left[\|\mathbf{g}^t\|^2\mid \Phi_t\right]
&=
\mathbb E\!\left[\|\mathbf{g}^t-\nabla_t+\nabla_t\|^2\mid \Phi_t\right]\\
&=
\|\nabla_t\|^2
+
\mathbb E\!\left[\|\mathbf{g}^t-\nabla_t\|^2\mid \Phi_t\right]
+
2\,\nabla_t^\top
\mathbb E[\mathbf{g}^t-\nabla_t\mid \Phi_t]\\
&\le
\|\nabla_t\|^2+\sigma^2,
\end{align*}
since $\mathbb E[\mathbf{g}^t-\nabla_t\mid \Phi_t]=\mathbf{0}$ by Assumption~\ref{ass:unbiased}.
Substituting these two identities into the conditional expectation of
\eqref{eq:one_step_before_expectation} yields
\begin{equation}
\label{eq:one_step_expected}
\mathbb E[\blue{\bar J_{t+1}}\mid \Phi_t]
\le
\blue{\bar J_t}
-
\eta\Big(1-\frac{L\eta}{2}\Big)\|\nabla_t\|^2
+
\frac{L\eta^2}{2}\sigma^2
\blue{+\gamma}.
\end{equation}
Assuming $0<\eta\le 1/L$ implies $1-\frac{L\eta}{2}\ge \frac12$, hence
\begin{equation}
\label{eq:one_step_expected_simple}
\mathbb E[\blue{\bar J_{t+1}}\mid \Phi_t]
\le
\blue{\bar J_t}
-
\frac{\eta}{2}\|\nabla_t\|^2
+
\frac{L\eta^2}{2}\sigma^2
\blue{+\gamma}.
\end{equation}

Taking total expectation in \eqref{eq:one_step_expected_simple} and summing from $t=0$ to $T$ gives
\[
\mathbb E[\blue{\bar J_{T+1}}]
\le
\blue{\bar J_0}
-
\frac{\eta}{2}\sum_{t=0}^{T}\mathbb E\|\nabla_t\|^2
+
\Big(\frac{L\eta^2}{2}\sigma^2\blue{+\gamma}\Big)(T+1).
\]
Rearranging and using the lower bound $\mathbb E[\blue{\bar J_{T+1}}]\ge J_{\inf}$ (Assumption~\ref{ass:lower}) yields
\begin{equation}
\label{eq:sum_grad_bound}
\frac{\eta}{2}\sum_{t=0}^{T}\mathbb E\|\nabla_t\|^2
\le
\Delta_0 \blue{+\gamma(T+1)} + \frac{L\eta^2}{2}\sigma^2(T+1),
\qquad
\Delta_0:=\blue{\bar J_0}-J_{\inf}.
\end{equation}
Dividing by $\eta(T+1)/2$ and using
$\min_{0\le t\le T} a_t \le \frac{1}{T+1}\sum_{t=0}^T a_t$ gives
\begin{equation}
\label{eq:min_grad_bound_eta}
\min_{0\le t\le T}\mathbb E\|\nabla_t\|^2
\le
\frac{2\blue{A_T}}{\eta} + L\eta\,\sigma^2,
\qquad
\blue{A_T:=\frac{\Delta_0}{T+1}+\gamma}.
\end{equation}

We now optimize the right-hand side in $\eta$ under the constraint $\eta\le 1/L$.
Let
$$\eta^\star := \sqrt{\frac{2\blue{A_T}}{L\sigma^2}}.$$ If $\eta^\star\le 1/L$, choosing $\eta=\eta^\star$ in \eqref{eq:min_grad_bound_eta} yields
\[
\min_{0\le t\le T}\mathbb E\|\nabla_t\|^2
\le
2\sqrt{2L\sigma^2\blue{A_T}}.
\]
If instead $\eta^\star>1/L$, we choose $\eta=1/L$ in \eqref{eq:min_grad_bound_eta} to obtain
\[
\min_{0\le t\le T}\mathbb E\|\nabla_t\|^2
\le
2L\blue{A_T}+\sigma^2\blue{\le 2L A_T+2\sqrt{2L\sigma^2A_T}},
\]
\blue{where the last inequality follows from $\sigma^2<2LA_T$, which is equivalent to $\eta^\star>1/L$, so that $\sigma^2=\sigma\cdot\sigma<\sigma\sqrt{2LA_T}$.}
Both cases are covered by the choice
\[
\eta=\min\left\{\frac{1}{L},\ \sqrt{\frac{2\blue{A_T}}{L\sigma^2}}\right\}.
\]
Moreover, with this choice we have the unified bound
\[
\min_{0\le t\le T}\mathbb E\|\nabla_t\|^2
\le
2\sqrt{2L\sigma^2\blue{A_T}}
+
2L\blue{A_T},
\]
where the second term is redundant when $\eta=\eta^\star$ but remains valid.
\blue{Since $A_T\to\gamma$ as $T\to\infty$, the bound is $\mathcal O(1/\sqrt{T})$ up to the neighbourhood term $2\sqrt{2L\sigma^2\gamma}+2L\gamma$.}
This proves Theorem~\ref{thm:stochastic_rate_V}.
\end{proof}

\subsubsection{Almost-sure Convergence with Diminishing Step Sizes}\label{subsec:stochastic_as}

We establish an almost-sure convergence guarantee under a Robbins--Monro step-size schedule.
Throughout, let $\{\Phi_t\}_{t\ge 0}$ denote the natural filtration and let the iterates be generated by
\eqref{eq:exactVstep_V}--\eqref{eq:sgdstep_V} \blue{and \cref{eq:mstep}}.

\begin{assumption}[Robbins--Monro step sizes]
\label{ass:rm_steps}
The step sizes satisfy $0<\eta_t\le 1/L$ for all $t\ge 0$, and
\[
\sum_{t=0}^\infty \eta_t = \infty,
\qquad
\sum_{t=0}^\infty \eta_t^2 < \infty.
\]
\end{assumption}

\begin{theorem}[Almost-sure convergence]
\label{thm:stochastic_as_V}
Assume \Cref{ass:lower,ass:smooth,ass:exact,ass:compact,ass:running_mean},
and \Cref{ass:unbiased,ass:var} and \cref{ass:rm_steps}\blue{, with either $\alpha=1$ or a momentum schedule $\alpha_t\in(0,1]$ such that $\sum_{t}\alpha_t<\infty$}.
Let $\{(\boldsymbol{\theta}^t,\mathbf{V}^t)\}_{t\ge 0}$ be generated by
\eqref{eq:exactVstep_V}--\eqref{eq:sgdstep_V} \blue{and \cref{eq:mstep}}.
Then:
\begin{enumerate}
\item The sequence $\{\blue{J_t}(\boldsymbol{\theta}^t,\mathbf{V}^t)\}$ converges almost surely to a finite random variable.

\item Moreover,
\[
\sum_{t=0}^\infty
\eta_t
\left\|
\nabla_{\boldsymbol{\theta}}
\blue{J_t}(\boldsymbol{\theta}^t,\mathbf{V}^{t})
\right\|^2
<\infty
\qquad \text{almost surely}.
\]

\item Consequently,
\[
\liminf_{t\to\infty}
\left\|
\nabla_{\boldsymbol{\theta}}
\blue{J_t}(\boldsymbol{\theta}^t,\mathbf{V}^{t})
\right\|
=
0
\qquad \text{almost surely}.
\]
\end{enumerate}
\end{theorem}

\begin{proof} We begin by recalling the Robbins--Siegmund lemma, which provides a standard convergence result for quasi-supermartingale sequences.

\begin{lemma}[Robbins--Siegmund, \citealp{robbins1971convergence}]
\label{lem:robbins_siegmund}
Let $\{X_t\}_{t\ge 0}$ be a nonnegative adapted sequence with respect to a filtration
$\{\Phi_t\}_{t\ge 0}$.
Assume there exist nonnegative $\Phi_t$-measurable sequences
$\{a_t\}_{t\ge 0}$ and $\{b_t\}_{t\ge 0}$ such that
\[
\mathbb E[X_{t+1}\mid \Phi_t] \le X_t - a_t + b_t
\qquad \text{a.s. for all } t,
\]
and $\sum_{t=0}^{\infty} b_t < \infty$ almost surely.
Then $X_t$ converges almost surely to a finite random variable and
$\sum_{t=0}^{\infty} a_t < \infty$ almost surely.
\end{lemma}

We now show that the sequence generated by the algorithm satisfies the conditions of this lemma.
Define
\[
\blue{\bar J_t} := \blue{J_t}(\boldsymbol{\theta}^t,\mathbf{V}^t),
\qquad
\nabla_t := \nabla_{\boldsymbol{\theta}}\blue{J_t}(\boldsymbol{\theta}^t,\mathbf{V}^{t}),
\qquad
X_t := \blue{\bar J_t} - J_{\inf} \ge 0.
\]

From the one-step bound (cf.\ \eqref{eq:one_step_expected_simple})\blue{, with $\eta_t$ in place of $\eta$ and $\gamma_t=L_{\mathbf m}\alpha_tD$ in place of $\gamma$ (\cref{lem:drift})}, for all $t\ge 0$,
\[
\mathbb E[\blue{\bar J_{t+1}}\mid \Phi_t]
\le
\blue{\bar J_t}
-
\frac{\eta_t}{2}\|\nabla_t\|^2
+
\frac{L\eta_t^2}{2}\sigma^2
\blue{+\gamma_t}.
\]
Subtracting $J_{\inf}$ from both sides yields
\[
\mathbb E[X_{t+1}\mid \Phi_t]
\le
X_t - a_t + b_t,
\qquad
a_t := \frac{\eta_t}{2}\|\nabla_t\|^2,
\quad
b_t := \frac{L\eta_t^2}{2}\sigma^2\blue{+\gamma_t}.
\]

By Assumption~\ref{ass:rm_steps} we have $\sum_{t=0}^{\infty}\eta_t^2<\infty$\blue{, and $\sum_{t=0}^{\infty}\gamma_t<\infty$ since either $\gamma_t=0$ ($\alpha=1$) or $\sum_t\alpha_t<\infty$}, hence
$\sum_{t=0}^{\infty} b_t < \infty$.
Therefore all the conditions of Lemma~\ref{lem:robbins_siegmund} are satisfied.
Applying the lemma implies that $X_t$ converges almost surely to a finite random variable,
and
\[
\sum_{t=0}^{\infty}\eta_t\|\nabla_t\|^2 < \infty
\qquad \text{a.s.}
\]

Finally, since $\sum_{t=0}^{\infty}\eta_t=\infty$ (Assumption~\ref{ass:rm_steps}),
the above summability implies
\[
\liminf_{t\to\infty}\|\nabla_t\| = 0
\qquad \text{a.s.}
\]
This proves Theorem~\ref{thm:stochastic_as_V}.
\end{proof}

\blue{
\begin{remark}[Gated alignment]
The analysis covers the symmetric update $a=\frac12$, which is gradient descent on $J$ with $\lambda_{ij}$ halved.
The confidence gate of Eq.~\eqref{eq:gate} varies across pilots and iterations and depends on the parameters, so its convergence is not covered by this analysis.
\end{remark}}
\section{Experimental details}\label{app:exp_details}

\blue{\spara{Common setup.}
Every agent consists of a convolutional encoder, global average pooling over the last feature map, the CenterNorm layer of \cref{eq:cn}, and an MLP head with one hidden layer of width $256$ (Linear, batch normalization, ReLU, Linear).
The encoders are built from two families of blocks: a CNN block, i.e., a $3\times3$ convolution, batch normalization and ReLU, optionally followed by $2\times2$ max-pooling, and the residual basic block of ResNet-18/34\nb\citep{he2016deep}.
Every agent trains with stochastic gradient descent with Nesterov momentum $0.9$, learning rate $0.05$, weight decay $5\cdot10^{-4}$, a cosine learning-rate schedule and gradient clipping at norm $1.0$.
The gluing penalty is evaluated on pilot minibatches of $128$ samples.
The graph is fully connected.}

\blue{\spara{Data.}
Each agent $i$ receives samples only from its target classes $\mathcal{C}_i$, with an equal number of samples per class.
The agents draw independently from the same training pool, so that two agents sharing a class may receive some of the same samples; these are not paired across agents, i.e., no correspondence between them is known or used, and they enter only the local objectives.
On FEMNIST, where the writers are partitioned among the agents, the local datasets are disjoint.
The pilot set is drawn uniformly at random from the training pool, over all classes, and removed from the pool before the local datasets are drawn, so that it is disjoint from every local dataset; the same pilots are given to every agent and used by every method at a given seed, while different seeds draw different pilot sets.
\cref{tab:setups} summarizes the four benchmarks.}

\spara{Evaluation metrics.}
\blue{Let $\mathcal{T}=\{(\x^n,y^n)\}_{n=1}^{M}$ be the test set, shared by all agents, and let $\mathcal{T}_i\coloneqq\{n: y^n\in\mathcal{C}_i\}$ index the test samples whose class agent $i$ is able to classify.
Then,
\begin{equation}\label{eq:monitored_metrics}
  \begin{aligned}
   A_i^\text{priv} &= \dfrac{1}{M} \sum_{n=1}^{M}\ones\Big[\arg \max_{c \in \mathcal{Y}} [h_{\bm\psi_i} \circ \mulberry{f_{\bm\varphi_i}(\x^{n})}]_c=y^{n}\Big]\,,\\
   A_i^\text{comm} &= \dfrac{1}{|\vertexset(i)|} \sum_{j \in \vertexset(i)} \dfrac{1}{|\mathcal{T}_i|} \sum_{n\in\mathcal{T}_i}\ones\Big[\arg \max_{c \in \mathcal{Y}}[h_{\bm\psi_i} \circ \T_{ij} \circ \teal{f_{\bm\varphi_j}(\x^{n})}]_c = y^{n}\Big]\,,
  \end{aligned}
\end{equation}
where $\T_{ij}$ transports the latent representations of the sender $j$ into the latent space of the receiver $i$.
For Sheaf-FRL and DiSheaf-FRL, $\T_{ij}=\V_{ij}$ are the restriction maps learned during training, with the convention $\V_{ij}\coloneqq\V_{ji}^\top$ for $j\in\vertexset(i)^+$ of \cref{subsec:gradient}.
For ComFed, $\T_{ij}$ is the map learned by the method through the projections onto its shared space.
For the non-cooperative baseline, Sheaf-FMTL, FedProto, FedMuscle and FedRE, $\T_{ij}=\U\W^\top$, where $\U\bm\Sigma\W^\top$ is the thin SVD of $\A_i\A_j^\top$, i.e., the Procrustes map of \cref{eq:procrustes_solution_V} fitted on the pilot representations at the end of training, with the CenterNorm layer in inference mode.
With a fully connected graph, the average of $A_i^\text{comm}$ across agents is the average over the $N(N-1)$ directed edges.
Restricting the inner sum to $\{n: y^n\in\mathcal{C}_i\setminus\mathcal{C}_j\}$ yields the accuracy on the classes that the receiver knows but the sender has never been trained on.}

\blue{\spara{Benchmarks.}
\emph{MNIST.} Four plain CNNs with three or four CNN blocks, each followed by max-pooling, and latent dimensions $128$, $224$, $256$ and $384$.
Each agent trains on seven classes, with $35{,}272$ samples per agent; the pilot set has $K=4096$ samples, and training lasts $100$ epochs with batch size $256$.
\emph{Fashion-MNIST.} Four ResNets, two ResNet-10s with four stages of one block and two networks with three stages of two blocks, with latent dimensions $256$, $192$, $128$ and $160$.
Each agent trains on six classes, with $34{,}673$ samples per agent; the pilot set has $K=2048$ samples, and training lasts $40$ epochs with batch size $128$.
\emph{FEMNIST.} The digits of FEMNIST, $402{,}957$ images from $3{,}597$ writers, with the class sets and architectures of the MNIST benchmark.
The writers are shuffled, $10\%$ of them are reserved for the pilots, and the others are assigned to the agents in a round-robin fashion, so that no writer appears at two agents; each agent trains on $51{,}554$ samples of its own writers.
The test set merges a stratified $10\%$ hold-out of each agent's writers, so that about three quarters of the test writers are unseen by any given agent, and the $K=4096$ pilots are drawn from the pilot writers, which are unseen by every agent; training lasts $100$ epochs with batch size $256$.
\emph{CIFAR-10.} Two half-width ResNets, a ResNet-10 and a ResNet-18, and two VGG-style CNNs\nb\citep{simonyan2015very} with six convolutions in three pairs, with latent dimensions $64$, $32$, $8$ and $16$.
The agents train on seven or eight classes, with $30{,}000$ samples per agent, augmented with random crops and horizontal flips; the missing class sets are disjoint and cover all ten classes.
The pilot set has $K=4096$ samples, and training lasts $100$ epochs with batch size $128$.}

\begin{table}[t]
\centering
\caption{\blue{Benchmarks of \cref{sec:sem_comm}. For each agent: encoder, latent dimension $d_i$, number of parameters (encoder and head), and missing classes $\mathcal{Y}\setminus\mathcal{C}_i$. CNN: output channels of the CNN blocks; ResNet: output channels of the stages and number of basic blocks per stage.}}
\label{tab:setups}
\small
\begin{tabular}{llrrl}
\toprule
Benchmark & Encoder & $d_i$ & Params & Missing classes \\
\midrule
\multirow{4}{*}{MNIST / FEMNIST}
 & CNN $(32,64,128)$ & 128 & 129k & 7, 8, 9 \\
 & CNN $(32,64,128,224)$ & 224 & 413k & 0, 1, 2 \\
 & CNN $(32,96,256)$ & 256 & 319k & 2, 3, 4 \\
 & CNN $(32,64,192,384)$ & 384 & 897k & 5, 6, 7 \\
\midrule
\multirow{4}{*}{Fashion-MNIST}
 & ResNet $(64,128,256,256)$, 1-1-1-1 & 256 & 2.54M & 0, 1, 2, 3 \\
 & ResNet $(64,128,256,192)$, 1-1-1-1 & 192 & 2.10M & 0, 4, 5, 6 \\
 & ResNet $(64,128,128)$, 2-2-2 & 128 & 1.32M & 1, 4, 7, 8 \\
 & ResNet $(64,128,160)$, 1-2-2 & 160 & 1.54M & 2, 5, 7, 9 \\
\midrule
\multirow{4}{*}{CIFAR-10}
 & ResNet-10, half width & 64 & 446k & 0, 3, 6 \\
 & VGG-style CNN & 8 & 155k & 1, 4, 7 \\
 & ResNet-18, half width & 32 & 776k & 2, 5 \\
 & VGG-style CNN & 16 & 166k & 8, 9 \\
\bottomrule
\end{tabular}
\end{table}

\blue{\spara{Two-agent bottleneck sweep (\cref{sec:compression}).}
We use the digits of FEMNIST, with the writers shuffled and split into three disjoint groups: $10\%$ of them ($358$) provide the pilots, and the remaining ones are split evenly between the two agents ($1611$ each).
Agent~$0$ trains on $\{0,\dots,6\}$ and agent~$1$ on $\{3,\dots,9\}$, with about $103.5$k samples each ($14{,}793$ per class), the largest equal draw per class that both agents can take.
A stratified $10\%$ hold-out of the samples of each agent is pooled into a single test set of about $36.2$k images over all ten digits, shared by the two agents, and the $K=4096$ pilots are drawn from the pilot writers, over all classes.
Both agents are plain CNNs with four CNN blocks, each followed by max-pooling, which reduce the $28\times28$ input to $1\times1$.
Agent~$0$ has $(32,64,128,d)$ channels and agent~$1$ has $(32,96,256,d/2)$ channels, with $d\in\{8,16,32,64\}$, so that the last block sets the latent dimension and the two latent spaces differ in dimension at every $d$.
Training lasts $100$ epochs with batch size $256$ and the common optimizer above.
The regularization coefficients are fixed to the values selected on the four-agent FEMNIST benchmark: $\lambda=3\cdot10^{-2}$ for Sheaf-FRL and ComFed, $10^{-1}$ for DiSheaf-FRL, $10^{-2}$ for Sheaf-FMTL and FedProto, and $0.3$ for FedMuscle.}

\spara{Selection of the regularization coefficient.}
For \blue{every method}, the regularization coefficient is selected by grid search in the $[10^{-5}, 1]$ interval.
\blue{For the sheaf-based methods, the selected values are $\lambda=10^{-2}$ on MNIST, $\lambda=3.98\cdot10^{-3}$ on Fashion-MNIST, tuned on a two-agent version of the benchmark, $\lambda=3\cdot10^{-2}$ for Sheaf-FRL and $8\cdot10^{-2}$ for DiSheaf-FRL on CIFAR-10, and $\lambda=3\cdot10^{-2}$ for Sheaf-FRL and $10^{-1}$ for DiSheaf-FRL on FEMNIST, where they are also used in the bottleneck sweep.
The temperature of the confidence gate in \cref{eq:gate} is $\tau=0.1$.}

\spara{Method-specific hyperparameters.}
Beyond the regularization coefficient, every baseline also carries its own additional, fixed hyperparameters that are untouched by the grid-search above.
\blue{Sheaf-FMTL introduces two further constants: the compression factor $\gamma$, which sets the dimension of the shared space on each edge and is discussed below, and the personalization rate $\eta=10^{-2}$.}
For further information to support the selection process of the Sheaf-FMTL hyperparameter $\gamma$ refer to the paragraph \enquote{Memory footprint and choice of $\gamma$ for Sheaf-FMTL} below.
ComFed instead learns, for every agent $i$, an unconstrained linear projection matrix $\myP_i \in \reall^{r \times d_i}$ mapping that agent's local $d_i$-dimensional latent space into a shared space of dimension $r$, common to every agent, in which class prototypes from all agents are aligned; \blue{we set $r=d_{\min}=\min_id_i$, which in the bottleneck sweep is $d/2$.}
FedProto and FedMuscle assume a common latent dimensionality, which does not hold under model heterogeneity.
We therefore equip each agent with the same construction used for ComFed, a learnable linear projection $\myP_i \in \reall^{r \times d_i}$ onto a shared space of dimension $r$, applied before the respective regularizer, \blue{with $r=d_{\max}=\max_id_i$, so that no agent is projected onto a space smaller than its own; in the bottleneck sweep, $r=d$.}
FedProto maintains a single global prototype per class, updated as an exponential moving average of the agents' local per-class embeddings with momentum $0.9$.
FedMuscle organizes training into repeating communication rounds, each consisting of four local epochs of ordinary task-loss training followed by one additional epoch in which a contrastive alignment loss, with temperature $0.1$, pulls every agent's representations toward the network average.
\blue{FedRE replaces the local heads with a head shared by all agents and trained by the server\nb\citep{yao2026fedre}.
At the end of every local epoch, each agent uploads a random convex combination of its class prototypes, together with the corresponding soft label, and the server runs one epoch of SGD on these $N$ pairs, with learning rate $10^{-2}$ and batch size $10$, before broadcasting the head.
To handle heterogeneous latent dimensions, each agent feeds the shared head through a private linear layer of output dimension $d_{\max}$, trained only by its local loss.
FedRE does not use the pilots during training, and its latent spaces are aligned post hoc for communication, as for the other baselines.}

\blue{\spara{Compression factor of Sheaf-FMTL.}
Sheaf-FMTL keeps a restriction map $\myP_{ij}\in\reall^{p_{ij}\times p_i}$ for every directed edge $i\to j$, acting on the parameter vectors, where $p_i$ is the number of trainable parameters of agent $i$ and $p_{ij}=\max(1,\lfloor\gamma\min(p_i,p_j)\rfloor)$.
The factor $\gamma$ thus sets the dimension of the shared space on each edge, and with it both the communication payload and the memory occupied by the maps.
The maps are stored in half precision and updated one edge at a time in single precision, so that their memory footprint, in bytes, is
\begin{equation*}
M(\gamma)=2\sum_{i\to j}p_{ij}\,p_i+8\max_{i\to j}p_{ij}\,p_i\,,
\end{equation*}
which is linear in $\gamma$, up to the floor in $p_{ij}$.
We set $\gamma$ as the largest value whose footprint fits in $90\%$ of the available memory $B$, which gives in closed form
\begin{equation*}
\gamma^\star=\frac{B}{2\sum_{i\to j}m_{ij}\,p_i+8\max_{i\to j}m_{ij}\,p_i}\,,\qquad m_{ij}=\min(p_i,p_j)\,.
\end{equation*}
For four agents with $P$ parameters in total, of equal size and on a fully connected graph, $\gamma^\star\approx B/(2P^2)$, so that $\gamma$ can be set from the available memory and the parameter count alone.
}

\section{\blue{Alignment}}
\label{app:alignment}

\blue{We complement the accuracy results of \cref{tab:main} with three diagnostics of the geometry of the learned latent spaces, computed on the test set of each benchmark of \cref{sec:sem_comm}.
\begin{itemize}
    \item \emph{Mean reciprocal rank (MRR).} For every directed edge $j\to i$ and test sample $\x^n$, we rank all the test embeddings of the receiver $i$ by their cosine similarity with the transported embedding $\T_{ij}f_{\bm\varphi_j}(\x^n)$, and denote by $\mathrm{rank}_n$ the position of the embedding of the same sample, $f_{\bm\varphi_i}(\x^n)$.
    The MRR is $\frac{100}{M}\sum_{n=1}^{M}1/\mathrm{rank}_n$, averaged over the directed edges, so that it ranges in $[0,100]$ and equals $100$ when every transported embedding is closest to its own counterpart.
    \item \emph{Total variation ($\mathcal{TV}$).} The sheaf total variation of \cref{eq:total_variation} of the test embeddings, averaged over the test samples; unlike in App.\nb\ref{app:global-sections}, it is not normalized by the energy of the embeddings.
    \item \emph{Smallest eigenvalue $\mu_{\min}$} of the sheaf Laplacian of \cref{eq:ns_Laplacian}, which measures the best agreement attainable by a unit-norm $0$-cochain, as discussed in App.\nb\ref{app:global-sections}.
\end{itemize}
The restriction maps are those learned by Sheaf-FRL and DiSheaf-FRL, and the post-hoc Procrustes maps of \cref{eq:procrustes_solution_V} for the other methods.
For ComFed, whose maps are not isometries between the latent spaces, TV and $\mu_{\min}$ are not defined, and the MRR is computed with its learned maps.}

\begin{table}[t]
\centering
\scriptsize
\setlength{\tabcolsep}{4pt}
\renewcommand{\arraystretch}{1.15}
\caption{\blue{Alignment diagnostics on the test set, mean and standard deviation over $3$ seeds. Per dataset and column, methods are ranked in groups of statistically indistinguishable results, as in \cref{tab:main}: the best group is in bold, the second is underlined, and the third is in italics. ``--'': not defined.}}
\label{tab:geom}
\vspace{4pt}
\begin{tabular}{l@{\hspace{10pt}}lccc}
\toprule
Dataset & Method & MRR $\uparrow$ & $\mathcal{TV}$ $\downarrow$ & $\mu_{\min}$ $(\times10^{-4})$ $\downarrow$ \\
\midrule
\multirow{8}{*}{CIFAR-10} & Non-coop & \rankC{1.0_{\pm 0.01}} & \rankC{4.57_{\pm 0.076}} & \rankB{14.6_{\pm 4}} \\
 & FedRE & $0.8_{\pm 0.03}$ & \rankC{4.49_{\pm 0.024}} & $35.7_{\pm 7.8}$ \\
 & FedProto & \rankC{1.0_{\pm 0.01}} & $4.59_{\pm 0.033}$ & \rankB{19.2_{\pm 4.3}} \\
 & FedMuscle & \rankB{1.1_{\pm 0.03}} & \rankB{4.25_{\pm 0.0084}} & \rankB{14.4_{\pm 3.5}} \\
 & ComFed & $0.5_{\pm 0.02}$ & -- & -- \\
 & Sheaf-FMTL & \rankC{1.0_{\pm 0.06}} & $4.58_{\pm 0.009}$ & \rankC{22.2_{\pm 4.1}} \\
 & Sheaf-FRL & \rankA{1.8_{\pm 0.08}} & \rankA{2.39_{\pm 0.035}} & \rankA{0.435_{\pm 0.093}} \\
 & DiSheaf-FRL & \rankA{1.8_{\pm 0.03}} & \rankA{2.43_{\pm 0.021}} & \rankA{0.604_{\pm 0.4}} \\
\midrule
\multirow{8}{*}{FEMNIST} & Non-coop & \rankC{2.0_{\pm 0.07}} & \rankC{3.55_{\pm 0.093}} & \rankC{19.1_{\pm 2.4}} \\
 & FedRE & $1.4_{\pm 0.04}$ & \rankC{3.45_{\pm 0.034}} & \rankC{20.7_{\pm 4.1}} \\
 & FedProto & \rankC{2.1_{\pm 0.14}} & $3.54_{\pm 0.043}$ & \rankC{17.6_{\pm 1.8}} \\
 & FedMuscle & \rankB{5.5_{\pm 0.21}} & \rankC{3.52_{\pm 0.074}} & \rankB{6.44_{\pm 0.5}} \\
 & ComFed & $0.1_{\pm 0.01}$ & -- & -- \\
 & Sheaf-FMTL & \rankC{2.1_{\pm 0.10}} & \rankC{3.53_{\pm 0.071}} & \rankC{18.0_{\pm 0.6}} \\
 & Sheaf-FRL & \rankA{12.4_{\pm 1.06}} & \rankA{0.344_{\pm 0.0039}} & \rankA{0.00335_{\pm 0.00067}} \\
 & DiSheaf-FRL & \rankA{13.2_{\pm 2.23}} & \rankB{0.439_{\pm 0.025}} & \rankA{0.00261_{\pm 0.00032}} \\
\midrule
\multirow{8}{*}{MNIST} & Non-coop & $10.5_{\pm 0.32}$ & \rankC{3.70_{\pm 0.042}} & \rankC{18.8_{\pm 3.2}} \\
 & FedRE & $6.7_{\pm 0.79}$ & \rankC{3.73_{\pm 0.098}} & $30.7_{\pm 4.3}$ \\
 & FedProto & $11.2_{\pm 1.02}$ & \rankC{3.65_{\pm 0.062}} & \rankC{21.3_{\pm 1.9}} \\
 & FedMuscle & \rankC{25.7_{\pm 1.94}} & \rankC{3.98_{\pm 0.28}} & \rankC{17.4_{\pm 3.1}} \\
 & ComFed & $0.3_{\pm 0.03}$ & -- & -- \\
 & Sheaf-FMTL & $10.3_{\pm 2.03}$ & \rankC{3.74_{\pm 0.12}} & \rankC{21.6_{\pm 2.3}} \\
 & Sheaf-FRL & \rankA{54.2_{\pm 2.04}} & \rankA{0.611_{\pm 0.014}} & \rankA{0.00805_{\pm 0.00049}} \\
 & DiSheaf-FRL & \rankB{33.4_{\pm 0.56}} & \rankB{0.92_{\pm 0.014}} & \rankB{0.0316_{\pm 0.0026}} \\
\midrule
\multirow{8}{*}{Fashion-MNIST} & Non-coop & \rankC{1.8_{\pm 0.05}} & \rankC{5.36_{\pm 0.032}} & \rankC{80.1_{\pm 4.9}} \\
 & FedRE & $1.3_{\pm 0.03}$ & \rankB{5.21_{\pm 0.062}} & \rankC{76.8_{\pm 0.94}} \\
 & FedProto & \rankC{1.8_{\pm 0.06}} & \rankB{5.32_{\pm 0.065}} & \rankC{76.1_{\pm 8.3}} \\
 & FedMuscle & \rankB{4.6_{\pm 0.04}} & $6.03_{\pm 0.018}$ & \rankC{82.6_{\pm 1.9}} \\
 & ComFed & $0.3_{\pm 0.01}$ & -- & -- \\
 & Sheaf-FMTL & \rankC{1.8_{\pm 0.01}} & \rankC{5.35_{\pm 0.041}} & \rankC{84.5_{\pm 12}} \\
 & Sheaf-FRL & \rankA{7.5_{\pm 0.27}} & \rankA{0.974_{\pm 0.02}} & \rankB{0.132_{\pm 0.019}} \\
 & DiSheaf-FRL & \rankA{7.6_{\pm 0.27}} & \rankA{0.954_{\pm 0.011}} & \rankA{0.00893_{\pm 0.0016}} \\
\bottomrule
\end{tabular}
\end{table}

\blue{\spara{Results.}
The two variants of Sheaf-FRL attain the best value of every diagnostic on every benchmark (\cref{tab:geom}).
Their MRR is between $1.6$ and $2.4$ times that of the best baseline, so that a transported representation is more often recognized among the receiver's embeddings as the one of the same sample.
Their total variation is lower than that of the best baseline by a factor between about $2$ on CIFAR-10 and $10$ on FEMNIST, and the smallest eigenvalue of their sheaf Laplacian is between one and four orders of magnitude lower.
The baselines, instead, remain close to the non-cooperative one in total variation and, with the exception of FedMuscle, also in MRR and $\mu_{\min}$: their latent spaces are hardly better aligned than those of independently trained agents, consistently with their communication accuracy on the classes unseen by the sender.
Hence, the gain of Sheaf-FRL in communication accuracy is not only a property of the heads, but reflects latent spaces that are better aligned across the network, in agreement with the emergence of approximate global sections discussed in App.\nb\ref{app:global-sections}.}

\section{\blue{Robustness to semantic compression}}\label{sec:compression}
An important feature in semantic communication is robustness to latent-space compression \citep{gunduz2022beyond,barbarossa2023semantic}.
This can be understood as the ability of methods to extract semantically meaningful information to be communicated.
\blue{We restrict our focus to two agents on FEMNIST, connected by a single edge, and sweep the latent dimension while keeping the two latent spaces heterogeneous, with $d_0=d$ and $d_1=d/2$ for $d\in\{8,16,32,64\}$.
The writers are split into three disjoint groups, one for each agent and one for the pilots, so that the two agents never share a handwriting style and the pilots come from writers that neither agent has seen.
Agent~$0$ trains on digits $\mathcal{C}_0=\{0,\dots,6\}$ and agent~$1$ on $\mathcal{C}_1=\{3,\dots,9\}$, so that each agent misses three of the classes known by the other.
Both agents are plain CNNs with four blocks, which differ in width, and the last block sets the latent dimension (App.\nb\ref{app:exp_details}).
Since $d_0\neq d_1$, FedProto and FedMuscle retain their projection, of dimension $d$, and ComFed its projection, of dimension $d/2$.
Besides the private accuracy on the target classes, we report the communication accuracy on the classes known by both the sender and receiver, $\mathcal{C}_i \cap \mathcal{C}_j$, and on those unseen by the sender, $\mathcal{C}_i\setminus\mathcal{C}_j$, where communication has to transfer knowledge.}

\begin{figure}[t]
\centering
\includegraphics[width=\linewidth]{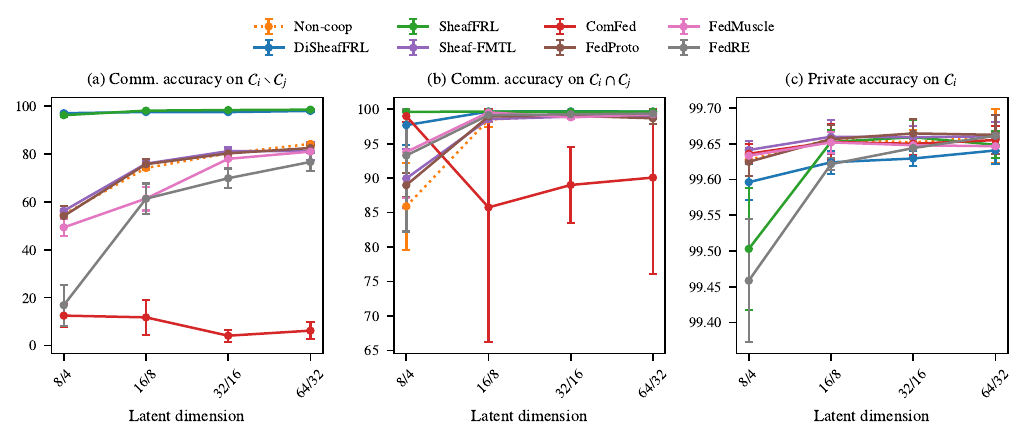}
\caption{\blue{Robustness to semantic compression on FEMNIST, with two agents of latent dimensions $d_0=d$ and $d_1=d/2$. (a) Communication accuracy on the classes unseen by the sender, $\mathcal{C}_i\setminus\mathcal{C}_j$. (b) Communication accuracy on the classes known by both agents, $\mathcal{C}_i\cap\mathcal{C}_j$. (c) Private accuracy on the target classes. Mean and standard deviation over $3$ seeds.}}
\label{fig:compression}
\end{figure}

\blue{\spara{Results.}
\cref{fig:compression} shows the results.
On the classes unseen by the sender, both variants of Sheaf-FRL are essentially insensitive to compression.
The baselines, instead, degrade as the latent spaces shrink, from at most about $80\%$ with $d=64$ to at most about $60\%$ with $d=8$, so that the gap of Sheaf-FRL grows from about $20$ to about $40$ points; FedRE drops to $17\%$, and ComFed stays below $15\%$ at every dimension.
On the classes known by both agents, all methods but ComFed are close for $d\ge16$. 
Additionally, at $d=8$ the other baselines drop to $85$-$95\%$ and Sheaf-FRL stays above $97\%$.
The private accuracy is high for every method and dimension, so that compression affects the communication, not the local task.
Hence, Sheaf-FRL extracts and transmits semantically meaningful information even when the latent spaces are strongly compressed, which is where independently trained, or loosely coupled, latent spaces become hardest to align.}
\section{\blue{Ablation studies}}
\label{app:ablation}

\blue{We study the sensitivity of DiSheaf-FRL to the number of pilots, the geometry of the restriction maps, the regularization coefficient, the geometric reference, and the communication topology.
All ablations are run on the MNIST benchmark of \cref{sec:sem_comm}, where each agent trains on seven of the ten classes, with the largest latent dimensions among our benchmarks, $d_i\in\{128,224,256,384\}$.
The pilot budget is varied at $\lambda=10^{-2}$ over nested pilot sets, with $K\in\{384,768,1536,2048,4096\}$; the smallest budget is just above the identifiability requirement $K\ge\max_{(i,j)\in\edgeset}\min(d_i,d_j)=256$ (App.\nb\ref{app:identifiability}).
The regularization coefficient is varied at $K=4096$ over $\lambda\in\{0.003,0.01,0.03,0.1,0.3,1\}$.}

\blue{Since each agent can only predict its own target classes, we split the test samples of each directed edge $j\to i$ according to their class.
We report \emph{(i)} the communication accuracy on $\mathcal{C}_i\setminus\mathcal{C}_j$, i.e., on the classes that the receiver $i$ knows but the sender $j$ has never been trained on, \emph{(ii)} the communication accuracy on $\mathcal{C}_i\cap\mathcal{C}_j$, i.e., on the classes known by both agents, and \emph{(iii)} the private accuracy of each agent on its own target classes.
All metrics are averaged across agents or directed edges, and reported as mean and standard deviation over $3$ seeds (\cref{fig:ablation-pilots,fig:ablation-lambda,fig:ablation-topology}).
The non-cooperative baseline is always reported, being a natural lower-bound.}

\begin{figure}[t]
\centering
\includegraphics[width=\linewidth]{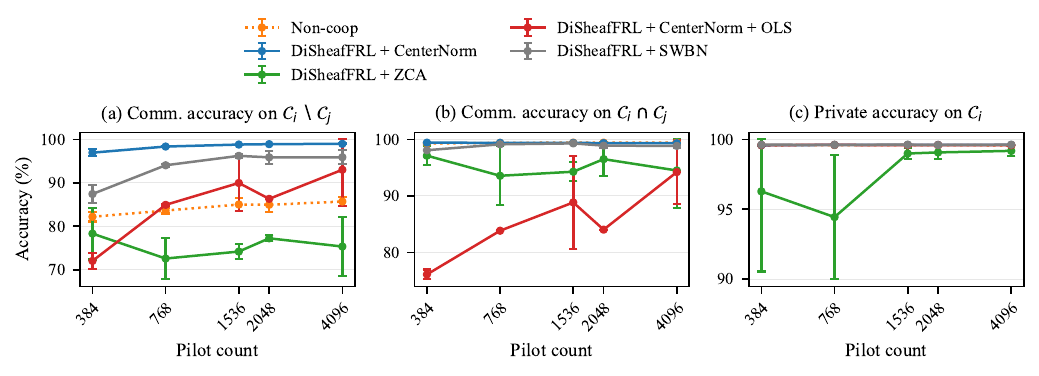}
\caption{\blue{Ablations against the pilot budget $K$ on the MNIST benchmark, at $\lambda=10^{-2}$. \textbf{(a)} Communication accuracy on the classes known by the receiver but unseen by the sender, $\mathcal{C}_i\setminus\mathcal{C}_j$. \textbf{(b)} Communication accuracy on the classes known by both agents, $\mathcal{C}_i\cap\mathcal{C}_j$. \textbf{(c)} Private accuracy on the target classes. DiSheaf-FRL with CenterNorm, with unconstrained least-squares maps (OLS), with analytical ZCA whitening, and with stochastic whitening batch normalization (SWBN), against the non-cooperative baseline. Mean and standard deviation over $3$ seeds.}}
\label{fig:ablation-pilots}
\end{figure}

\begin{figure}[t]
\centering
\includegraphics[width=\linewidth]{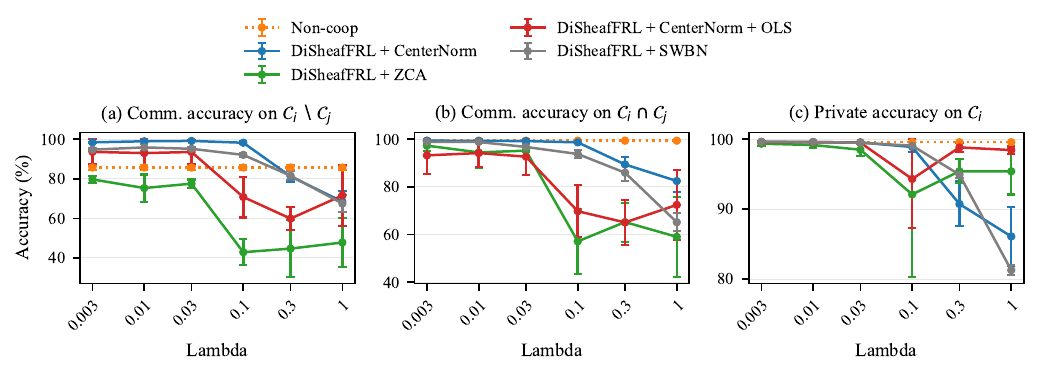}
\caption{\blue{Ablations against the regularization coefficient $\lambda$ on the MNIST benchmark, at $K=4096$. Panels and methods as in \cref{fig:ablation-pilots}.}}
\label{fig:ablation-lambda}
\end{figure}

\begin{figure}[t]
\centering
\includegraphics[width=\linewidth]{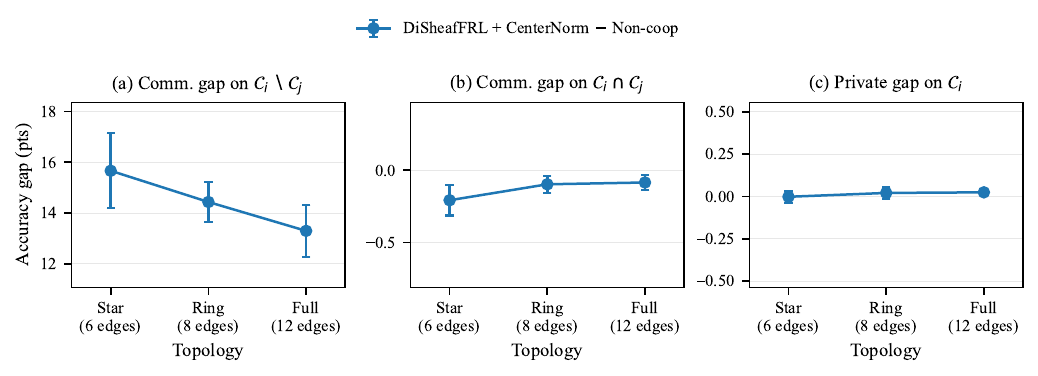}
\caption{\blue{Gap between DiSheaf-FRL and the non-cooperative baseline on the MNIST benchmark, for a star, a ring and a fully connected graph over the four agents. \textbf{(a)} Communication accuracy on $\mathcal{C}_i\setminus\mathcal{C}_j$. \textbf{(b)} Communication accuracy on $\mathcal{C}_i\cap\mathcal{C}_j$. \textbf{(c)} Private accuracy on the target classes. Mean and standard deviation over $3$ seeds.}}
\label{fig:ablation-topology}
\end{figure}

\subsection{Number of available pilots}\label{subapp:pilots}
\blue{DiSheaf-FRL is nearly insensitive to the pilot budget (\cref{fig:ablation-pilots}).
On the classes unseen by the sender, it reaches about $97\%$ with $K=384$ pilots and about $99\%$ with $K=4096$, whereas the non-cooperative baseline, whose post-hoc maps are fitted on the same pilots, stays roughly between $82\%$ and $86\%$.
Since the pilots are drawn from all classes, the sender observes, without labels, samples of classes it has never been trained on, and the gluing penalty drives it to represent them where the receiver represents the same classes.
Hence, the sender learns to encode classes it has never been trained on in a form that the receiver can decode, which is precisely the capability required by semantic communication.
On the classes known by both agents, DiSheaf-FRL and the non-cooperative baseline are on par, both around $99\%$ at every budget.
This is expected, since on these classes the private accuracy of both methods is equally high, and an agent already represents them well enough for a post-hoc Procrustes alignment to transfer them to a neighbor.
The gain of DiSheaf-FRL arrives precisely where it is needed, i.e., on the classes that lie outside the training set of the sender, and it is there that it improves the interoperability between the agents of the network.
Since the private accuracy stays at this level for every budget, communication through the learned maps is nearly lossless with respect to local inference.
The pilot budget required to achieve this is small compared with the size of the local datasets, about $1\%$ of the training samples of each agent for $K=384$.}

\subsection{Stiefel geometry}\label{subapp:stiefel-geometry}
\blue{We replace the Stiefel restriction maps with unconstrained linear maps, estimated in closed form by ordinary least squares (OLS) on the pilots, both during training and for communication at test time.
With unconstrained maps, the communication accuracy depends strongly on the pilot budget, with a large spread across seeds (\cref{fig:ablation-pilots}).
On the classes unseen by the sender, it drops to around $70\%$ with $K=384$ pilots, below the non-cooperative baseline, and only exceeds $90\%$ with $K=4096$, below DiSheaf-FRL with Stiefel maps at every budget.
On the classes known by both agents, it goes from roughly $75\%$ to $95\%$, below both DiSheaf-FRL and the non-cooperative baseline at every budget.
Across $\lambda$ (\cref{fig:ablation-lambda}), the unconstrained maps stay slightly above $90\%$ on both groups of classes for $\lambda\le0.03$, and degrade sharply from $\lambda=0.1$.
Their private accuracy, however, recovers to about $99\%$ for large $\lambda$, whereas it drops with Stiefel maps: 
this suggests that an unconstrained map can absorb the discrepancy between the latent spaces, instead of transferring it to the encoders through the gluing penalty.
Moreover, an unconstrained map between latent spaces of dimensions $d_i$ and $d_j$ has $d_id_j$ free parameters and requires many pilots to be estimated reliably. 
In contrast, the isometric constraint makes the map well determined already with $K\ge\min(d_i,d_j)$ pilots (App.\nb\ref{app:identifiability}).}

\subsection{Choice of the regularization hyperparameter}\label{subapp:lambdas}
\blue{DiSheaf-FRL is robust to the choice of $\lambda$ over almost two orders of magnitude (\cref{fig:ablation-lambda}).
For $\lambda\in[0.003,0.1]$, it reaches about $99\%$ on the classes unseen by the sender, against around $85\%$ for the non-cooperative baseline.
As for the communication accuracy on the shared classes and the private accuracy the two are almost indistinguishable.
From $\lambda=0.3$ the gluing penalty dominates the local objective, and all the metrics degrade.
The values used in our experiments, $\lambda\approx 10^{-2}$, lie within the robust range.}

\subsection{Choice of the geometric transformation}\label{subapp:geometric-transform}
\blue{We compare the performance of DiSheaf-FRL when equipped with three geometric reference transforms for the latent spaces. 
Specifically:
\emph{(i)} the CenterNorm layer of \cref{eq:cn}, 
\emph{(ii)} analytical ZCA whitening\nb\citep{kessy2018optimal}, 
and \emph{(iii)} the stochastic whitening batch normalization (SWBN) of\nb\citet{zhang2021stochastic}.
CenterNorm yields the highest communication accuracy across pilot budgets and regularization coefficients (\cref{fig:ablation-pilots,fig:ablation-lambda}).
SWBN comes second.
On the classes unseen by the sender, its gap from CenterNorm is largest with few pilots, and narrows without closing as the budget grows; on the shared classes, it matches CenterNorm from $K=768$ on.
Analytical ZCA whitening is unstable.
On the classes unseen by the sender, it stays below $80\%$ and the non-cooperative baseline at every budget, it shows a large spread across seeds, and it also degrades the private accuracy for $K\le768$.
Across $\lambda$, it plummets from $\lambda=0.1$.
Whitening makes the latent representations isotropic, thereby potentially amplifying semantically noisy directions, whereas CenterNorm only removes offsets and scales and preserves the anisotropies for the alignment.}

\subsection{Topology}\label{subapp:topology}
\blue{We run DiSheaf-FRL on three communication topologies over the four agents of the MNIST benchmark.
Specifically: 
\emph{(i)} a star, a ring, and a fully connected graph (used in \cref{sec:sem_comm}), with \num{6}, \num{8}, and \num{12} edges, respectively.
\cref{fig:ablation-topology} reports the gap with respect to the non-cooperative baseline, evaluated on the same topologies.
Overall, the gain on the classes unseen by the sender persists and is higher on sparser topologies.
This trend is consistent with the results in \cref{app:global-sections}.
Indeed, the star with the coarsest node at its center is a tree, so the gluing penalty imposes no consistency conditions along cycles.
In contrast, every additional cycle adds a constraint that the restriction maps cannot, in general, satisfy exactly.
On the shared classes and on the private accuracy, the gap stays close to zero on every topology.
Therefore, DiSheaf-FRL does not need a dense graph, and a sparse topology, which also reduces communication, is enough to obtain to increase semantic interoperability across the network.
}

\section{\blue{Emergence of global agreement}}\label{app:global-sections}
\blue{In this section, we compare the spectral structure of the sheaf Laplacian learned by DiSheaf-FRL with that of the sheaf obtained by a post-hoc alignment of independently trained agents. We do not focus on accuracy, but on how jointly shaping the latent spaces through the gluing penalty, as opposed to aligning them only once they have been trained, changes the geometry of the learned representations.}

\blue{\spara{Setup.} 
We consider $N=4$ agents on CIFAR-10, each equipped with a ResNet encoder terminating in a linear projection and in the CenterNorm layer of \cref{eq:cn}, followed by an MLP head $h_{\bm\psi_i}$. The latent dimensions are $d_i\in\{4,8,16,32\}$, and we index the agents as $i=0,1,2,3$ by increasing latent dimension, so that $d_{\min}=d_0=4$. Data heterogeneity is induced by label shift as in \cref{sec:sem_comm}, the regularization coefficient is $\lambda_{ij}=10^{-2}$ on every edge, and training lasts 30 epochs. We denote by $\mathbf{L}_{\mathcal F}(k)$ the sheaf Laplacian of \cref{eq:ns_Laplacian} learned by DiSheaf-FRL at epoch $k$, and by $\mathbf{L}_{\mathcal P}(k)$ the one obtained by training the agents independently, as in the non-cooperative baseline, and fitting at epoch $k$ the Procrustes maps of \cref{eq:procrustes_solution_V} on the pilot features $\{\mathbf{A}_i(\bm\varphi_i^k)\}_{i\in\mathcal{N}}$. The restriction maps of $\mathbf{L}_{\mathcal F}(k)$ enter the gluing penalty and therefore shape the training trajectory, whereas $\mathbf{L}_{\mathcal P}(k)$ is a passive measurement of the relation between latent spaces that evolve independently.}

\blue{\spara{Topologies.} As usual, edges follow the embedding orientation of \cref{subsec:sfrl}, i.e., they point from the lower- to the higher-dimensional agent. We consider four topologies, depicted in \cref{fig:global-sections-topologies}:
\begin{itemize}
    \item[(i)] the cycle $0\rightarrow1\rightarrow2\rightarrow3\leftarrow0$;
    \item[(ii)] the tree $0\rightarrow2\leftarrow1$, $2\rightarrow3$, with a collider at agent $2$;
    \item[(iii)] the star centered at agent $3$, i.e., the one with the highest-dimensional latent space;
    \item[(iv)] the star centered at agent $0$, i.e., the one with the lowest-dimensional latent space.
\end{itemize}
These topologies may differ in how many global sections they admit, being $\dim\ker(\mathbf{L}_{\mathcal{F}})\le d_{\min}$ \citep{dacunto2026} in every scenario. While in (i)--(iii) the restricion maps need to satisfy specific constraints, such as intersecting images, in topology (iv) the bound is tight by construction, since any $\mathbf{z}_0\in\mathbb{R}^{4}$ can extend to the global section $\mathbf{z}_i=\mathbf{V}_{i0}\mathbf{z}_0$. In the remaining topologies, global sections are subject to compatibility conditions impairing the emergence of a non-trivial global section space.}
\begin{figure}[h]
\centering
\resizebox{\linewidth}{!}{%
\begin{tikzpicture}[
    agent/.style={circle, draw, thick, fill=gray!15, inner sep=0pt, font=\normalsize},
    d4/.style={agent, minimum size=0.50cm},
    d8/.style={agent, minimum size=0.71cm},
    d16/.style={agent, minimum size=1.00cm},
    d32/.style={agent, minimum size=1.41cm},
    emb/.style={-{Stealth[length=2.2mm]}, thick},
    lab/.style={font=\normalsize}
]
\begin{scope}
    \node[d4]  (c0) at (0,0)     {$0$};
    \node[d8]  (c1) at (0,1.8)   {$1$};
    \node[d16] (c2) at (2,1.8)   {$2$};
    \node[d32] (c3) at (2,0)     {$3$};
    \draw[emb] (c0) -- (c1);
    \draw[emb] (c1) -- (c2);
    \draw[emb] (c2) -- (c3);
    \draw[emb] (c0) -- (c3);
    \node[lab] at (1,-1.1) {(i)};
\end{scope}
\begin{scope}[xshift=3.6cm]
    \node[d4]  (t0) at (0,1.8)   {$0$};
    \node[d8]  (t1) at (0,0)     {$1$};
    \node[d16] (t2) at (1.5,0.9) {$2$};
    \node[d32] (t3) at (3.2,0.9) {$3$};
    \draw[emb] (t0) -- (t2);
    \draw[emb] (t1) -- (t2);
    \draw[emb] (t2) -- (t3);
    \node[lab] at (1.6,-1.1) {(ii)};
\end{scope}
\begin{scope}[xshift=8.3cm]
    \node[d32] (s3) at (1.6,0.9) {$3$};
    \node[d4]  (s0) at (0,1.9)   {$0$};
    \node[d8]  (s1) at (0,-0.1)  {$1$};
    \node[d16] (s2) at (3.4,0.9) {$2$};
    \draw[emb] (s0) -- (s3);
    \draw[emb] (s1) -- (s3);
    \draw[emb] (s2) -- (s3);
    \node[lab] at (1.7,-1.1) {(iii)};
\end{scope}
\begin{scope}[xshift=13cm]
    \node[d4]  (r0) at (1.4,0.9) {$0$};
    \node[d8]  (r1) at (0,1.9)   {$1$};
    \node[d16] (r2) at (0,-0.1)  {$2$};
    \node[d32] (r3) at (3,0.9)   {$3$};
    \draw[emb] (r0) -- (r1);
    \draw[emb] (r0) -- (r2);
    \draw[emb] (r0) -- (r3);
    \node[lab] at (1.5,-1.1) {(iv)};
\end{scope}
\end{tikzpicture}%
}
\caption{\blue{Topologies considered in App.~\ref{app:global-sections}: (i) cycle, (ii) tree with a collider at agent $2$, (iii) star centered at agent $3$, and (iv) star centered at agent $0$. Node $i$ carries a latent space of dimension $d_i\in\{4,8,16,32\}$, and its radius grows with $d_i$; edges follow the embedding orientation, from the lower- to the higher-dimensional agent.}}
\label{fig:global-sections-topologies}
\end{figure}

\blue{\spara{Metrics.} Let $\{(\x^n,y^n)\}_{n=1}^{M}$ be the test set of App.~\ref{app:exp_details}, let $\mathbf{Z}_i(k)=[f_{\bm\varphi_i^k}(\x^n)]_{n=1}^{M}\in\mathbb{R}^{d_i\times M}$ collect the embeddings of agent $i$ at epoch $k$, and let $\mathbf{Z}(k)=[\mathbf{Z}_0(k)^\top,\dots,\mathbf{Z}_3(k)^\top]^\top\in\mathbb{R}^{D\times M}$, with $D=\sum_{i\in\mathcal{N}}d_i=60$, be their concatenation, whose $n$-th column $\mathbf{z}^n(k)\in C^0(\mathcal{G};\mathcal{F})$ is the $0$-cochain collecting the embeddings of $\x^n$ across the agents. 
We denote by $\bar{\mathbf{Z}}(k)=\mathbf{Z}(k)\big(\mathbf{I}_M-\tfrac{1}{M}\mathbf{1}\mathbf{1}^\top\big)$ the centered embeddings, with columns $\bar{\mathbf{z}}^n(k)$, and by $\boldsymbol{\Sigma}(k)=\tfrac{1}{M}\bar{\mathbf{Z}}(k)\bar{\mathbf{Z}}(k)^\top$ their empirical covariance. 
Centering removes the contribution of the mean cochain, which does not reflect the agreement on individual samples.
For $\star\in\{\mathcal{F},\mathcal{P}\}$, we write the eigendecomposition $\mathbf{L}_{\star}(k)=\sum_{\ell=1}^{D}\mu_\ell\,\mathbf{u}_\ell\mathbf{u}_\ell^\top$, with $0\le\mu_1\le\dots\le\mu_D$, dropping the dependence on $\star$ and $k$ when clear from the context, and we define the fraction of energy of the embeddings along the $\ell$-th eigenvector as
\begin{equation}
    \omega_\ell := \frac{\mathbf{u}_\ell^\top\boldsymbol{\Sigma}\,\mathbf{u}_\ell}{\mathrm{tr}\{\boldsymbol{\Sigma}\}}\,,\qquad \sum_{\ell=1}^{D}\omega_\ell=1\,.
\end{equation}
The restriction maps (and thus the sheaf Laplacians) are learned on the pilots, whereas all the metrics are evaluated on the test set, so that they also measure how well the alignment generalizes beyond the pilots. 
We report the following quantities:
\begin{itemize}
    \item \emph{Spectrum.} The eigenvalues $\mu_1,\dots,\mu_D$ of $\mathbf{L}_{\mathcal F}(k)$ and $\mathbf{L}_{\mathcal P}(k)$ at the first and at the last epoch. 
    Each eigenvalue is the total variation of the corresponding eigenvector, $\mu_\ell=\mathcal{TV}(\mathbf{u}_\ell)$, i.e., the sum over the edges of the disagreements in \cref{eq:total_variation}.
    The eigenvectors thus admit a variational interpretation: $\mathbf{u}_1$ minimizes the total variation among all unit-norm $0$-cochains, and each subsequent eigenvector minimizes it among those orthogonal to the preceding ones.
    Since the sheaf Laplacian may have repeated eigenvalues, the intrinsic spectral objects are the eigenspaces rather than the individual eigenvectors\nb\citep{d2026sheafsp}, and we consider the subspace spanned by the first $d_{\min}$ eigenvectors.
    Since $\dim\ker(\mathbf{L}_{\star})\le d_{\min}$, the first $d_{\min}$ are the only eigenvalues that may vanish, and their decay towards zero signals that this subspace approaches a space of global sections, i.e., of $0$-cochains on which all agents agree.
    \item \emph{Normalized total variation.} The sheaf total variation of \cref{eq:total_variation}, accumulated over the test set and normalized by the total variance of the embeddings,
    \begin{equation}
        \overline{\mathcal{TV}}_{\star}(k) := \frac{\sum_{n=1}^{M}\mathcal{TV}^k(\bar{\mathbf{z}}^n)}{\sum_{n=1}^{M}\|\bar{\mathbf{z}}^n\|_2^2} = \frac{\mathrm{tr}\{\bar{\mathbf{Z}}^\top\mathbf{L}_{\star}(k)\bar{\mathbf{Z}}\}}{\|\bar{\mathbf{Z}}\|_\mathrm F^2} = \frac{\mathrm{tr}\{\mathbf{L}_{\star}(k)\boldsymbol{\Sigma}\}}{\mathrm{tr}\{\boldsymbol{\Sigma}\}} = \sum_{\ell=1}^{D}\mu_\ell\,\omega_\ell\,,
    \end{equation}
    where $\mathcal{TV}^k(\cdot)$ is evaluated with the sheaf learned at epoch $k$. 
    It quantifies the cumulative disagreement of the system: it is invariant to a common rescaling of the embeddings, and vanishes if and only if every centered test embedding is a global section.
    \item \emph{Low-frequency energy.} The fraction of energy of the embeddings captured by the subspace of the first $d_{\min}$ eigenvectors,
    \begin{equation}
        \rho_{\star}(k) := \frac{\|\bm \Pi_{\star}(k)\bar{\mathbf{Z}}\|_\mathrm F^2}{\|\bar{\mathbf{Z}}\|_\mathrm F^2} = \frac{\mathrm{tr}\{\bm \Pi_{\star}(k)\boldsymbol{\Sigma}\}}{\mathrm{tr}\{\boldsymbol{\Sigma}\}} = \sum_{\ell=1}^{d_{\min}}\omega_\ell\,,\qquad \bm{\Pi}_{\star}:=\sum_{\ell=1}^{d_{\min}}\mathbf{u}_\ell\mathbf{u}_\ell^\top\,,
    \end{equation}
    where $\bm \Pi_{\star}(k)$ is the orthogonal projector onto this subspace at epoch $k$.
    Since $\dim\ker(\mathbf{L}_{\star}(k))\le d_{\min}$, this subspace always contains $\ker(\mathbf{L}_{\star}(k))$, 
    and it collects the directions of lowest inconsistency, that is, of approximate global agreement.
\end{itemize}
The two latter metrics capture two complementary aspects of the emerging geometry: the normalized total variation $\overline{\mathcal{TV}}_{\star}(k)$ effectively measures the overall disagreement between the embeddings, whereas the low-frequency energy $\rho_*(k)$ quantifies the 
energy lying in the directions of approximate global agreement, regardless of the values of the corresponding eigenvalues.}

\begin{figure}
    \centering
    \includegraphics[width=1\linewidth]{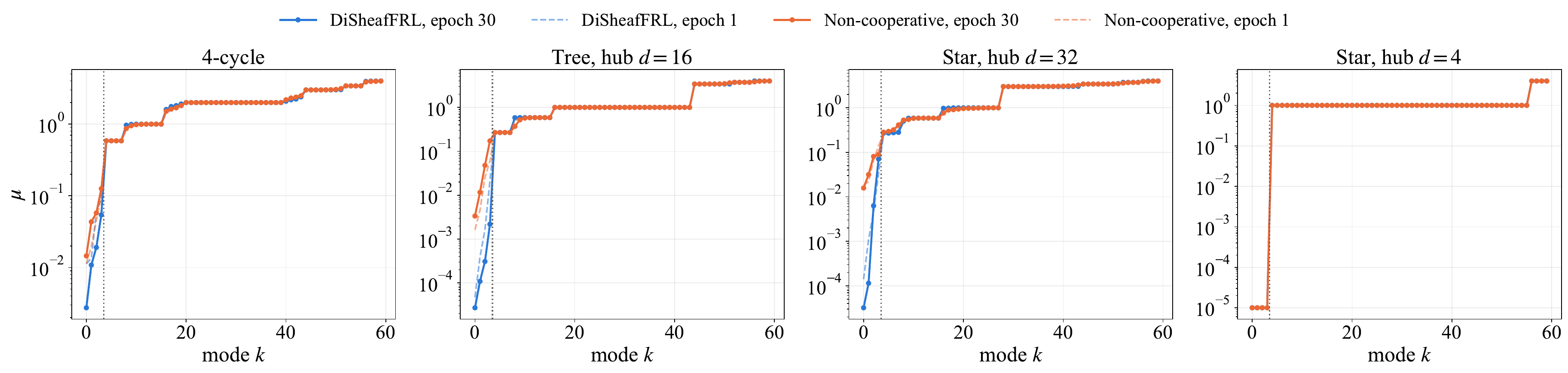}
    \caption{\blue{Spectrum of the Sheaf Laplacian after the first epoch of training and at the end of the training.}}
    \label{fig:spectrum}
\end{figure}

\begin{figure}
    \centering
    \includegraphics[width=1\linewidth]{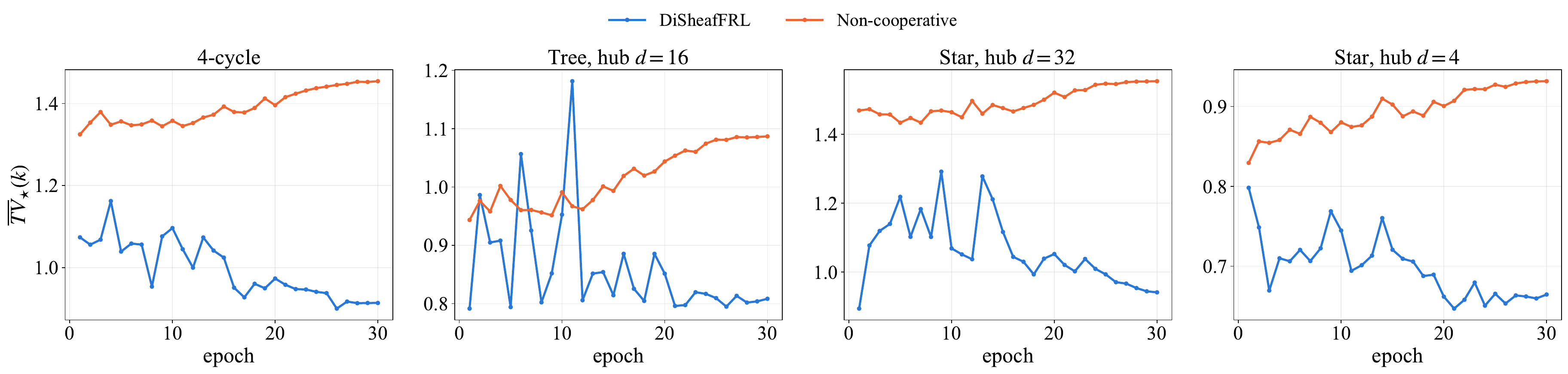}
    \caption{\blue{Normalized Total Variation for the embedded test set.}}
    \label{fig:tv_normalized}
\end{figure}

\begin{figure}
    \centering
    \includegraphics[width=1\linewidth]{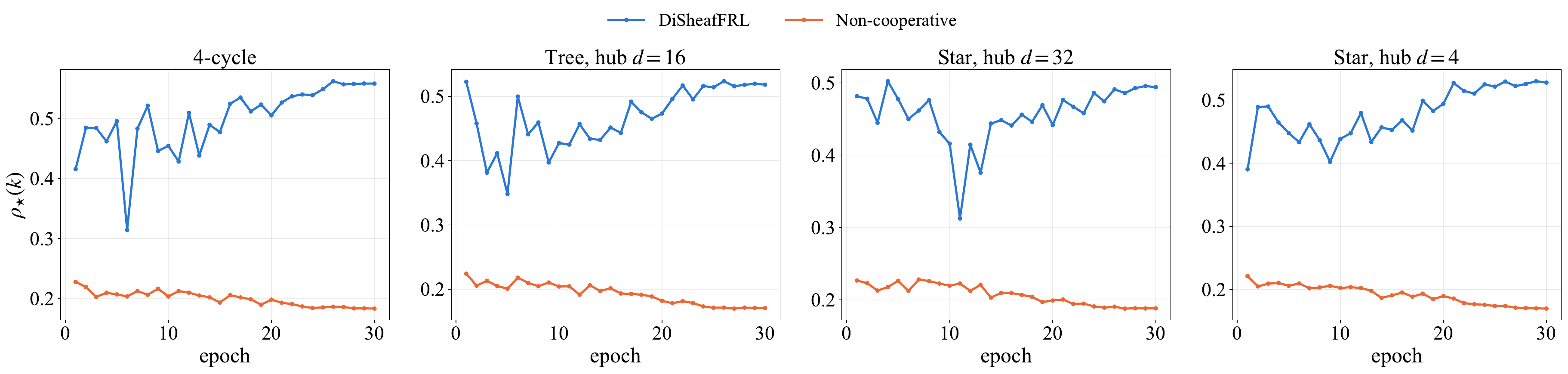}
    \caption{\blue{Low-frequency energy ratio of the embedded test set.}}
    \label{fig:low_energy}
\end{figure}

\blue{\spara{Results.} \cref{fig:spectrum} reports the spectra of $\mathbf{L}_{\mathcal F}(k)$ and $\mathbf{L}_{\mathcal P}(k)$ at the first and at the last epoch. Beyond the first $d_{\min}$ eigenvalues, the two Laplacians exhibit similar spectra across all topologies, suggesting that the alignment mechanism shapes the spectrum mostly at its lowest end, i.e., along the directions closest to the space of global sections. On the first $d_{\min}$ eigenvalues, instead, DiSheaf-FRL consistently attains a lower minimum eigenvalue in topologies (i)--(iii). The cycle (i) is the topology in which the compatibility conditions defining a global section are most entangled, since they must hold along the whole loop: here, neither pipeline reaches a subspace of high agreement, although the minimum eigenvalue of DiSheaf-FRL is almost one order of magnitude lower than that of the non-cooperative baseline. 
In topologies (ii) and (iii) the gap widens: 
the minimum eigenvalue of DiSheaf-FRL is about $10^{-5}$, against about $10^{-3}$ and $10^{-2}$ for the non-cooperative baseline, i.e., almost two and three orders of magnitude lower.
This indicates that jointly shaping the latent spaces and the restriction maps makes an approximate space of global sections emerge, of dimension at most $d_{\min}$.}

\blue{Topology (iv) is the only one in which a non-trivial space of global sections is guaranteed to exist, and its first $d_{\min}$ eigenvalues vanish for both pipelines, since they follow from the same isometric embedding structure rather than from training. The spectrum alone therefore cannot discriminate between the two pipelines in this scenario, which is precisely what the remaining metrics are designed for: while the existence of global sections is a property of the parameterization, $\overline{\mathcal{TV}}_{\star}$ and $\rho_{\star}$ assess how relevant this space actually is for the learned representations. 
As shown in \cref{fig:tv_normalized}, the normalized total variation of the non-cooperative baseline increases steadily along training in every topology, as the independently trained latent spaces drift apart while specializing, whereas that of DiSheaf-FRL decreases over the second half of training and ends lower in every topology.
Consistently, \cref{fig:low_energy} shows that the low-frequency energy is higher for DiSheaf-FRL: on the non-cooperative sheaf, not only is the disagreement larger, but the eigenspace associated with alignment also captures a smaller share of the energy of the embeddings. Both effects are observed consistently across all topologies, supporting the conclusion that DiSheaf-FRL builds a system in which the learned alignment does not merely relate the latent spaces after the fact, but actively guides how they are glued together.
At the same time, the low-frequency energy of DiSheaf-FRL saturates at about $0.5$ and its normalized total variation remains far from zero: the agreement emerges on a low-dimensional subspace, of dimension at most $d_{\min}=4$ out of $D=60$, while about half of the energy of the embeddings remains in agent-specific directions.
Hence, the learned representations are not glued into a single shared latent space, consistently with our claim that global consistency need not be imposed for effective communication.}

\end{document}